\documentclass[hidelinks,onefignum,onetabnum]{siamart250211}

\usepackage{lipsum}
\usepackage{amsfonts}
\usepackage{graphicx}
\usepackage{epstopdf}
\usepackage{mathrsfs}

\ifpdf
  \DeclareGraphicsExtensions{.eps,.pdf,.png,.jpg}
\else
  \DeclareGraphicsExtensions{.eps}
\fi
\usepackage{tikz}
\usepackage{tikz-cd}
\usepackage{enumitem}
\setlist[enumerate,1]{label=(\arabic*).}   
\usepackage{xcolor}
\usepackage{booktabs} 
\usepackage{cleveref}  
\usepackage{subcaption} 
\usepackage{algorithmicx}
\usepackage{algorithm}
\usepackage{algpseudocode} 
\algblockdefx{SubStep}{EndSubStep}{}{}

\usepackage{bm} 
\usepackage{verbatim} 

\newsiamremark{remark}{Remark}
\newsiamremark{assumption}{Assumption} 
\newsiamremark{hypothesis}{Hypothesis}
\crefname{hypothesis}{Hypothesis}{Hypotheses}
\newsiamthm{claim}{Claim}
\newsiamremark{fact}{Fact}
\crefname{fact}{Fact}{Facts}

\newcommand{\calP}{\mathcal{P}}
\newcommand{\E}{\mathcal{E}}
\newcommand{\W}{\mathcal{W}}

\newcommand{\bbD}{\mathbb{D}}
\newcommand{\bbX}{\mathbb{X}}
\newcommand{\bbA}{\mathbb{A}}

\newcommand{\rmd}{\mathrm{d}}

\newcommand{\bbR}{\mathbb{R}}
\newcommand{\JKO}{\operatorname{JKO}}

\newcommand{\bfT}{\bm{T}}
\newcommand{\bfV}{\bm{V}}
\newcommand{\bfI}{\bm{I}}
\newcommand{\bfNN}{\bm{NN}}
\newcommand{\scrT}{\mathcal{T}}

\headers{Learn to Evolve}{Feng, Wang, Needell, and Lai}

\title{Learn to Evolve: Self-supervised Neural JKO Operator for Wasserstein Gradient Flow}

\author{
  Xue Feng\thanks{Department of Mathematics, University of California, Los Angeles, CA 
  (\email{xffeng@ucla.edu}, \email{deanna@math.ucla.edu}).}
  \and
  Li Wang\thanks{School of Mathematics, University of Minnesota, Minneapolis, MN 
  (\email{liwang@umn.edu}).}
  \and
  Deanna Needell\footnotemark[1]
  \and
  Rongjie Lai\thanks{Department of Mathematics, Purdue University, West Lafayette, IN 
  (\email{lairj@purdue.edu}).}
}

\usepackage{amsopn}

\ifpdf
\hypersetup{
  pdftitle={An Example Article},
  pdfauthor={D. Doe, P. T. Frank, and J. E. Smith}
}
\fi

\usetikzlibrary{positioning, arrows.meta, shapes, calc}
\tikzset{
  nodeStyle/.style={
    ellipse, draw=pink!50!black, fill=pink,
    minimum size=1.2cm, font=\small,
    align=center 
  },
  op2/.style={
    rectangle, draw=yellow!50!black, fill=yellow!10,
    minimum width=1.6cm, minimum height=1cm,
    font=\small, rounded corners=3pt
  },
  op/.style={
    rectangle, draw=cyan!50!black, fill=cyan!10,
    minimum width=1.2cm, minimum height=1cm,
    font=\small, rounded corners=3pt
  },
  arrow/.style={-{Latex[length=2mm]}, thick, color=gray!60!black},
  label/.style={font=\scriptsize}
}

\begin{document}

\maketitle

\begin{abstract}
The Jordan–Kinderlehrer–Otto (JKO) scheme provides a stable variational framework for computing Wasserstein gradient flows, but its practical use is often limited by the high computational cost of repeatedly solving the JKO subproblems. We propose a self-supervised approach for learning a JKO solution operator without requiring numerical solutions of any JKO trajectories. The learned operator maps an input density directly to the minimizer of the corresponding JKO subproblem, and can be iteratively applied to efficiently generate the gradient-flow evolution.
A key challenge is that only a number of initial densities are typically available for training. To address this, we introduce a Learn-to-Evolve algorithm that jointly learns the JKO operator and its induced trajectories by alternating between trajectory generation and operator updates. As training progresses, the generated data increasingly approximates true JKO trajectories. Meanwhile, this Learn-to-Evolve strategy serves as a natural form of data augmentation, significantly enhancing the generalization ability of the learned operator. Numerical experiments demonstrate the accuracy, stability, and robustness of the proposed method across various choices of energies and initial conditions.

\end{abstract}

\begin{keywords}
example, \LaTeX
\end{keywords}

\begin{MSCcodes}
68Q25, 68R10, 68U05
\end{MSCcodes}
\section{Introduction}

In recent years, gradient flows in the space of probability measures have found rich applications in areas such as machine learning~\cite{richemond2017wasserstein,zhang2018policy} and sampling~\cite{hertrich2024importance, mokrov2021large}. The JKO scheme, a proximal iteration under Wasserstein distance,  plays a central role in numerically approximating Wasserstein gradient flows (WGF)~\cite{JordanKinderlehrerOtto1998,ambrosio2008gradient,Otto2001PorousMedium,fan2021variational} in two aspects:
 it provides a theoretical framework for proving the existence and properties of solutions to these nonlinear PDEs, and  
it offers a discretized WGF solution from an optimization perspective.  
As an implicit Euler method, JKO offers stability, energy decay, and robustness to approximation noise.
While exact JKO steps yield linear convergence to the global minimizer of $\mathcal{E}$, the inexact JKO remains well-behaved and converges to a neighborhood of the minimizer rather than diverging~\cite{zhuconvergence, cheng2024convergence, salim2020wasserstein}. These properties make JKO highly attractive.

However, solving each JKO sub-problem is generally computationally expensive.  
Numerical optimization approaches such as the primal–dual method~\cite{carrillo2022primal} and its regularized version \cite{li2020fisher}, the back-and-forth method~\cite{jacobs2021back}, and the augmented Lagrangian methods~\cite{santambrogio2017euclidean} involve solving large-scale optimization problems on the grid, which will become expensive in high dimensions.   
These challenges motivate the development of efficient and reliable solvers for the JKO problem. Recent work on training NNs to learn the JKO trajectory includes~\cite{lee2024deep, xu2024normalizing, vidal2023taming, mokrov2021large, fan2021variational, jin2025parameterized}.  
A key limitation of these methods is that they train a new network for each individual JKO iterations,  and they typically follow a \emph{progressive} training
strategy: the $n$-th iteration is trained only after the previous $n-1$ iterations are fixed.  
This approach is computationally expensive, fails to fully exploit the strong similarity between subproblems, and limits the use of larger architectures due to the need to train multiple models.

A promising alternative is to learn the \emph{solution operator} directly.  
For example, \cite{huang2024unsupervised} employs a single transformer network to solve multiple mean-field game (MFG) problems by learning the corresponding MFG solutions for different pairs of initial and target distributions.  
After training, the network generalizes across tasks and effectively serves as a solution operator.
This operator-learning viewpoint motivates our approach.
Instead of training a new network for each problem instance, we aim to learn the JKO operator itself. Our goal is to train a single neural network that approximates the JKO operator for a given energy functional (or a parameterized family). Once trained, it can be applied iteratively from an initial density to efficiently approximate the entire JKO trajectory.

Unlike the MFG setting, where diverse pairs of initial and target distributions are available,
the main challenge in training a JKO operator is the lack of training data.
We usually only have access to one or a small family of initial distributions.  
This is insufficient to guarantee the performance of the learned JKO operator  when applied iteratively,  
since the subsequent densities along the JKO trajectories are unseen during training and may differ significantly from the initial distributions.  
Moreover, they are generally difficult to approximate using random densities.  
To address this, we introduce the Learn-to-evolve algorithm, which learns both the JKO operator and the JKO trajectories from the given initials simultaneously.

In this framework, the evolving NN serves as a dynamic data generator: we repeatedly apply the current network to produce trajectory data starting from the prescribed initial densities. Training proceeds in a self-reinforcing loop, alternating between generating new trajectory data with the current network and updating the network parameters on the refreshed dataset. Under suitable conditions, the evolving data are guaranteed to converge to the true JKO trajectories, which drives the network toward the exact JKO operator.
Moreover, this process acts as a form of dynamic data augmentation. Although the evolving training dataset eventually converges to the true JKO trajectories, the intermediate iterates explore a much broader region of the state space.
Training on this enlarged union of intermediate datasets enables the learned operator to generalize well beyond both the initial density family and the true JKO trajectories themselves.
Our experiments demonstrate both the accuracy and the strong generalization ability of the approach across diverse examples, especially for the interaction kernel energy functional.

Another key insight from this work is that proximal formulations are particularly well-suited for learning iterative solution operators. Neural networks inherently approximate operators inexactly because of optimization and generalization errors. As a result, each iterative application of the learned operator produces a perturbed update. Proximal schemes are inherently robust to such perturbations, whereas explicit forward methods tend to propagate and amplify errors and may diverge rapidly.
Importantly, the proposed Learn-to-Evolve framework is not specific to the JKO scheme. Its self-generated data mechanism and iterative training paradigm extend naturally to a wide class of proximal operators and implicit iterative mappings. This suggests that the methodology can be applied far beyond Wasserstein gradient flows.

Overall, this work bridges Wasserstein gradient flows, JKO variational discretization, and operator learning, offering a scalable and data-efficient paradigm for high-dimensional problems. We summarized our contributions as follows: 
\begin{enumerate}
    \item Operator Learning for JKO. We train a single neural network to approximate the JKO operator for a parameterized family of energy functional. Once trained, this JKO operator can be applied iteratively to generate full Wasserstein gradient flow trajectories.
   \item Learn-to-Evolve Framework. To mitigate data scarcity, we introduce a dynamic training process where, in each iteration, the current network generates new densities along evolving trajectories, which are then used to update the network. This self-evolving data augmentation enables JKO operator to generalize beyond the prescribed training family.
   \item Theoretical and Empirical Validation. We provide convergence guarantees under suitable conditions, showing that the evolving dataset converges to the true JKO trajectory. Our numerical experiments demonstrate that the JKO operator achieves both high accuracy and strong generalization across diverse examples, with potential extensions to other iterative variational problems.
\end{enumerate}

The rest of the paper is organized as follows. Section 2 reviews related work. Section 3 introduces the Learn-to-Evolve algorithm and establishes its convergence analysis.  Implementation details are discussed in Section \ref{sec:Implementation}. Numerical experiments are presented in Section \ref{sec:numerical} to demonstrate the effectiveness of robustness of the proposed method, followed by conclusions in Section \ref{sec:conclusion}..

\section{Preliminaries}
\subsection{Notation}

We denote by $\mathcal{P}(\Omega)$ the set of probability measures on $\Omega \subset \mathbb{R}^d$ 
with finite second moment that are absolutely continuous with respect to the Lebesgue measure. 
In the following, we do not distinguish between a probability measure and its corresponding density when it is clear from context.
Given $\rho, \nu \in \mathcal{P}(\Omega)$, a Borel map $\bfT: \Omega \to \Omega$ is said to transport $\rho$ onto $\nu$ if
\(
\nu(B) = \rho\!\left(\bfT^{-1}(B)\right) \)
for all Borel sets \( B \subseteq \mathbb{R}^d.
\)
We call $\nu$ the \emph{push-forward} measure  of $\rho$ under $\bfT$ and write $\nu = \bfT_{\sharp} \rho$. 
We denote by $\bfI$ the identity map, i.e., $\bfI(x) = x, \, \forall \,x$.
The space of all Borel maps from $\Omega$ to $\Omega$ is denoted by $\mathrm{Map}(\Omega,\Omega)$.

The \emph{2-Wasserstein distance} between $\rho$ and $\nu$ is defined by
\begin{equation}
    \label{equ:wasserstein2}
    \W_2(\rho,\nu) := \left( \inf_{\bfT: \Omega \to \Omega} 
\left\{ \int |x - \bfT(x)|^2 \, d\rho(x) : \bfT_{\sharp}\rho = \nu \right\} \right)^{1/2}.
\end{equation}
When $\rho$ is absolutely continuous, the \emph{optimal transport map} that pushes $\rho$ toward $\nu$ is unique, and 
we denote it by $\bfT_{\rho}^{\nu}$, so that $(\bfT_{\rho}^{\nu})_{\sharp} \rho = \nu$.

Given a functional $\E : \mathcal{P}(\Omega) \to \mathbb{R}$, 
when $\E$ and $\rho \in \mathcal{P}(\Omega)$ are sufficiently smooth, the \emph{Wasserstein gradient} of $\E$ at $\rho$ is
\[
\nabla_\W \E(\rho) = - \nabla \cdot \left( \rho \nabla \frac{\delta \E}{\delta \rho}(\rho) \right),
\]
where $\delta \E / \delta \rho$ is the functional derivative of $\E$ 
(see Chapters~8 and~10 of \cite{ambrosio2008gradient}). 
The associated \emph{Wasserstein gradient flow} of $\E$ is given by the Cauchy problem
\begin{equation}
\frac{d}{dt}\rho(t) = - \nabla_\W \E(\rho(t)), 
\qquad \rho(0) = \rho^0,
\label{equ:continuousGradientFlow}
\end{equation}
which is well-defined as long as $\nabla_\W \E(\rho(t))$ exists along the flow.
If $\E$ is $\lambda$-convex along generalized geodesics with $\lambda > 0$, 
then it admits a unique minimizer $\rho^*$ and the flow satisfies the exponential convergence estimate
\(
\W_2\bigl(\rho(t), \rho^*\bigr) \;\leq\; e^{-\lambda t} \, \W_2(\rho^0, \rho^*).
\)
(See Section~11.2 of \cite{ambrosio2008gradient}.) 
Several important examples of functionals that are convex along generalized geodesics 
are discussed in Section~9.3 of \cite{ambrosio2008gradient}.

\subsection{JKO operator}

Given an energy functional $\E_\beta: \calP(\Omega) \to \bbR$, which is parameterized by $\beta \in \Lambda$ (or $\E$ if there is no parameter),
and a step size $\Delta t > 0$,
we define the JKO operator (equivalently, the Wasserstein proximal operator) as 
\begin{equation}
\JKO(\,\cdot\,\,;\,\beta): \calP(\Omega) \mapsto \calP(\Omega),
\quad
    \JKO(\rho;\beta) :=
    \arg \min_{\nu \in \calP(\Omega)} 
    \left\{
        \W_2^2(\nu, \rho)
        \;+\;
        2\,\Delta t \,\E_\beta(\nu)
    \right\}.
    \label{equ:jkoOperator}
\end{equation}
Throughout this paper,
we assume that
for all $\beta \in \Lambda$,
$\E_\beta$ is proper, coercive, lower semicontinuous, and $\lambda$-convex along generalized geodesics, with $\Delta t \lambda > -1$. Under these conditions, the minimization problem admits a unique solution and the JKO operator is well-defined \cite{ambrosio2008gradient,carlen2013contraction}.


The aforementioned Wasserstein proximal operator can be used to generate a discrete-in-time gradient flow, known as 
the \emph{JKO scheme} or the proximal point algorithm in Wasserstein spaces.
It iteratively solves the minimization problems
in \Cref{equ:jkoOperator} from a  given $\rho^0$,
i.e., 
\(
    \rho^{n} =\; \JKO(\rho^{n-1}).
\)
The sequence $\{\rho^n\}$ is a backward Euler approximation of the continuous Wasserstein gradient flow \eqref{equ:continuousGradientFlow},
satisfying
\(
    \W_2\big(\rho^n, \rho(n\,\Delta t)\big) \;\leq\; \mathcal{O}(\Delta t);
\)
see Theorem 4.0.4 of \cite{ambrosio2008gradient}. 
Compared to forward methods, the backward Euler scheme is more stable, and the energy $\E(\rho^n)$ is guaranteed to decrease for any $\Delta t>0$. 
When the JKO operator is computed exactly, the sequence $\rho^n$ can be shown to converge to the global minimizer of $\mathcal{E}$ under appropriate assumptions; for inexact JKO schemes, one instead obtains convergence to a neighborhood of the global minimizer~\cite{zhuconvergence,cheng2024convergence,salim2020wasserstein}.

Let $\rho^{+}=\JKO(\rho)$.
If $\E$ and $\rho^{+}$ are sufficiently smooth such  that the Wasserstein-2 gradient $\nabla_\W \E(\rho^{+})$ is well-defined, then
$\rho^{+}$ is the unique minimizer of the objective function in \Cref{equ:jkoOperator} if any only if the optimal transport map pushing 
$\rho$ toward $\rho^+$ satisfies 
\begin{equation}
    \bfT_{\rho^{+}}^{\rho} \;=\; \bfI + \Delta t \nabla \frac{\delta \E}{\delta \rho}(\rho^{+}),
    \label{equ:OT}
\end{equation}
where $\delta \E / \delta \rho$ is the functional derivative of $\E$;
see Lemma 2.2 of \cite{carlen2013contraction} and Theorem 2.23 of \cite{craig2016exponential}.
Note that 
\eqref{equ:OT} 
is an implicit formula of the optimal map, and is equivalent to
$[\bfI+\Delta t \nabla \frac{\delta \E}{\delta \rho}\left(\rho^{+}\right)]_{\#}\rho^{+} = \rho$.

\section{Learn to Evolve}
The original JKO subproblem in \eqref{equ:jkoOperator} is notoriously difficult to solve. Eulerian-based approaches, such as augmented Lagrangian methods \cite{santambrogio2017euclidean}, primal dual schemes \cite{carrillo2022primal}, and back-and-forth iterations \cite{jacobs2021back}, work on discretized densities and hence suffer from the curse of dimensionality: grid-based representations and PDE constraints scale exponentially with the ambient dimension, making these methods impractical beyond low dimensions. They also tend to become brittle in optimization for large time steps due to non-convexity and stability restrictions. An explicit kernel reformulation was recently proposed in \cite{li2023kernel}, which avoids inner optimization but still requires high-dimensional sampling and is largely limited to linear energies. Lagrangian particle methods with neural parameterizations \cite{lee2024deep} show promise in higher dimensions; however, they still require (re)training a transport map at every time step and for each new problem instance, which compounds computational cost and error.

In this work, we propose to train one single neural operator,
denoted as  $\scrT^*$,
to learn the mapping
\(
(\rho,\beta) \mapsto \JKO(\rho;\beta)\,
\)
such that  
the output of the $\scrT^*$ is the optimal transport map pushing $\rho$ to $\JKO(\rho;\beta)$, i.e.,
\begin{equation}
    \label{equ:neuroJKOtarget}
\scrT^*(\rho,\beta)_\sharp \rho = \JKO(\rho;\beta). 
\end{equation}
Here $\beta \in \Lambda$ encodes the parametric family of energy functionals.
Since forward passes of neural networks are computationally efficient and well-suited for modern CPU/GPU acceleration, the proposed neural operator can, once trained, be used to efficiently generate the JKO sequence. Starting from any initial density $\rho^0$ and parameter $\beta$, repeated application of
\[
\rho^{n+1} =  \scrT^*(\rho^{n}, \beta)_\sharp \rho^{n}, 
\quad n = 0,1,2,\cdots,
\]
produces the trajectory
\[
\rho^0 \xrightarrow{\scrT^*} \rho(\Delta t) \xrightarrow{\scrT^*} \rho(2\Delta t)\xrightarrow{\scrT^*} 
 \rho(3\Delta t)\xrightarrow{\scrT^*}\cdots.
\]

More specifically, we consider the following population risk minimization problem:
\begin{equation}
    \begin{aligned}
        \label{equ:optimalmap}
  \min_{\scrT: \calP (\Omega) \times  \Lambda \,\,\mapsto \mathrm{Map}(\Omega,\Omega)} 
    \mathcal{L}(\scrT)&:= \mathbb{E}_{\rho\sim \mu_\rho, \beta\sim \mu_\beta} \,\,\ell(\scrT;\rho, \beta), \\
\quad \text{with} \quad
\ell(\scrT; \rho, \beta)&:=
        \W_2^2\big(\scrT(\rho,\beta)_\sharp \rho, \rho\big)
        \;+\;
        2\,\Delta t \,\E_\beta\big(\scrT(\rho,\beta)_\sharp \rho\big),
    \end{aligned}
\end{equation}
where $\scrT$ represents a neural operator, and $\mu_\rho, \mu_\beta$ denotes the meta distribution of $\rho$ and $\beta$, respectively. 
It is easy to verify that the minimizer of $\mathcal{L}(\scrT)$ satisfies \eqref{equ:neuroJKOtarget} for $\rho$-almost every $\rho$ and $\beta$-almost every $\beta$.
However, parameterizing a single operator \(\scrT\) that works uniformly for \emph{all} \((\rho,\beta)\in \mathcal{P}(\Omega)\times\Lambda\) would exceed the capacity of any practical neural network as $\mathcal{P}$ is an infinite-dimensional space. 
Besides, achieving high accuracy for each $\rho$ requires dense sampling in $\mathcal{P}(\Omega)$, which is impossible.

Instead, since Wasserstein gradient flows are essentially initial value problems,
we consider approximating the solution operator on a \emph{low-dimensional task manifold}: the initial data \(\rho^0\) of interest lie on a manifold \(\mathbb{M}_0 \subset \mathcal{P}(\Omega)\). For discrete times \(t_n=n\Delta t\), define the \emph{$n$th JKO evolution manifold} as
\[
\mathbb{M}_{n}
\;:=\;
\Big\{
\rho^n \;\Big|\;
\rho^{i+1}=\JKO(\rho^i;\beta)\ \text{for } i=0,1,\cdots,n-1,\ 
\rho^0\in\mathbb{M}_0,\ \beta\in\Lambda
\Big\}.
\]
Ideally, we would like to have our training distribution over states supported on the union
\[
\operatorname{supp}(\mu_\rho) \subset \bigcup_{n\ge 0} \mathbb{M}_n,
\]
reflecting the belief that the JKO trajectories of interest remain confined to (or near) a low-dimensional subset of \(\mathcal{P}(\Omega)\) over the horizons we care about.
One immediate challenge is the absence of ground-truth solution paths.
A natural approach would be to generate each $\mathbb{M}_{n}$ using existing numerical methods, but this can be both time and memory-intensive. 
We thus need a training framework that simultaneously (i) generates training rollouts approximating \(\bigcup_{n}\mathbb{M}_{n}\) and (ii) trains the solution operator by enforcing one-step JKO optimality and stability criteria, rather than relying on precomputed supervision.


In the following, we introduce a novel {\it Learn-to-Evolve} framework. This procedure ties these pieces together by iterating between operator updates via \eqref{equ:optimalmap} and self-generated rollouts over \((\rho,\beta)\sim (\mu_\rho,\mu_\beta)\)  where $\mu_\rho$ is supported on \(\bigcup_n \mathbb{M}_{n}\), optionally augmented with consistency (e.g., energy dissipation, mass preservation) and stability regularizers. The framework naturally extends to treat \(\Delta t\) as an additional input; however, for clarity of presentation, we focus here on the case of a fixed $\Delta t$.


\subsection{Iterative-and-performative training framework}

Given an initial data manifold $\mathbb{M}_0$ and a parameter space $\Lambda$,
the proposed Learn-to-Evolve framework aims to jointly learn the JKO evolution manifold and the optimal neural operator $\scrT^*$.
For ease of presentation, we assume that $\bbX$ and $\bbA$ consist of finite sample points drawn from $\mathbb{M}_0$ and $\Lambda$, respectively, with $M$ and $N$ denoting the sizes of these sampled sets.
Then the ideal training dataset $\bbD^*$ is
\begin{equation}
    \bbD^* :=
    \bigcup_{\beta \in \bbA}
    \bigcup_{\rho^0 \in \bbX}\, \{ 
     (\rho^{*,t},\beta)_{t=0}^{T} \;:\;
\rho^{*,0}=\rho^{0},\;
\rho^{*,t+1} = \JKO( \rho^{*,t}\,; \beta)\, \}.
\label{equ:truedata}
\end{equation}
containing $MN(T+1)$ sample points.
The optimal neural operator $\scrT^*$ satisfying 
\Cref{equ:neuroJKOtarget} on $\bbD^*$
is defined as $\scrT^* \in \arg\min_{\scrT} \mathcal{L}_{{\bbD^*}}(\scrT),$ with
\begin{equation}
\mathcal{L}_{{\bbD}}(\scrT) := 
\frac{1}{ |\bbD| }
\sum_{(\rho,\beta) \in {\bbD}}
~\ell(\scrT;\rho, \beta),
\end{equation}
for any given dataset $\bbD$. In practice, $\bbD^*$ is either unavailable or computationally prohibitive to obtain. To address this, we introduce the following iterative procedure that jointly learns $\bbD^*$ and the corresponding solution operator along $\bbD^*$.

The training scheme is designed to be {\it iterative} and {\it performative}. 
By iterative, 
it means that since
the learned operator will be applied repeatedly at inference time,
we generate the training dataset by iteratively applying the current neural operator $\scrT$, i.e.,
\begin{equation}
\label{equ:data}
    \bbD(\scrT) :=
    \bigcup_{\beta \in \bbA}
    \bigcup_{\rho^0 \in \bbX}\, \{ 
     (\rho^t,\beta)_{t=0}^{T} \;:\;
\rho^{0} \in \bbX,\;
\rho^{t+1} = \scrT(\rho^t,\beta)_\sharp \rho^t\, \}\,.
\end{equation}
We also assume that the chosen approximation of $\mathcal{T}$ has sufficient capacity so that, for any finite sample set $\bbD$, 
\begin{align} \label{assume}
    \scrT \in \arg \min_{\mathcal{T}} \mathcal{L}_{\bbD}(\mathcal{T}) \qquad \Longrightarrow  \qquad 
    \scrT(\rho,\beta)_\sharp \rho = \JKO(\rho\,;\beta) ~~ \text{for all}~~(\rho,\beta) \in \bbD\,.
\end{align}

The following lemma says that
when an operator achieves the minimal loss on the training data generated by itself,
it is the optimal operator we aim to learn.
The proof is by iteratively applying \Cref{assume}.

\medskip 

\begin{lemma}
Let $\scrT^*$ satisfy
\begin{equation}
    \scrT^* \in \arg \min_{\scrT} \mathcal{L}_{{\bbD(\scrT^*)}}(\scrT).
   \label{equ:optimalcondition}
\end{equation}
then $ \bbD(\scrT^*) = \bbD^*$. Consequently, $\scrT^* \in \arg \min_{\scrT} \mathcal{L}_{{\bbD^*}}(\scrT)$.
\end{lemma}

\begin{proof}
Let $\beta \in \bbA$.  
For any $\rho^0 \in \bbX$, since $(\rho^0, \beta) \in \bbD(\scrT^*)$, by combining \Cref{equ:optimalcondition} and \Cref{assume}, we obtain
\[ \displaystyle
   \scrT^*(\rho^0,\beta)_\sharp \rho^0
   = \JKO(\rho^0;\beta)
   = \rho^{*,1}.
\]
It then follows from the definitions of $\mathbb{D}(\mathcal{T}^*)$ and $\mathbb{D}^*$ that $(\rho^{,1}, \beta)$ belongs to both $\mathbb{D}(\mathcal{T}^*)$ and $\mathbb{D}^*$.  
Applying the same argument gives
\[ 
   \scrT^*(\rho^{*,1},\beta)_\sharp \rho^{*,1}
   = \JKO(\rho^{*,1};\beta)
   = \rho^{*,2}, \\
\]
and thus $(\rho^{*,2},\beta) \in \bbD(\scrT^*)$ and $(\rho^{*,2},\beta) \in \bbD^*$.  
By repeatedly applying this argument, we obtain that
\(
   (\rho^{*,t},\beta) \in \bbD(\scrT^*),
 \) 
  for all \( t=0, 1,2,\dots,T.
\)
Consequently,
\(
   \bbD(\scrT^*) = \bbD^*\),
   and
   \(
   \scrT^*(\rho,\beta)_\sharp \rho
   = \JKO(\rho;\beta), \)
    for all \( (\rho,\beta) \in \bbD^* \),
 which completes the proof.
\end{proof}

\medskip
This lemma motivates us to adopt
a {\it performative} training approach, where we alternate between generating training datasets and training the neural operator on a fixed training dataset.
This naturally leads to a fixed-point iteration of 
the right-hand side of
\eqref{equ:optimalcondition}.
Specifically, let $\scrT_k$ denote the $k$-th iterate. We fix the training data to $\bbD(\scrT_k)$ and use it to find the next iterate $\scrT_{k+1}$, i.e.,
 \begin{equation}
 \begin{aligned}     
\scrT_{k+1} \in \arg\min_{\scrT} \mathcal{L}_{{\bbD(\scrT_k)}}(\scrT). 
 \end{aligned}
 \label{equ:loss}
\end{equation}
The training scheme is summarized in Algorithm~\ref{alg:neurojko}.


\begin{algorithm}[h]
\caption{Learn-to-Evolve (Generic Framework)}\label{alg:neurojko}
\begin{algorithmic}[1]
  \Require 
          stepsize $\Delta t$, the number of JKO steps $T$, 
          initial operator $\scrT_0$
  \For{ outer iteration $k = 0,1,2,\dots$}
      \State \textbf{Step 1 (data generation):} Generate training data $\bbD(\scrT_k) $ as in \Cref{equ:data}
      \Statex   
      \State \textbf{Step 2 (training):} Update the operator by minimizing $\mathcal{L}_{\bbD(\scrT_k) }\!\bigl(\scrT \bigr)$ given in \Cref{equ:loss}: 
        \begin{equation}  \label{equ:updaterule}
        \scrT_{k+1}
          \approx
       \arg \min_{\scrT}
            \mathcal{L}_{{\bbD(\scrT_k)} }\!\bigl(\scrT\bigr),
        \end{equation}
  \EndFor
\State \textbf{Output:} the learned operator $\scrT^* := \scrT_{k+1}$
\end{algorithmic}
\end{algorithm}

For efficiency,
the minimization problem is initialized from $\scrT_k$,
and is not solved exactly but only up to an accuracy $\epsilon_k$. In the next section, we show that, under suitable Lipschitz assumptions on both the true JKO operator and the neural operator $\mathcal{T}_k$, one has when $k$ goes to infinity,
\[
\epsilon_k \rightarrow 0
\qquad \Longrightarrow  \qquad 
\bbD(\scrT_k) \rightarrow \bbD^*; 
\]
A similar performative training paradigm was proposed in \cite{perdomo2020performative}; see \Cref{sec:conclusion} for further discussion.

\medskip

\textbf{Discussion on Generalization.}
In classical learning theory, generalization typically reflects the ability to interpolate smoothly between training sample points, assuming the testing data lie within a neighborhood of the training distribution. That is, good generalization is usually expected only when the test inputs are not far from those seen during training.
However, in our experiments (see \Cref{sec:aggregeneralize}), we observe a more intriguing phenomenon: the learned solution operator exhibits meaningful extrapolation behavior, providing satisfactory trajectory from initial data that lie outside the prescribed training family~$\bbD^*$. This suggests that the neural operator may have, to some extent, learned an approximation of the underlying true operator rather than merely memorizing the training trajectories. Below, we discuss two possible reasons behind this phenomenon:

(i) The evolving nature of the training data. The progressive evolution of the training set implicitly acts as a form of data augmentation. At iteration~$k$, the training dataset $\bbD_k$ is generated from the current model state, leading to the cumulative dataset:
\[
    \bbD^{o} \;:=\; \bigcup_k \bbD_k ,
\]
where each $\bbD_k$ denotes the dataset generated at iteration $k$.
Although the model is updated using only the most recent~$\bbD_k$, its parameters retain smooth influence from earlier data through continuous training. As training progresses, $\bbD_k$ gradually approaches the target family $\bbD^*$, and $\bbD^{o}$ effectively forms a neighborhood around it. Consequently, the learned operator achieves high accuracy on~$\bbD^{*}$ while maintaining stability and smoothness in nearby regions, yielding stronger generalization than models trained solely on a fixed dataset. This evolving data exposure naturally mitigates overfitting and acts as an implicit form of data augmentation.

(ii) Inexact JKO Dynamics and Empirical Convergence.
When the trained operator is applied iteratively from an arbitrary initial distribution $\rho^0 \in \mathcal{P}(\Omega)$, it effectively realizes an inexact JKO iteration.
Although formal convergence to the equilibrium $\rho_\infty$ is not guaranteed for $\rho^0 \notin \bbD^*$, progressive training and the inherent stability  of the JKO scheme jointly promote empirical convergence.
Inexact JKO iterations are known to converge provided that the stepwise approximation error decreases along the trajectory. Empirically, our trained operator satisfies this condition when $\bbD^*$ includes the equilibrium state~$\rho_\infty$: as $\rho^t$ evolves toward $\rho_\infty$, it simultaneously enters regions closer to the training family, yielding progressively more accurate updates.
Thus, even for initial conditions not close to $\bbD^*$, repeated application of the learned operator can gradually improve accuracy along the path, effectively driving $\rho^t$ toward the equilibrium.

The observed generalization beyond the training family appears to arise from the interplay between the evolving training distribution and the inherent stability of the JKO dynamics.
The former broadens the effective data manifold, while the latter propagates learned structure toward unseen inputs in a contractive manner. Together, these mechanisms enable the operator to approximate the true underlying evolution law rather than merely interpolating within $\bbD^*$.

Although a complete theoretical characterization remains an open question, these observations highlight an interesting question: whether solution operators trained within contractive dynamical frameworks can generalize beyond the training manifold by approximating the underlying evolution law itself.  Developing a rigorous understanding of this phenomenon presents both a compelling and challenging research avenue at the intersection of operator learning, stability theory, and variational convergence.

\begin{figure}
    \centering
    \setlength{\tabcolsep}{0pt} 
    \begin{tabular}{ccccc}   
        \includegraphics[width=0.24\textwidth]{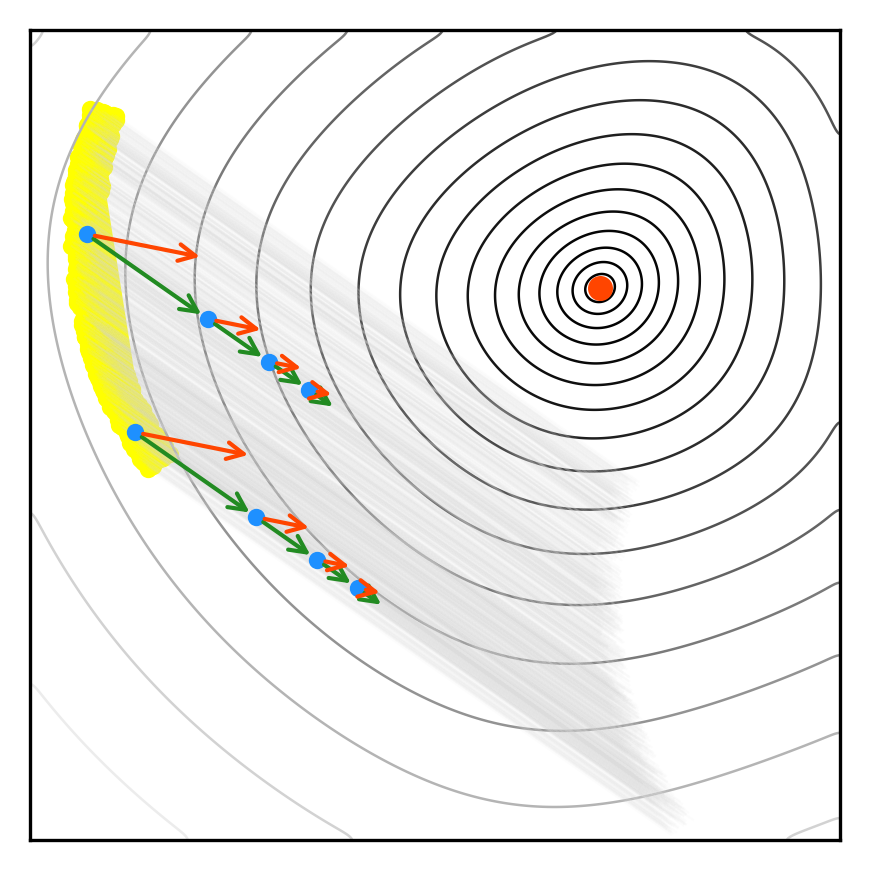} &
        \includegraphics[width=0.24\textwidth]{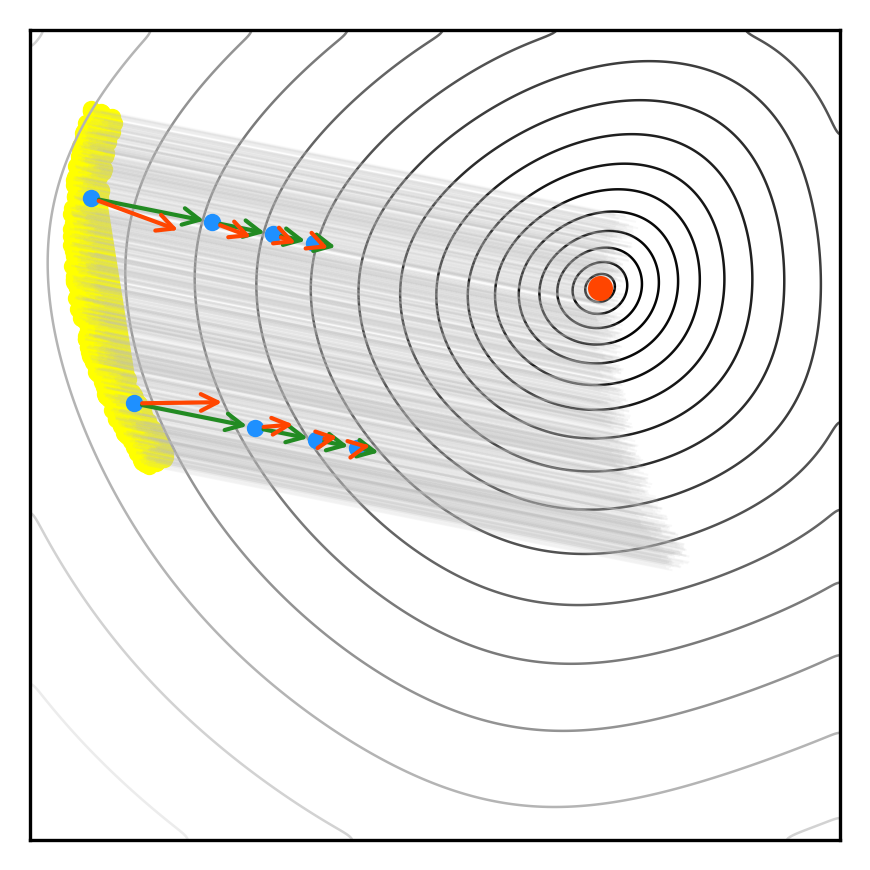} &
        \includegraphics[width=0.24\textwidth]{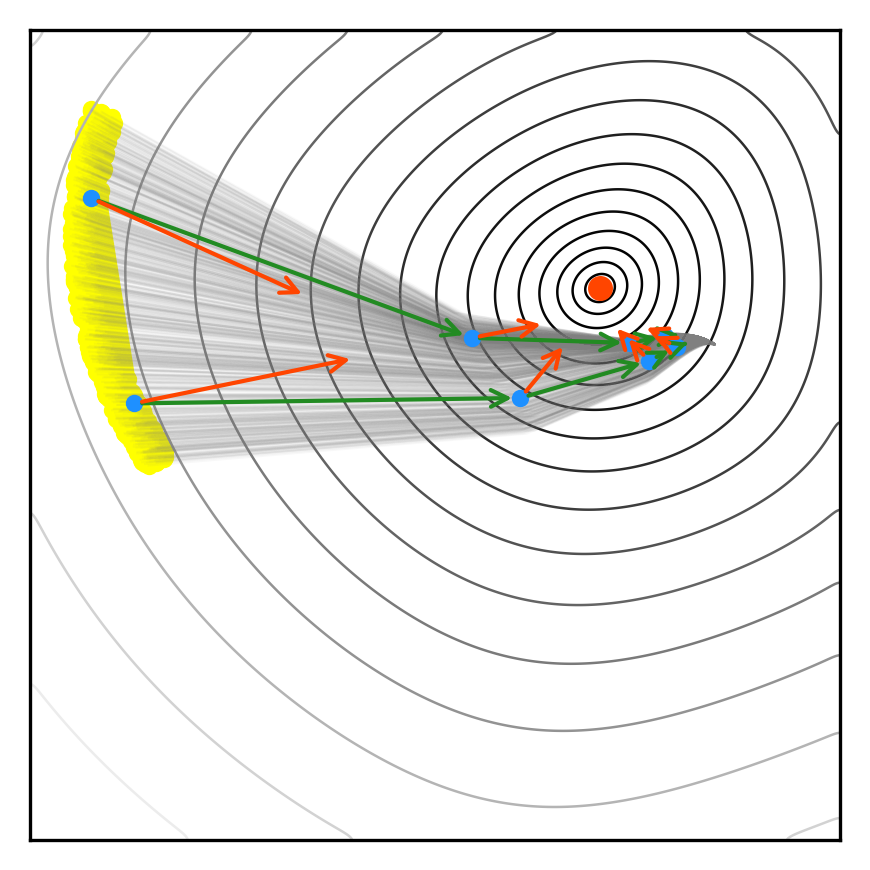} &
        \includegraphics[width=0.24\textwidth]{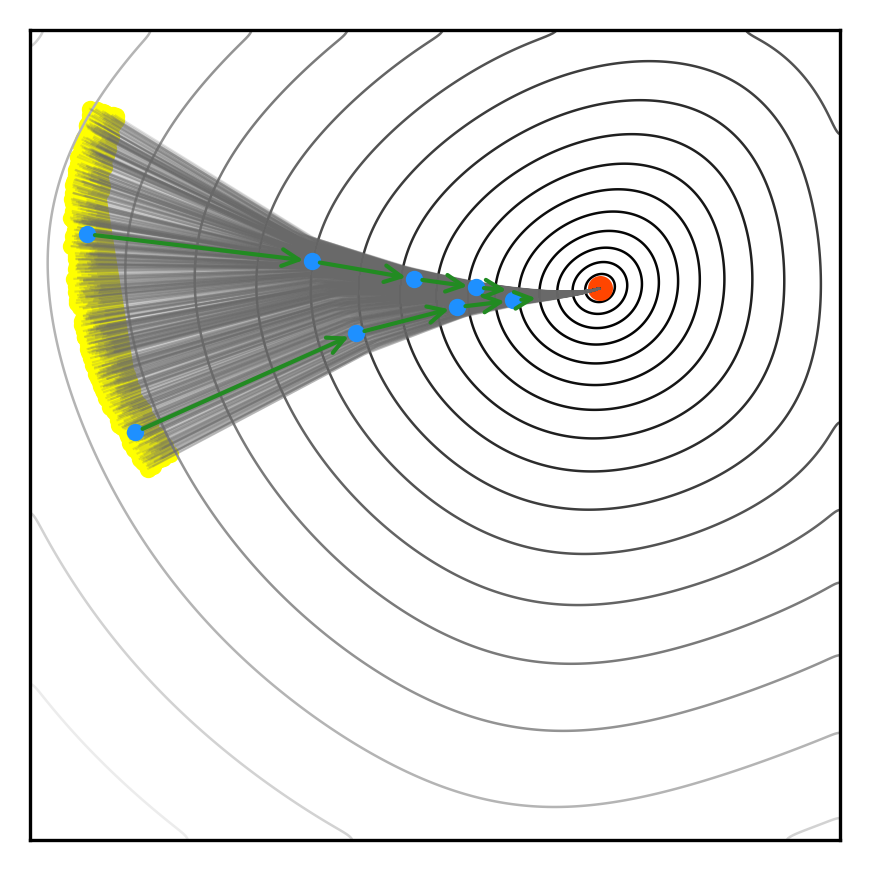} \\
      $\bbD_0$    &  $\bbD_k$ &  $\bbD_{k+1}$ &  $\bbD^*$
    \end{tabular}
    \caption{Evolution of the training dataset shown on the energy functional contours. Yellow area denotes the initial density family $\bbX$; red markers indicate equilibria; the gray region is the generated training set $\bbD_k$. For illustration, we select two landmark initial from $\bbX$ and plot the generated trajectory starting from them, shown as blue dots. Green arrows show the neural operator’s predicted directions when generating $\bbD_k$, while red arrows show the updated predictions after training on $\bbD_k$.
    During training, $\bbD_k$ gradually converges to $\bbD^*$.}
\label{fig:dataevolution}
\end{figure}

\subsection{Convergence analysis}
In this subsection, we provide the convergence analysis of \Cref{alg:neurojko}, showing that $\bbD (\scrT_k)$ progressively converges to $\bbD^*$.  We begin by stating the following assumptions on the Lipschitz continuity of both the JKO operator and the neural operator.

\begin{assumption}[Lipschitz continuity of the JKO operator]
\label{assumption:JKO}
There exists $\lambda > 0$ such that the JKO operator from $\mathcal{P}(\Omega)$ to itself  satisfies
\begin{equation} \label{927}
    \W_2(\JKO(\rho\,;\,\beta), \JKO(\nu\,;\,\beta)) 
    \;\leq\; \lambda \, \W_2(\rho, \nu),
    \quad \forall\,\, \rho,\nu \in \calP(\Omega), \, \beta \in \Lambda.
\end{equation}
\end{assumption}
Note that this condition is weaker than non-expansiveness. In classical Hilbert spaces, the proximal operator is non-expansive. In Wasserstein space, however, this property does not generally hold. There are special cases: for example, if the functional $\mathcal{E}$ is totally convex (a stronger condition than convexity along generalized geodesics), then the JKO operator is non-expansive, i.e., \eqref{927} holds with $\lambda = 1$; see Theorem 3.4 and Proposition 5.2 of \cite{cavagnari2023lagrangian}. Examples of such functionals $\mathcal{E}$ include certain linear and interaction energies. Other works \cite{jacobs20201} show that the Wasserstein proximal operator is non-expansive under the $L^1$ metric for spatially inhomogeneous free energies. Moreover, \cite{carlen2013contraction,craig2016exponential} establish bounds with an $\mathcal{O}(\Delta t)$ error term:
\(
    \W\big(\JKO(\rho), \JKO(\nu)\big) 
    \;\leq\; \W_2(\rho, \nu) + \mathcal{O}(\Delta t).
\)

\vspace{0.5cm}
Analogously, for the neural operator \(\scrT : \mathcal{P}(\Omega) \times \Lambda  \mapsto \mathrm{Map}(\Omega,\Omega) \),  we assume:

\begin{assumption} [Lipschitz continuity of the neural operator]
\label{assumption:NN}
There exists constant \(\gamma > 0\)  and a set $\mathbb{T}$ such that for all neural operator \( \scrT \in \mathbb{T} \), it satisfies
\[
      \W_2\big(  \scrT(\rho, \beta)_\sharp \rho, \scrT(\nu, \beta)_\sharp \nu \big) \leq \gamma \W_2( \rho, \nu  ),
      \qquad
\forall\,\,      \rho,\nu \in \mathcal{P}(\Omega), \beta \in \Lambda.
        \]
\end{assumption}
This assumption ensures that the neural operator does not produce drastic changes when regenerating new training data. It could be satisfied, for example, when the operator is parameterized with bounded weights and employs Lipschitz continuous activation functions such as ReLU or tanh \cite{khromov2023some,fazlyab2019efficient,gouk2021regularisation}. A similar assumption is used in Sec. 5.2 of \cite{cheng2024convergence} for the convergence analysis of flow-based generative models.

At the $k-$the iteration,
we denote 
\begin{equation}
    \bbD_{k} 
   : = \bbD (\scrT_k)
    = \bigcup_{\beta \in \bbA}
    \bigcup_{\rho^0 \in \bbX}
    \{ (    \rho^{k,t}, \beta)_{t=0}^{T} 
:   \rho^{k,0}= \rho^{0},
    \rho^{k,t+1} \!= \! \scrT_k(\rho^{k,t}, \beta)_\sharp \rho^{k,t}  \} .
    \label{equ:dataskt}
\end{equation}
Here, the superscript $k$ represents training iterates, $t$ represents the time on the trajectory. Since our training in \eqref{equ:updaterule} is not aggressively to the minimum, we denote the training error as 
\begin{equation}
        \label{equ:eacherror}
     \epsilon^{k,t}:=   \W_2\big(\scrT_{k+1}(\rho^{k,t }, \beta)_\sharp \rho^{k,t }, \JKO(\rho^{k,t }, \beta)\big)\,.
\end{equation}
When $\epsilon_k : = \max_{0\leq t \leq T} \epsilon^{k,t} $ goes to $ 0$, the transport map $\scrT_{k+1}$ achieves the JKO optimum under the  \Cref{assume}:
\[
\W_2\left( \scrT_{k+1}(\rho^{k,t}, \beta)_\sharp \rho^{k,t}, \JKO(\rho^{k,t}, \beta) \right) = 0, \quad \forall\,\, 0\leq t\leq T.
\]

\medskip
Our main result is summarized as follows. 

\begin{theorem}
Under Assumptions~\ref{assumption:JKO} and~\ref{assumption:NN}, together with \Cref{assume},
suppose  that $ \{\scrT_{k}\} \subseteq \mathbb{T} $ where $\mathbb{T}$ is defined in \Cref{assumption:NN}. 
 Then, the training dataset $ \{\bbD_{k}\} $ satisfies
 \begin{itemize}
    \item[i)] if $\epsilon_k=0$ for all $k$, then $ \{\bbD_{k} \}_k$  converges 
    to $ \bbD^{*}$
    in at most $T$ steps; 
    \item[ii)] if $\epsilon_k > 0$,  then for any given $\beta \in \bbA$ and $\rho^0 \in \bbX$,
   we have
\begin{equation}
\begin{aligned}
\,\,&\W_2(\rho^{k,t}, \rho^{*,t}  )  \\
 \leq &(\lambda + \gamma) \W_2\big( \rho^{k-1,t-1},  \rho^{*,t-1} \big) + \epsilon^{k-1,t-1}+\gamma \W_2(\rho^{k,t-1}, \rho^{*,t-1}  )\,,
\end{aligned}
\label{equ:convbound}
\end{equation}
for all $1\leq t \leq \min\{k, T\}$,
where $\rho^{*,t} $ is the exact JKO solution given in \eqref{equ:truedata}.
Therefore, 
as $k \rightarrow \infty$,
if $\epsilon^{k,s} \rightarrow 0$ for all $s< t$, then
$\rho^{k,t}$ converges to $ \rho^{*,t} $ in $\W_2$ sense.
    
\end{itemize}
\end{theorem}

\medskip

While fixed-point iteration typically requires a contraction property to establish convergence, the update rule in \eqref{equ:updaterule} does not lend itself easily to such an analysis. Nevertheless, this theorem shows that the iteration defined by \eqref{equ:updaterule} generates datasets that progressively converge to the true dataset. This result hinges on two key components. First, the initial density $\rho^0 $ is included in the training dataset $\bbD_k$ for each iteration $k$. Second, when training is error-free (i.e., $\epsilon^{k,t}$ in \eqref{equ:eacherror} vanishes), the learned operator produces the true JKO solution for each data in $\bbD_k$, as ensured by assumption \eqref{assume}. Together, these points imply the following: since $\rho^{k,0}=\rho^{0}=\rho^{*,0}$ and training introduces no error,
after the first iteration,
 the learned operator $\mathcal{T}_1$ produces
\[
\scrT_1(\rho^0, \beta)_\sharp \rho^0 = \rho^{*,1}\,.
\]
This desired data $\rho^{*,1}$ is then added to the training set and used to learn a new operator. When trained optimally at the second iteration, the new operator $\mathcal{T}_2$ satisfies
\[
\scrT_2(\rho^0, \beta)_\sharp \rho^0 = \rho^{*,1}, ~~\scrT_2(\rho^{*,1}, \beta)_\sharp \rho^{*,1} = \rho^{*,2}\,,
\]
thereby extending the number of desired data in the training dataset to the next iteration. Repeating this process yields convergence of $\bbD_k$ to $\bbD^*$.
The above argument covers the error-free case. When training errors are present, they propagate into the generated sample points and hence into subsequent operators. However, as long as these errors diminish over training iterations, the convergence still holds, as stated in part ii) of the theorem. 

The formal proof is given below.

\begin{proof}
To simplify the presentation, in the following proof, we omit the parameter 
$\beta$ in the training data. That is, instead of $(\rho,\beta)$ appearing 
in both $\bbD_k$ and $\bbD^*$, we simply write $\rho$. Moreover, we slightly 
abuse notation and denote
\(
   \scrT(\rho,\beta)_\sharp \rho \;=:\; \scrT(\rho).
\)
We recall that
by definition,
$ \rho^{k,t} =\scrT_k(\rho^{k,t-1})$ and
$ \rho^{*,t} =\JKO(\rho^{*,t-1})$.

\vspace{0.2cm}

For i), when $k=0$,
since $\epsilon_0=0$ and $\rho^0 \in \bbD_0 = \{\rho^0,\cdots\}$ ,
after training,
we have $\W_2 \big(\scrT_{1}(\rho^{0}), \JKO(\rho^{0 })\big)  =0$, 
followed by $\rho^{*,1 }=\JKO(\rho^{0 }) \in \bbD_1  = \{\rho^0, \scrT_1(\rho^0),\cdots\}$.
when $k=1$,
since  $\epsilon_1=0$ and $\rho^0, \rho^{*,1 } \in \bbD_1$ 
we have  both $\W_2 \big(\scrT_{2}(\rho^{0}), \JKO(\rho^{0 })\big)  =0$ and $\W_2 \big(\scrT_{2}(\rho^{*,1}), \JKO(\rho^{*,1 })\big)  =0$.
Since $\bbD_2 = \{\rho^0,
\scrT_{2}(\rho^{0}),
\scrT_{2}(\scrT_{2}(\rho^{0})),\cdots
\}$,
we have $\rho^{*,1 }= \scrT_{2}(\rho^{0})\in \bbD_2$,
followed by $\rho^{*,2 }=\JKO(\rho^{*,1 }) =\scrT_{2}(\rho^{*,1})\in \bbD_2$.
Recursively, we have 
$\rho^{*,t }\in \bbD_k$ for all $t\leq k$,
followed by 
$ \bbD_{k} $  converges 
    to $ \bbD^{*}$
    in at most $T$ steps.

\vspace{0.2cm}

For ii), we have the following estimate for each $\rho^{k,t} \in \bbD_k$: 
\begin{align} \label{928}
&\quad \,\, \W_2(\rho^{k,t}, \rho^{*,t}  ) \nonumber \\
&= \W_2 \big( \scrT_k(\rho^{k,t-1}),  \rho^{*,t} \big) \nonumber \\
&\leq 
\W_2 \big( \scrT_k(\rho^{k,t-1}),  \scrT_k(\rho^{k-1,t-1}) \big) +\W_2 \big( \scrT_k(\rho^{k-1,t-1}),  \rho^{*,t} \big)  \nonumber \\
&\leq \gamma \W_2 \big( \rho^{k,t-1}, \rho^{k-1,t-1} \big) +\W_2 \big( \scrT_k(\rho^{k-1,t-1}),  \rho^{*,t} \big)\,,
\end{align}
where the last inequality uses the Assumption~\ref{assumption:NN}. Then the first term will be further bounded by the triangle inequality
\begin{align} \label{930}
 \W_2 \big( \rho^{k,t-1}, \rho^{k-1,t-1} \big) 
 \leq \W_2 \big( \rho^{k-1,t-1},  \rho^{*,t-1} \big) +\W_2 ( \rho^{*,t-1},  \rho^{k,t-1} )
\end{align}
which traces the error at $(k,t)$ back to both the previous iteration and time: $(k-1,t-1)$ and $(k, t-1)$. 

The second term in \eqref{928} admits the following estimate:
\begin{align} \label{929}
&\quad  \W_2 \big( \scrT_k(\rho^{k-1,t-1}),  \rho^{*,t} \big)  \nonumber 
\\ & = \W_2 (\scrT_k(\rho^{k-1,t-1}),  \JKO(\rho^{*,t-1}) ) \nonumber 
\\ & \leq \W_2 (\scrT_k(\rho^{k-1,t-1}),  \JKO(\rho^{k-1,t-1}) ) + 
\W_2 (\JKO(\rho^{k-1,t-1}),  \JKO(\rho^{*,t}) ) \nonumber 
\\ & \leq \epsilon^{k-1,t-1} + \lambda \W_2(\rho^{k-1,t-1}, \rho^{*,t})\,.
\end{align}
Here, the first equality uses  $ \rho^{*,t} =\JKO(\rho^{*,t-1})$, and the last inequality uses definition of the error $\epsilon^{k-1,t-1} $ given in \Cref{equ:eacherror} and \Cref{assumption:JKO}. Substituting \eqref{930} and \eqref{929} into \eqref{928} leads to the result \eqref{equ:convbound}. 
\medskip

If  $\epsilon^{k,0} \rightarrow 0$ as $k \rightarrow \infty$, we have $\W_2(\rho^{k,1}, \rho^{*,1}) \rightarrow 0$ by the definiton  of $\epsilon^{k,0}$. From \eqref{equ:convbound}, it then follows that if $\epsilon^{k,1} \to 0$,
    \[
    \W_2(\rho^{k,1}, \rho^{*,1}  )  
 \leq (\lambda + \gamma) \W_2\big( \rho^{k-1,0},  \rho^{*,0} \big) + \epsilon^{k-1,0}+\gamma \W_2(\rho^{k,0}, \rho^{*,0}  ) \rightarrow 0.
    \]
Likewise, if $\epsilon^{k,2} \rightarrow 0$,
    \[
    \W_2(\rho^{k,2}, \rho^{*,2}  )  
 \leq (\lambda + \gamma) \W_2\big( \rho^{k-1,1},  \rho^{*,1} \big) + \epsilon^{k-1,1}+\gamma \W_2(\rho^{k,1}, \rho^{*,1}  ) \rightarrow 0
    \]
    The same argument can then be repeated for later times,
    which finishes the proof.
\end{proof}

\begin{remark}
    In practice, since the error of each JKO step is carried forward, achieving a low training error for early time $t$ benefits subsequent steps. Hence, it is natural to use a decay factor $0<\beta\leq 1$ to adjust the importance of generated densities in the loss function as $\sum_t \beta^t \ell(\scrT;\rho^{k,t})$. The use of the decay factor is also seen in a similar work \cite{wang2019prnet,li2022learning}. 
\end{remark}

\subsection{Training strategies}
\label{sec:trainingstrategy}

In this section, we discuss various training strategies. In particular, we propose a method termed the better-than-birth strategy, which effectively balances data generation and training in \Cref{alg:neurojko}.

The update in \Cref{equ:updaterule} is generic, where standard optimizers (e.g., Adam or SGD) can be used to solve the minimization problem.
For efficiency,
to obtain $\scrT_{k+1}$,
the minimization is initialized from $\scrT_k$ to avoid training from scratch each time.
Here we emphasize that since the training data $\bbD(\scrT_k)$
is also generated by $\scrT_k$,
it needs to be {\it detached } from $\scrT_k$ to prevent gradients from flowing back through the data generation process during differentiation. 
The key is then to decide when to generate the new training dataset, equivalently,
when to stop the inner operator update in \Cref{equ:updaterule}.

We refer to the \textit{greedy-training strategy} as the case where the minimization problem is solved to its minimum at each outer iteration. The drawback of this approach is that, due to the non-convex nature of neural network training, the loss may decrease rapidly in the initial stage but then stagnate, making it difficult, even impossible, to reach the global minimizer. In our setting, this difficulty is amplified, as solving for the global minimizer is required $T$ times.

 In contrast, the {\it aggressive training} strategy is designed to alternate between a single inner update and new data generation, thereby updating the training data at high frequency.
For some problems, such as minimizing KL divergence with simple target distributions, 
this strategy quickly produces a reasonably accurate JKO trajectory. However, due to the non-convex nature of neural network training, even well-performing models may regress after parameter updates, degrading the generated data. 
Especially, 
errors in early JKO steps can propagate to later ones such that the 
generated training data diverge significantly, making overall convergence difficult—an effect particularly evident in porous medium problems. Nonetheless, this strategy is useful in the early stages of training, as it rapidly warms up the model to a coarse but reasonable state.

To further improve the training, we introduce a new strategy, called {\it better-then-birth} strategy, that generates new training data when the neural operator becomes ``better'' than the operator previously used for data generation, thereby ensuring progressive improvement in the training data.
Intuitively, for a fixed training data, the operator is considered ``better'' when its loss is smaller.
However, inspired by \Cref{equ:convbound}, we introduce a new indicator: the \textit{accumulated loss}.

Specifically,
at iteration $k$, we first sample
a batch of initial and parameter pairs as $\mathbb{B} = \{(\rho^{[b],0},\beta^{[b]} )\}_{b=1}^B \subseteq \bbX \times \bbA$ where $B$ is the batch size.
Following \Cref{equ:dataskt},
the training dataset  $\bbD_k$ is  as follows
\begin{equation}
 \bbD_k=   \bbD(\scrT_k;\mathbb{B}) = 
    \bigl\{ \,
   (\rho^{[b],t}, \beta^{[b]})_{t=0}^T
   \;\bigr\}_{b=1}^B
    \label{equ:batchdata}
\end{equation}
with $\rho^{[b],t+1} = \scrT_k(\rho^{[b],t},\beta)_\sharp \rho^{[b],t}$.
For the given training data $\bbD_k$, the accumulated loss of an arbitrary operator $\scrT$ is defined as
\begin{equation}
\label{equ:indicator}
   L\big(\scrT;\bbD_k \big) \in \mathbb{R}^{B\times (T+1)},
\qquad
L\big(\scrT;\bbD_k \big)_{b,t} = \sum_{s=0}^t \ell(\scrT; \rho^{[b],s}, \beta^{[b]}). 
\end{equation}
During the minimization,
if the accumulated loss  $L\big(\scrT;\bbD_k\big)$ of the current iterate $\scrT$ is smaller than $L\big(\scrT_k;\bbD_k \big)$  element-wise, $\scrT$ is regarded as superior to $\scrT_k$.
Thus, the better-then-birth strategy stops the inner update in \Cref{equ:updaterule} when $L\big(\scrT;\bbD_k\big)\leq L\big(\scrT_k;\bbD_k \big)$, 
and 
use the new operator for data generation in the next iteration.

This indicator emphasizes achieving low error in previous time steps. Meanwhile,
to prevent training from stalling, we also set a maximum number of inner update $S_{\mathrm{in}}$ for each outer iteration. 
Furthermore,  
to control the overall training cost in practice,, we set a global budget $S_{\max}$ that specifies the total number of inner updates executed over all outer iterations. 
The effectiveness of this strategy to generate accurate JKO trajectories has been confirmed by the numerical examples in \Cref{sec:numerical}.
Furthermore,
under this strategy,
if $\mathbb{B}=\bbX \times \bbA$, 
it is easy to verify that 
each element in
the accumulated losses $L\big(\scrT_k;\bbD_k \big)$ is decreasing monotonically in terms of $k$. 
The full procedure is summarized in Algorithm~\ref{alg:neurojko2stage},
where an optional aggressive training strategy can be applied for the first $K_0$ iteration, and it corresponds to setting $S_{\mathrm{in}}=1$.

\begin{algorithm}[ht]
\caption{Learn-to-Evolve Algorithm (better-than-birth strategy)}
\label{alg:neurojko2stage}
\begin{algorithmic}[1]
\Require 
stepsize $\Delta t$, number of JKO steps $T$, batch size $B$,
initial operator $\mathcal{T}_0$,
maximum inner steps per outer iteration $S_{\mathrm{in}}$,
maximum total inner steps $S_{\max}$

\State Initialize operator $\mathcal{T} \gets \mathcal{T}_0$
\State Initialize global inner-step counter $S \gets 0$

\For{outer iteration $k = 0,1,2,\dots$}

    \State \textbf{Step 1: Data generation}
    \State \hspace{1em} Sample a batch of initial density--parameter pairs from $\mathbb{X} \times \mathbb{A}$
    \State \hspace{1em} Generate the dataset $\mathbb{D}_k = \mathbb{D}(\mathcal{T}; B)$ as in \Cref{equ:batchdata}
    \State \hspace{1em} Detach $\mathbb{D}_k$ from the current NN parameters
    \State \hspace{1em} Compute reference loss $L_k = L(\mathcal{T}; \mathbb{D}_k)$ as in \Cref{equ:indicator}

    \vspace{0.5em}
    \State \textbf{Step 2: Operator update}

    \State Initialize inner-step counter $s \gets 0$

    \While{
        $L(\mathcal{T}; \mathbb{D}_k) \not< L_k$ (element-wise)
        \textbf{ and } $s < S_{\mathrm{in}}$
        \textbf{ and } $S < S_{\max}$
    }

        \State Update operator parameters:
        \(
        \mathcal{T} \gets \mathcal{T}
        - \eta \nabla_{\mathcal{T}} \mathcal{L}_{\mathbb{D}_k}(\mathcal{T})
        \)
        \Comment{or Adam}

        \State Increment counters:
        $s \gets s + 1$ \emph{(local)},
        $S \gets S + 1$ \emph{(global)}

    \EndWhile

    \State Set $\mathcal{T}_{k+1} \gets \mathcal{T}$

    \If{$S \ge S_{\max}$}
        \State \textbf{stop outer iterations}
    \EndIf

\EndFor

\vspace{2pt}
\State \textbf{Output:} the learned operator $\mathcal{T}^* := \mathcal{T}$

\State (\textbf{Optional warm-up phase}: Apply the aggressive training strategy during the first $K_0$ outer iterations.)

\end{algorithmic}
\end{algorithm}

 \subsection{Related Works}

The proposed learn-to-evolve algorithm is closely related to recent work on
\emph{performative prediction}~\cite{perdomo2020performative}, where predictions
influence the outcomes they aim to model. In this setting, a model with
parameters $\theta$ is evaluated on the induced distribution $\mathcal{D}(\theta)$,
and the loss takes the form
\(
\mathbb{E}_{Z \sim \mathcal{D}(\theta)} \, \ell(Z ; \theta).
\)
A standard algorithmic framework in this field is
\emph{repeated risk minimization}, which iteratively trains models on
distributions induced by the previous iterate, which is similar to the greedy strategy we have proposed. Variants such as lazy updates and
their theoretical guarantees can be found in
\cite{perdomo2020performative, mofakhami2023performative, hardt2023performative}.
The key difference from our setting is that, in performative prediction,
training data are sampled directly from the induced distribution
$\mathcal{D}(\theta)$, whereas in our case the training data are generated by
iteratively applying the current network as described in
\Cref{alg:neurojko}. This difference is substantial: convergence analyses in
performative prediction typically rely on the $\epsilon$-sensitivity assumption
\[
\W_1\!\left(\mathcal{D}(\theta), \mathcal{D}(\theta')\right)
   \leq \varepsilon \, \|\theta - \theta'\|_2,
\]
which does not hold in our case. Because our data are propagated through
the network up to $T$ times, even small perturbations of network parameters
may be amplified, producing significantly different distributions.

Moreover, the idea of boosting training data by repeatedly applying the current network
also appears in \cite{wang2019prnet}, and has been used in the
context of learning proximal operators in Euclidean spaces
\cite{li2022learning}. The perspective of \cite{li2022learning}, however, treats the procedure as an
importance-sampling heuristic that generates data near
near-optimal solutions, and train a universal proximal operator. In contrast, our focus is to design algorithms 
balancing data generation and network training, as well as guaranteeing
the convergence of the training dataset.

\section{Particle-Based Discretization and Implementation}
\label{sec:Implementation}

This section discusses the implementation details of the training of the Neural JKO operator using \Cref{alg:neurojko2stage},
including 
the computation of the loss function,
the generation of the training data
and neural network architecture.
For clarity of presentation, the energy functional parameter $\beta$ is omitted in the following discussion. The method for incorporating $\beta$ as an input to the neural network is described in the numerical results section.

Given an energy functional $\E$, 
the minimization in \eqref{equ:jkoOperator} can be reformulated in terms of transport maps as follows: for $\rho \in \calP(\Omega)$,
\begin{equation}
    \bfT^* = \arg\min_{\bfT:\Omega \to \Omega} \int_{\Omega} |\bfT(x) - x|^2 \,\rmd \rho(x)  
    + 2\Delta t\, \mathcal{E}( \bfT_{\#} \rho ),
\end{equation}
where $\bfT^*$ is the optimal map between $\rho$ and $\JKO(\rho)$, satisfying
\((\bfT^*)_{\#}\rho = \JKO(\rho)\).  
The \emph{JKO velocity field} (more precisely, the JKO displacement field) is then defined by
\[
\bfV_{\JKO}(\rho) := \bfT^*- \bfI,
\qquad
\bfV_{\JKO}(\rho) :\Omega \to \bbR^d.
\]
In this work, we use a neural network $\bfNN_{\theta}:\calP(\Omega)  \times \Omega \to \bbR^d$ to parameterize the JKO velocity field,
such that $(\bfI + \bfNN_{\theta})(\rho)= \scrT(\rho)$ in the previous section.

In general, there is no closed-form expression for $\bfV_{\JKO}$.  
However, when the step size $\Delta t$ is sufficiently small, $\bfV_{\JKO}$ approximates
the forward Euler velocity field
$\bfV_{\mathrm{FE}} = -\Delta t\, \nabla \frac{\delta \mathcal{E}}{\delta \rho}(\rho)$ with an $\mathcal{O}(\Delta t)$ error (Theorem 4.0.4 of \cite{ambrosio2008gradient},  as well as section SM1 in the supplement), which can  be used to assess the accuracy of the predicted 
$\bfV_{\JKO}$.

\subsection{Particle-based loss computation}

In practice, the loss function in \Cref{alg:neurojko2stage} 
is the empirical sum over the training dataset $\bbD_k$:
\begin{equation}
    \mathcal{L}_{\bbD_k}(\theta) 
= \frac{1}{|\bbD_k|} \sum_{\rho \in \bbD_k} \ell(\theta;\rho) =  \frac{1}{|\bbD_k|}
 \sum_{\rho \in \bbD_k} \W^2_2\left(  \bfT_{\sharp}(\rho), \rho\right) + 2\Delta t\, \mathcal{E}\left(\bfT_{\sharp}(\rho)\right),
\end{equation}
where $\bfT = (\bfI + \bfNN_{\theta})(\rho): \Omega \mapsto \Omega$ is the parameterized at $\rho$.
We consider the energy functional as the sum of internal energy, external  energy, and interaction energy, given by:
\[
\mathcal{E}(\rho) 
= \int_{\Omega}  \mathcal{U}(\rho(x))\,\rmd x 
+ \int_{\Omega}  \mathcal{V}(x)\rho(x) \,\rmd x 
+ \frac{1}{2} \int_{\Omega \times \Omega} \mathcal{K}(x - y) \rho(x) \rho(y) \,\rmd x\,\rmd y\,.
\]
where the external and interaction energy can be straightforwardly approximated using Monte Carlo method. The main computational challenge lies in evaluating the internal energy term, which we address by adopting the approach proposed in \cite{lee2024deep}.

More precisely, for a given $\rho$, we have
\(
    \W^2_2\left(  \bfT_{\sharp}(\rho), \rho\right)  
    = \int_\Omega \| x - \bfT(x) \|^2\, \rho(x) \,\rmd x,
\)
and
\begin{equation}
\begin{aligned}
\mathcal{E}\!\left(\bfT_{\sharp}\rho\right)
&= \int_{\Omega} \mathcal{U}\!\big( \bfT_{\sharp}\rho(x) \big)\,\rmd x
  + \int_{\Omega} \mathcal{V}(x)\,\bfT_{\sharp}\rho(x)\,\rmd x
  + \tfrac12 \!\int_{\Omega \times \Omega} \mathcal{K}(x-y)\,\bfT_{\sharp}\rho(x)\,\bfT_{\sharp}\rho(y)\,\rmd x\,\rmd y \\
&= \int_{\Omega}
    \mathcal{U}\!\left( \frac{\rho(x)}{|\det \nabla_x \bfT(x)|} \right)
    \frac{|\det \nabla_x \bfT(x)|}{\rho(x)} \rho(x)\rmd x
  + \int_{\Omega} \mathcal{V}\!\big(\bfT(x)\big)\rho(x)\,\rmd x \\
  &\qquad +\,\, \tfrac12 \!\int_{\Omega \times \Omega} \mathcal{K}\!\big(\bfT(x)-\bfT(y)\big)\rho(x)\rho(y)\rmd x\rmd y \\
&= \mathbb{E}_{x\sim\rho}
    \left[\mathcal{U}\!\left( \frac{\rho(x)}{|\det \nabla_x \bfT(x)|} \right)
    \frac{|\det \nabla_x \bfT(x)|}{\rho(x)} \right]
  + \mathbb{E}_{x\sim\rho}\left[\mathcal{V}\!\big(\bfT(x)\big)\right] \\
  &\qquad +\,\, \tfrac12~ \!\mathbb{E}_{x\sim\rho,y\sim\rho} \left[\mathcal{K}\!\big(\bfT(x)-\bfT(y)\big)\right]
\end{aligned}
\end{equation}
where we have used the relation  $\bfT_{\sharp}p(\bfT(x)) = p(x)/|\mathrm{det} \nabla_x \bfT(x)| $.
The main computational bottleneck is evaluating $|\det \nabla_x \bfT(x)|$, which incurs cubic time complexity with respect to the state dimension. To mitigate this cost, we approximate the log-determinant using the divergence of the predicted velocity field:
\begin{equation}
   |\det \nabla_x \bfT(x)| \approx 
 e^{\mathrm{div}\,\bfNN_{\theta}(\rho)(x)}. 
 \label{equ:detapprox}
\end{equation}
This approximation corresponds to the first-order Taylor expansion of the exact log-determinant.
The following lemma establishes that the resulting approximation error is of order $\mathcal{O}(\Delta t^2)$, and a formal proof is given in Appendix B. The divergence is computed using finite difference methods.

\begin{lemma}
Let $\bfT: \Omega \to \Omega$ be invertible with $\|\bfT(x) -x\| = \mathcal{O}(\Delta t)$ and $\nabla_x \bfT(x)$ positive definite.  
Setting $\bfV = \bfT - \bfI$, we have
\begin{equation}
    \label{equ:diverrorbounf}
      \big|\,|\det \nabla_x \bfT(x)| - e^{\mathrm{div}\,\bfV(x)}\,\big| 
   \leq \mathcal{O}(\Delta t^2), \quad \forall\,\, x\in \Omega.
\end{equation}
\end{lemma}

\medskip
Finally, we approximate the integrals using a particle method, i.e., by Monte Carlo estimation over sample points $\{x_i\}_{i=1}^m \sim \rho$:
\begin{equation}
\begin{aligned}
       \ell(\theta; \rho) 
    \approx \frac{1}{m} \sum_{i=1}^m
\|x_i-\bfT(x_i)\|^2 
+ 2\Delta t \, \bigg[ &\mathcal{U}\!\left( \frac{\rho(x_i)}{|\det \nabla_x \bfT(x_i)|} \right) 
\frac{|\det \nabla_x \bfT(x_i)|}{\rho(x_i)} \\
&+ \mathcal{V}(\bfT(x_i)) 
+ \frac{1}{m} \sum_{i=1}^m \mathcal{K}(\bfT(x_i), \bfT(x_j)) 
\bigg].
\end{aligned}
\label{equ:discreteloss}
\end{equation}
For a more detailed discussion of the loss function, we refer the reader to \cite{lee2024deep}.

\subsection{Density generation}

Following \Cref{alg:neurojko2stage},
in each outer iteration $k$,
we sample a batch of initial density–parameter pairs, and
generate the dataset $\bbD_k$ as in \Cref{equ:batchdata}.

For each $\rho \in \bbD_k$,  
we need both sample points $x_i \sim \rho$ and the corresponding density values $\rho(x_i)$.  
When $t=0$, $\rho^0$ is given explicitly and these are computed directly.  
For later times, we use the fact that if $\{ x_i\}$ are i.i.d. sample points from $\rho$, then $\{ \bfT(x_i)\}$ are i.i.d. from $\bfT_{\sharp}\rho$, with updated densities given by
\begin{equation}
    \label{equ:densityupdate}
    \rho^{t+1}(\bfT(x_i)) = \frac{\rho^{t}(x_i)}{|\det \nabla_x \bfT(x_i)|},
    \qquad
    \text{with}\quad
    \bfT = (\bfI + \bfNN_{\theta})(\rho^t): \Omega \mapsto \Omega.
\end{equation}
The complete algorithm to generate the training dataset is given in \Cref{algo:datageneration}.

\begin{algorithm}[t]
\caption{Generation of densities in $\bbD_k$}
\label{algo:datageneration}
\begin{algorithmic}[1]
    \Require Number of particles $m$; a batch of initial densities $\{\rho^0\}$; current network $\bfNN_{\theta_k}$
    \For{each initial density $\rho^0$}
        \State Sample $\{x_i^0\}_{i=1}^m \sim \rho^0$ and evaluate $\rho^0(x_i^0)$
        \For{$t = 0, \dots, T-1$}
            \State Forward the NN to obtain the velocity field $\bfV = \bfNN_{\theta_k}(\{x_i^t\}_{i=1}^m)$
            \State Update particle positions: $x_i^{t+1} = x_i^t + \bfV(x_i^t)$
            \State Update densities:
            \[
                \rho^{t+1}(x_i^{t+1}) = \rho^{t}(x_i^t)\, e^{-\operatorname{div}\bfV(x_i^t)}
            \]
        \EndFor
        \State Detach all generated sample points and density values from the computation graph   
    \EndFor     
\State \textbf{Output:} the training dataset $\bbD_k$ 
\end{algorithmic}
\end{algorithm}

\subsection{Neural Network Architecture}

The network architecture is a key factor in achieving accurate solutions.  
Since each density is represented by random sample points, the NN must take a set of sample points as input and output the corresponding velocity field values.
The permutation invariant nature of the sample points is challenging for traditional  architectures used in operator learning, such as deepONet\cite{lu2021learning} and FNO\cite{li2023fourier}.
In this work, we adopt the \emph{Transformer}~\cite{vaswani2017attention}, which has shown outstanding performance in diverse applications~\cite{liu2024prose,yang2023context}. Our design is inspired by \cite{huang2024unsupervised}, which trains  a transformer-based mean field game (MFG) solution operator.

Given a density $\rho\in\mathcal{P}(\mathbb{R}^d)$, we sample $m$ points $\{x_i \in \mathbb{R}^d\}_{i=1}^m$ and concatenate them into a matrix $X \in \mathbb{R}^{m \times d}$, serving as the transformer input, i.e., the prompts.
The network outputs $\{v_i\}_{i=1}^m$, representing the velocity field (more precisely, the displacement field) at each sample location.  
In the encoder part, we adopt \emph{self-attention} layers to capture nonlocal interactions among the sampled points, allowing the model to encode the global geometric and density correlations of $\rho$. 
In the decoder part,
we incorporate \emph{cross-attention} layers, which enable the query location to align with the encoded density features and 
improve the computation of the divergence operator.

This NN design has several other benefits.
First, $\bfV_{\JKO}(\rho)$ 
depends on the global structure of $\rho$,
making simple architectures such as pointwise MLPs insufficient.
On the contrary, the transformer’s core component, the multi-head attention (MHA) mechanism, naturally captures nonlocal interactions.  
It can be interpreted as trainable kernel integrators~\cite{shen2025understanding}, aligning well with the structure of the JKO operator.  
More specifically, after an initial pointwise MLP lifts $X$ into a $h$-dimensional  embedding space, MHA computes interactions between queries $Q = XW^{Q} \in \mathbb{R}^{m \times h}$, keys $K = XW^{K} \in \mathbb{R}^{m \times h}$, and values $V = XW^{V} \in \mathbb{R}^{m \times h}$:
\[
\mathrm{Attn}(Q, K, V) = \operatorname{softmax}\!\left(\frac{QK^{\top}}{\sqrt{h}}\right)V,
\]
where $W^Q, W^K, W^V \in \mathbb{R}^{h \times h}$ are learnable parameters.  
This formulation allows the network to dynamically weight the relevance of each sample point, capturing both long-range dependencies and fine-grained local structures.  
Multi-head attention extends this capability by learning diverse interaction patterns across parallel heads.

Secondly, 
our loss involves computing $\mathrm{div} \bfV_{\JKO}$ using finite differences, which requires velocity evaluations at perturbed points $\{x_i \pm \epsilon\}$.  
Using self-attention here can suffer from distribution shift, since the query points differ from the original sample points.  
Cross-attention addresses this: to compute the velocity at an arbitrary $x \in \Omega$, the query embedding $q = xW^{Q} \in \mathbb{R}^{1 \times l}$ interacts with keys and values derived from $X$.  
This ensures accurate predictions even for out-of-sample query locations. 

Additional advantages of applying Transformers include:
(i) their natural ability to handle variable-length inputs, allowing flexible discretizations of $\rho$ without retraining;  
(ii) sampling invariance when $m$ is large enough, as shown in \cite{huang2024unsupervised}, so that the learned discrete operator converges to the continuous one.

An overview of the architecture is shown in \Cref{fig:transformer_structure}.
For all experiments in the numerical section, the embedding dimension $h$ is set to $1024$.
The number of sample points $m$ we draw from each density is fixed as $1024$, except for the KL example where $m$ is set to $512$ for both the initial and target distributions.

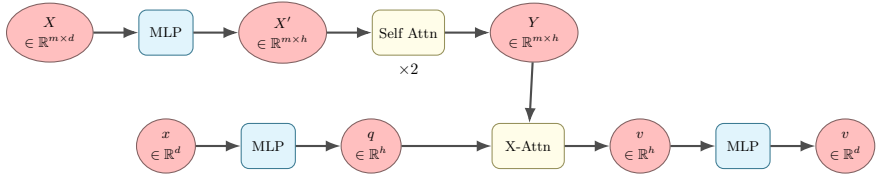
\begin{figure}
    \centering
\begin{tikzpicture}[scale=0.6, transform shape,
                    node distance=1.4cm and 1cm]

\node[nodeStyle] (X) {$X$\\[2pt]$\in \mathbb{R}^{m \times d}$};
\node[op, right=of X] (MLP1) {MLP};
\node[nodeStyle, right=of MLP1] (X0) {$X'$\\[2pt]$\in \mathbb{R}^{m \times h}$};
\node[op2, right=of X0] (MHT) {Self Attn};
\node[below=2pt of MHT] {\(\times 2\)};
\node[nodeStyle, right=of MHT] (Y) {$Y$\\[2pt]$\in \mathbb{R}^{m \times h}$};

\node[nodeStyle, below=of MLP1] (x){$x$\\[2pt]$\in \mathbb{R}^{ d}$};
\node[op, right=of x] (MLP2) {MLP};
\node[nodeStyle, right=of MLP2] (q) {$q$\\[2pt]$\in \mathbb{R}^{ h}$};

\node[op2, right=2cm of q] (Cross) {X-Attn};
\node[nodeStyle, right=of Cross] (v1) {$v$\\[2pt]$\in \mathbb{R}^{ h}$};
\node[op, right=of v1] (MLP3) {MLP};
\node[nodeStyle, right=of MLP3] (v){$v$\\[2pt]$\in \mathbb{R}^{ d}$};

\draw[arrow] (X) -- (MLP1);
\draw[arrow] (MLP1) -- (X0);
\draw[arrow] (X0) -- (MHT);
\draw[arrow] (MHT) -- (Y);
\draw[arrow] (Y) -- (Cross);

\draw[arrow] (x) -- (MLP2);
\draw[arrow] (MLP2) -- (q);
\draw[arrow] (q) -- (Cross);

\draw[arrow] (Cross) -- (v1);
\draw[arrow] (v1) -- (MLP3);
\draw[arrow] (MLP3) -- (v);


\end{tikzpicture}

\caption{
Illustration of the proposed NN architecture where the output \(v = NN_{\theta}(\rho)(x)\) denotes the velocity at a query point $x$.  
Here, a pointwise Multi-Layer Perceptron (MLP) is applied to lift the sample point from the physical dimension $d$ to the embedding dimension $h$, or vice versa.
``Self Attn'' and ``X-Attn'' refer to self-attention and cross-attention blocks respectively.  
The upper panel represents the encoder, in which the density $\rho$, represented by the sampled points $\{x_i\}_{i=1}^m$, serves as the prompts. 
One may concatenate the density values or the energy function parameters to each of the sample points. 
The lower panel depicts the decoder. 
In the KL divergence case, the target distribution $\rho_{\text{target}}$, represented by $\{y_i\}_{i=1}^{m}$, is introduced as an additional input; it is processed by a separate MLP to lift its dimension before concatenation with $X'$.
}

    \label{fig:transformer_structure}
\end{figure}

\section{Numerical Examples}
\label{sec:numerical}

We evaluate the proposed scheme on three representative gradient flow models:

\begin{enumerate}
    \item The \emph{aggregation equation} with attraction-repulsion kernel.  The interaction energy is parameterized, leading to qualitatively different equilibrium types. Accordingly, we extend the operator input to include both the density and the corresponding parameter pair, enabling a single learned operator to predict the evolution and equilibrium type across a family of aggregation dynamics.

\item The \emph{porous medium} equation.  We train separate solution operators for each choice of parameters in the energy function, such as the spatial dimension and the nonlinearity exponent. In special cases, the availability of exact Barenblatt solutions, which correspond to exact continuous-time Wasserstein gradient flow solutions of the porous medium energy, 
enables a direct quantitative comparison between the predicted flow and the analytic solution, thereby providing a precise assessment of the accuracy of the learned operators.

\item The \emph{Fokker-Planck} equation. The KL energy can be viewed as parameterized by the target distribution. Here, the target is represented through its sample points and incorporated into the operator input, allowing the learned model to infer the target distribution and predict the KL gradient flow between different initial and target densities.
\end{enumerate}
Together, these experiments demonstrate that the proposed method learns a flexible solution operator capable of handling parameterized energies and data-dependent targets, rather than being restricted to a fixed gradient flow.
The numerical results demonstrate that our method is not only computationally efficient but also achieves high accuracy and exhibits strong generalization. 

Furthermore, we show additional experiments on the aggregation problem in a simplified setting. These experiments include visualizations of the evolution of the training data, which offer further insight into the behavior of the proposed algorithm. In addition, we compare the generalization performance of the solution operator trained using our method with that of a baseline model trained on a fixed dataset, and demonstrate that our approach achieves consistently better generalization performance.

Training for all models follows the procedure described in \Cref{alg:neurojko2stage}. 
We apply a warm-up phase of $K_0 = 200$ outer iterations (line~21) and set the maximum number of inner steps per outer iteration to $S_{\mathrm{in}} = 50$, unless stated otherwise. 
Instead of fixing the number of outer iterations, the training process is controlled by fixing the maximum total number of inner steps $S_{\max}$ across all outer iterations, which is set to $20{,}000$, $10{,}000$, and $10{,}000$ for the three experiments accordingly. At each outer iteration, a new batch of initial densities is drawn from the initial family $\mathbb{M}_0$ with newly generated sample points, as described in lines 3–4 of \Cref{alg:neurojko2stage}. The batch size $B$ depends on the specific problem and will be specified in the corresponding experiment section.

For the aggregation problem, we use Adam with a weight decay of $10^{-4}$ and a cosine-annealing learning-rate schedule ranging from $10^{-4}$ to $10^{-5}$.
For the porous-medium and Fokker–Planck equations, a custom learning-rate scheduler is employed with four phases: (i) linear decay from $10^{-4}$ to $10^{-5}$ over the first $3000$ inner updates, (ii) constant $10^{-5}$ until inner update $6000$, (iii) linear decay to $10^{-6}$ by $10000$ inner update, (iv) constant $10^{-6}$ until inner update $20000$. 
All of our numerical experiments are conducted in NVIDIA H100.
\subsection{Aggregation Equation}
Consider the equation: 
\begin{equation}
    \label{eq:agg-cont}
    \partial_t \rho + \nabla \cdot (\rho v) = 0, 
    \qquad v(t,x) = -\nabla \mathcal{K} * \rho(t,\cdot)(x)\,, 
\end{equation}
where $\mathcal{K}$ is an attraction–repulsion interaction kernel defined through: 
\begin{align} \label{kernel}
  \mathcal{K}(r) = \frac{ |r|^{q+1}}{q+1} - \frac{|r|^{p+1}}{p+1}, \quad p < q\,. 
\end{align}
Here $p$ controls short-range repulsion and $q$ governs long-range attraction. The associated energy functional is:
\begin{equation}
    \label{equ:interactionloss}
    \mathcal{E}_{\beta}(\rho) 
= \frac{1}{2} \iint_{\Omega \times \Omega} \mathcal{K}(x-y)\,\rho(x)\,\rho(y) \,\mathrm{d}x \,\mathrm{d}y\,.
\end{equation}
This model arises in various contexts, including swarming, self-assembly, and vortex dynamics~\cite{topaz2008model,mogilner2003mutual}.

For given pairs $(p,q)$, the equilibrium of \eqref{eq:agg-cont} takes the form of a ring centered at the same location as the initial density, i.e. $\int x \rho^{in}(x) \mathrm{d} x $ , with radius $r_0$ determined by
\begin{equation}
\int_{0}^{\pi/2} \mathcal{F}\!\bigl(2 r_{0}\sin\theta\bigr)\,\sin\theta \,\mathrm{d}\theta = 0\,,
\qquad
\text{where}   \quad  \mathcal{F}(r) = -\nabla \mathcal{K}(r).
\label{equ:radius}
\vspace{-0.5em}
\end{equation}

When discretizing \eqref{eq:agg-cont} with $m$ particles $\{x_i(t)\}_{i=1}^{m}$, the location of the particles evolves according to the  $m$ -body ODE system:
\begin{equation}
  \frac{\mathrm{d}x_i}{\mathrm{d}t}
  = \frac{1}{m} \sum_{\substack{k=1 \\ k\neq i}}^{m} 
    \mathcal{F}(|x_i-x_k|) \,\frac{x_i-x_k}{|x_i-x_k|},
  \qquad i=1,\dots,m .
  \label{equ:particle-ode}
\end{equation}
It is easy to see that  the center of  the particle is conserved
by summing~\eqref{equ:particle-ode} over $i$ i.e.,
\(
\frac{\mathrm{d}}{\mathrm{d}t} \left( \frac1m \sum_{i=1}^m x_i \right) = 0.
\) 
We define the \emph{potential ring solution} of particles for any given $(p,q)$ pair as the ring whose center coincides with the particle center of mass and whose radius satisfies \eqref{equ:radius}. 
Since the dynamics and equilibrium of \eqref{eq:agg-cont} are translation-invariant, recentering the initial density at the origin does not affect the resulting evolution or steady state.
For numerical stability, we therefore recenter all input densities to have zero mean before applying the trained neural JKO operator during evaluation.

In the experiments below, we train two different JKO operators: 
\begin{itemize}
\item \textbf{Single Initial Density and $(p,q)$ Pair:}
A JKO operator is trained for a fixed energy functional in \Cref{equ:interactionloss} with parameters $p = 0.5$ and $q = 3$. In this setting, only a single initial density is provided, chosen specifically to illustrate how the learn-to-evolve framework can generate an evolving training dataset.

\item \textbf{Generalizable JKO Operator:}
A JKO operator is trained for the energy functional in \Cref{equ:interactionloss} using various $(p,q)$ pairs and a collection of initial densities. Its performance is then evaluated on unseen $(p,q)$ pairs and initial densities to assess the operator’s generalization ability.

\end{itemize}
Both operators are trained with a fixed time step $\Delta t=1$.

\subsubsection{Single Initial Density and $(p, q)$ Pair}
For $p = 0.5$ and $q = 3$ in \Cref{equ:interactionloss}, the equilibrium forms a stable ring centered at the origin with radius $r_0 = 0.58$. A single initial density, $\rho^0$, is fixed as the uniform distribution on $[-1,1]^2$.
The JKO operator is trained to compute the first five JKO steps starting from $\rho^0$, and the corresponding ideal training dataset is denoted by
\[
\mathbb{D}^* = \{(\rho^{*,t})_{t=0}^4\,:\, \rho^{*,0}=\rho^{0}, 
\quad
\rho^{*,t+1} = \JKO(\rho^{*,t})
\}\,.
\]

In this experiment, a simplified setting is adopted by disabling the warm-up phase, setting $K_0 = 0$ in \Cref{alg:neurojko2stage}. The total number of inner updates is set to $E = 1000$, which leads to $ 717$  total outer iterations in this case. Since the initial density family contains only the fixed $\rho^0$, the resampling step in line~3 of \Cref{alg:neurojko2stage} is skipped. Instead, sample points of $\rho^0$ are generated once before training and remain fixed throughout. A simple Adam optimizer with a fixed learning rate of $10^{-4}$ is applied.
At outer iteration $k$, the training dataset is denoted by
\[
\mathbb{D}_k = \{
(\rho^{k,0})_{t=0}^4\,:\, \rho^{k,0}=\rho^{0},
\rho^{k,t+1} = \scrT(\rho^{k,t})_{\sharp} \rho^{k,t}
\}.
\]
Accordingly, $\rho^{k,t}$ is expected to converge to the JKO solution $\rho^{*,t}$ as $k \to \infty$.

The evolution of $\mathbb{D}_k$ is visualized in the first five columns of Figure~\ref{fig:train-data}, where $\rho^{k,0}$ remains unchanged. At $k = 0$, the randomly initialized network maps each density to a shifted configuration (first column). This shift becomes more pronounced in the second column. By $k = 2$, the initial density begins to contract toward the center, with $\rho^{2,1}$ and $\rho^{2,2}$ moving closer to the true solutions. By $k = 717$, the initial density progressively converges toward the equilibrium.

To illustrate this process more quantitatively, the last column of Figure~\ref{fig:train-data} displays the logarithm of the normalized loss:
\begin{equation}
\label{equ:losskt}
\ell(\rho^{k,t}) 
:= \ell(\scrT_k; \rho^{k,t}) - \min_{0 \leq k \leq717} \ell(\scrT_k; \rho^{k,t}) + 10^{-5},
\quad
0\leq t\leq 4
\end{equation}
The horizontal axis represents the cumulative number of inner iterations, with a total budget of $1000$. The green dots denote the logarithm of $\ell(\rho^{k,t})$ at each inner iteration, while the red dots record the values only at each outer iteration $k$.
Due to the nonconvex nature of neural network training, the loss as a function of the number of inner updates typically exhibits oscillatory behavior. In contrast, because new data are generated whenever the cumulative loss defined in \Cref{equ:indicator} decreases, as described in Algorithm~\ref{alg:neurojko2stage}, we have
$\sum_{i=1}^t\ell(\rho^{k+1,i}) <\sum_{i=1}^t\ell(\rho^{k,i})$ for all $t$. Numerically, we observe that each individual loss $\ell(\rho^{k,t})$ decreases nearly monotonically with respect to $k$, as indicated by the red markers.
Consequently, when $\ell(\rho^{k,0})$ converges to the unique minimum of the $0$th JKO subproblem,
$\mathcal{W}_2^2(\rho^{k,0},\rho^{*,1}) + \Delta t\,\mathcal{E}(\rho^{*,1})$,
the resulting density $\rho^{k,1}$ approaches $\rho^{*,1}$ during training.
This improvement is clearly visible in the second row of Figure~\ref{fig:train-data}. As training proceeds, these improvements propagate sequentially:
$\ell(\rho^{k,1})$ converges to its minimum $\ell(\rho^{*,1})$ as $k \to \infty$,
followed by $\rho^{k,2} \to \rho^{*,2}$, and so forth.
By $k = 717$, we observe a clear ring-like structure in $\rho^{717,4}$. The equilibrium state is not reached within the five-step training horizon.


\begin{figure}
  \centering
  \renewcommand{\arraystretch}{1.8}  
  
  \newcommand{\TimeRow}[1]{%
    \parbox[b][0.7cm][t]{0.06\textwidth}{\centering \small $t = #1$} &
      \includegraphics[width=0.13\textwidth]{fig_exp/aggFIX/train_Epoch_0_#1.png} &
      \includegraphics[width=0.13\textwidth]{fig_exp/aggFIX/train_Epoch_3_#1.png} &
      \includegraphics[width=0.13\textwidth]{fig_exp/aggFIX/train_Epoch_4_#1.png} &
      \includegraphics[width=0.13\textwidth]{fig_exp/aggFIX/train_Epoch_249_#1.png} &
      \includegraphics[width=0.13\textwidth]{fig_exp/aggFIX/train_Epoch_1005_#1.png} &
     \includegraphics[width=0.2\textwidth]{fig_exp/aggFIX/Loss_Step#1_new.png}\\[-2pt]
  }

  \setlength{\tabcolsep}{2pt}           
\begin{tabular}{c *{5}{c} @{\hspace{2pt}} c}

   &   { \small$k= 0$} &  { \small $k= 1$} &   { \small$k= 2$} &   { \small$k= 195$ }&  { \small $k= 717$} &  { \small $\log\ell(\rho^{k,t})$ vs. $k$(red) }
    \\
      \TimeRow{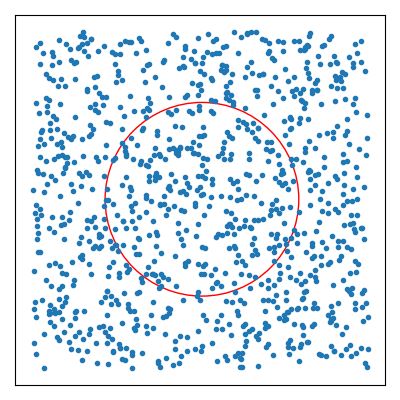}
      \TimeRow{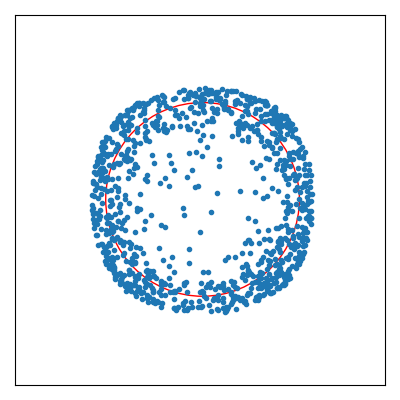}
      \TimeRow{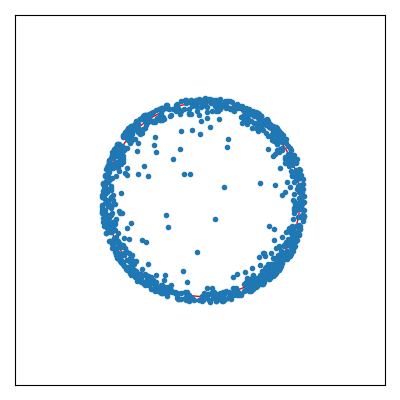}
      \TimeRow{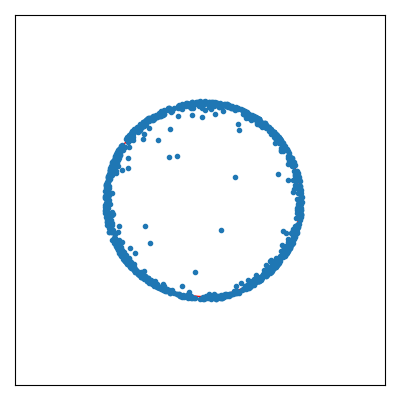}
      \TimeRow{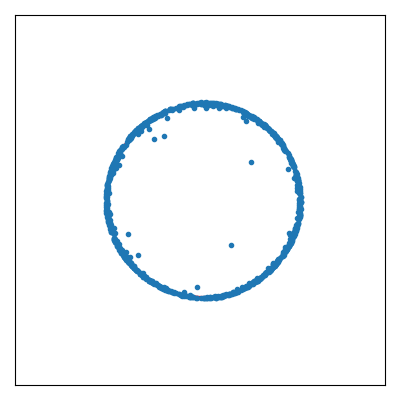}
  \end{tabular}
 
  \caption{
Visualization of $\mathbb{D}_k$: 
  (Columns 1-5) Evolution of the training dataset $\bbD_k$ at selected training iterations $k =0,1,2,195,717$, displayed in the domain $[-1.1,1.1]^2$. The red circles indicate the equilibrium state.
(Column~6) Normalized individual logarithmic loss $\log \ell(\rho^{k,t})$ versus outer iteration $k$:
green curves show losses over all inner updates, and red markers indicate outer updates 
where the accumulated loss (defined in~\Cref{equ:indicator}) decreases.
           }
  \label{fig:train-data}
\end{figure}

\begin{figure}[htbp]
    \centering
    \setlength{\tabcolsep}{0.5pt} 
    \begin{tabular}{cccccccc}
$t=0$ &  $t=1$ &  $t=2$ &  $t=3$ &  $t=4$ &  $t=5$ &  $t=10$ &  $t=50$\\
        \includegraphics[width=0.12\textwidth]{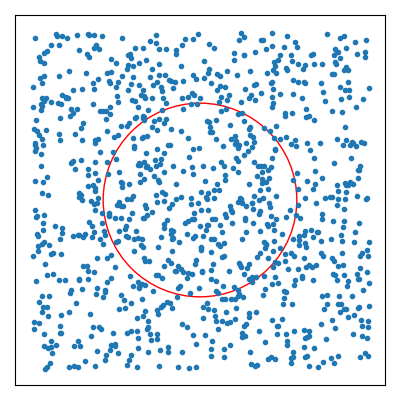} &
        \includegraphics[width=0.12\textwidth]{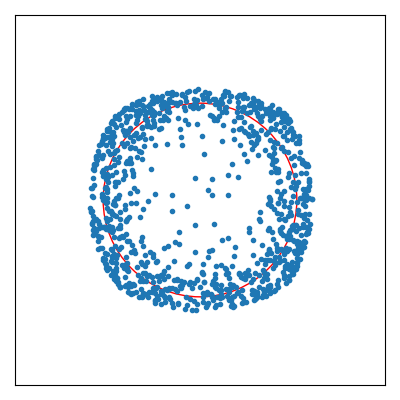} &
        \includegraphics[width=0.12\textwidth]{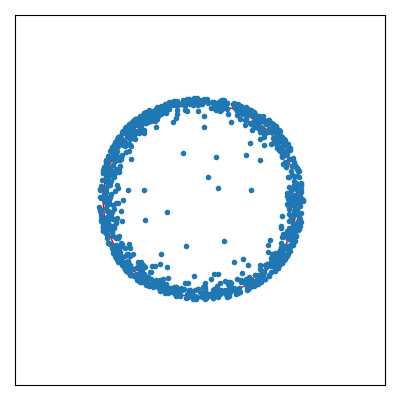} &
        \includegraphics[width=0.12\textwidth]{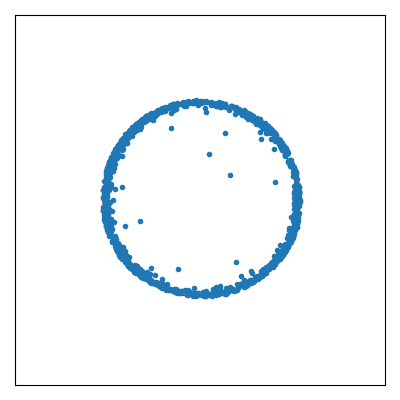} &
        \includegraphics[width=0.12\textwidth]{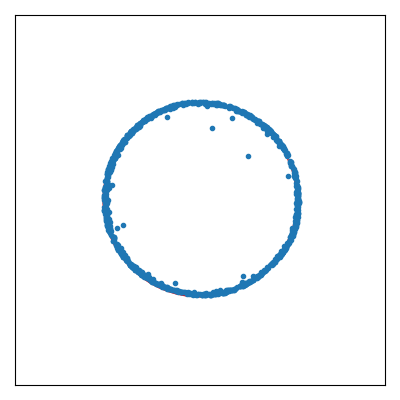} &
        \includegraphics[width=0.12\textwidth]{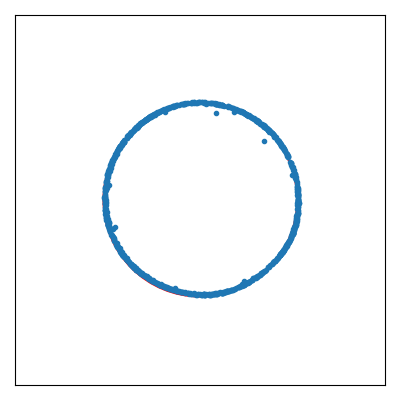} &
        \includegraphics[width=0.12\textwidth]{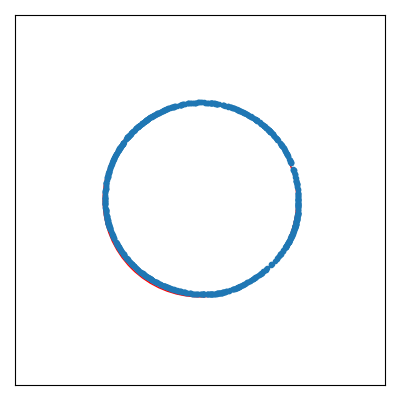} &
        \includegraphics[width=0.12\textwidth]{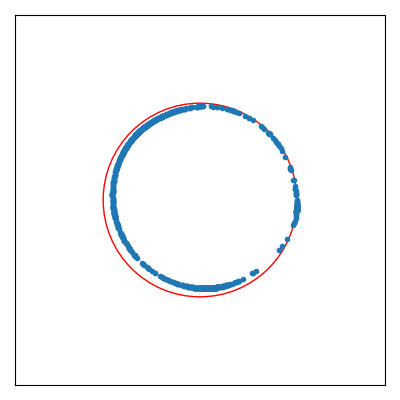} \\

{\footnotesize $-0.3115$} & {\footnotesize $-0.3188$} & {\footnotesize $-0.3200$} & {\footnotesize $-0.3203$} & {\footnotesize $-0.3203$} & {\footnotesize $-0.3203$} & {\footnotesize $-0.3200$} & {\footnotesize $-0.3068$} \\

        \includegraphics[width=0.12\textwidth]{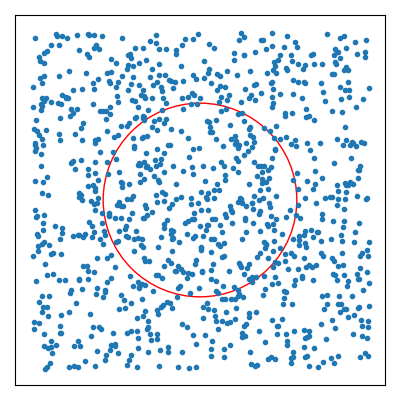} &
        \includegraphics[width=0.12\textwidth]{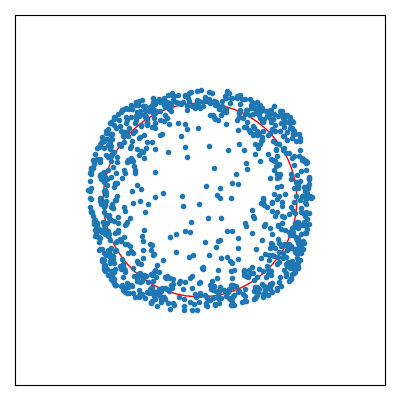} &
        \includegraphics[width=0.12\textwidth]{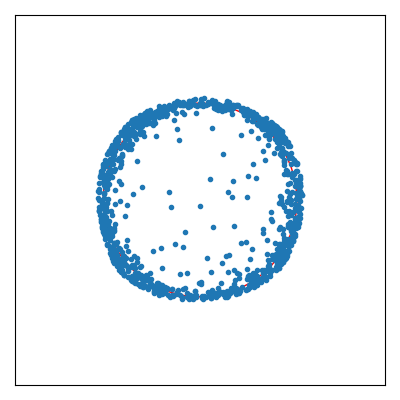} &
        \includegraphics[width=0.12\textwidth]{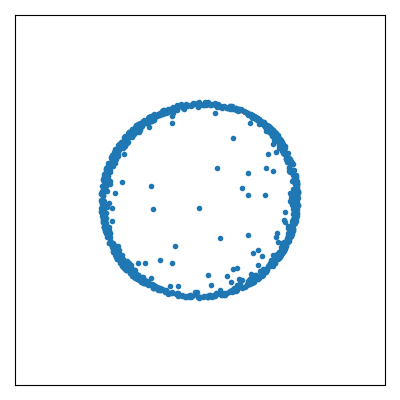} &
        \includegraphics[width=0.12\textwidth]{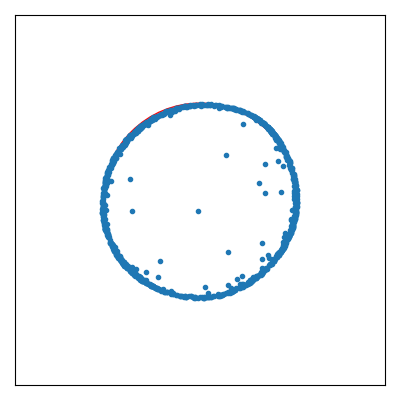} &
        \includegraphics[width=0.12\textwidth]{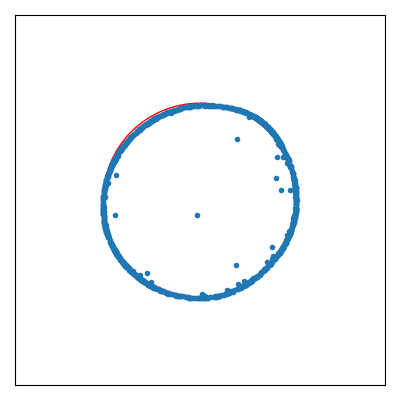} &
        \includegraphics[width=0.12\textwidth]{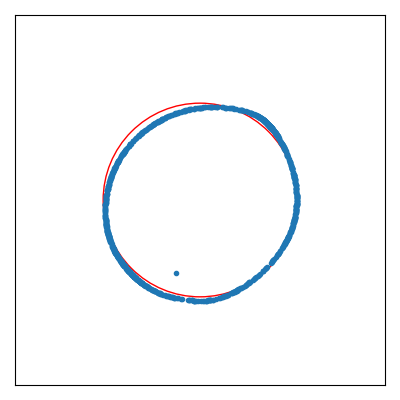} &
        \includegraphics[width=0.12\textwidth]{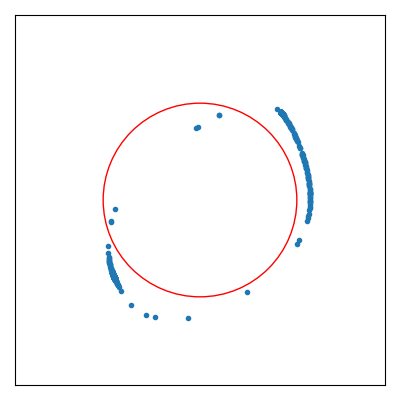} \\
{\footnotesize $-0.3115$} & {\footnotesize $-0.3185$} & {\footnotesize $-0.3196$} & {\footnotesize $-0.3197$} & {\footnotesize $-0.3192$} & {\footnotesize $-0.3184$} & {\footnotesize $-0.3086$} & {\footnotesize $-0.0178$} \\
\multicolumn{8}{c}{\textbf{CASE (1)}} \\
\hline
\\

        \includegraphics[width=0.12\textwidth]{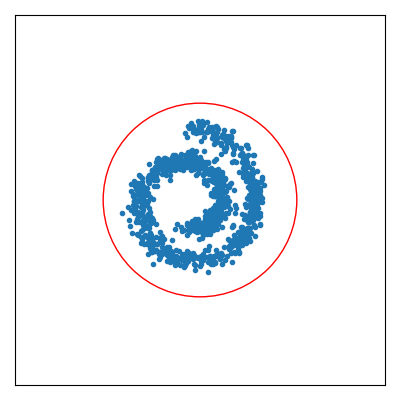} &
        \includegraphics[width=0.12\textwidth]{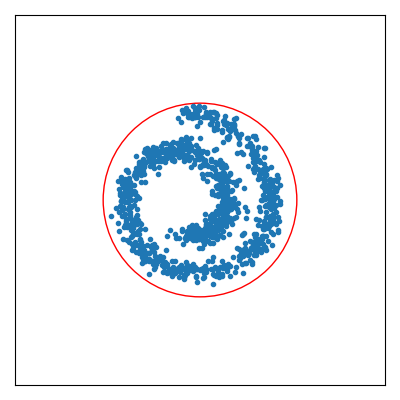} &
        \includegraphics[width=0.12\textwidth]{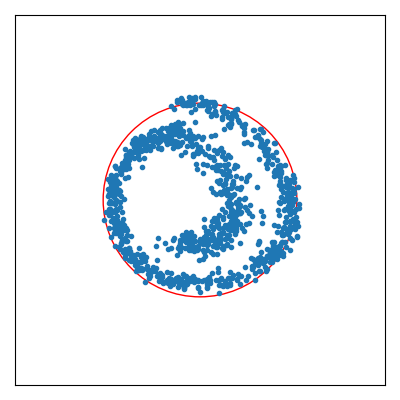} &
        \includegraphics[width=0.12\textwidth]{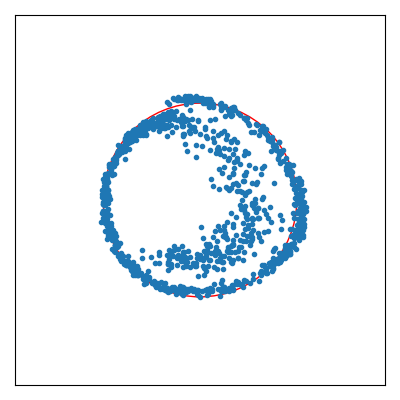} &
        \includegraphics[width=0.12\textwidth]{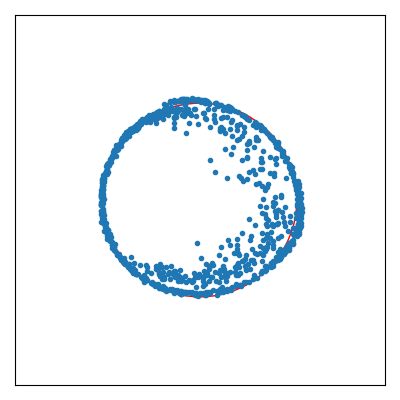} &
        \includegraphics[width=0.12\textwidth]{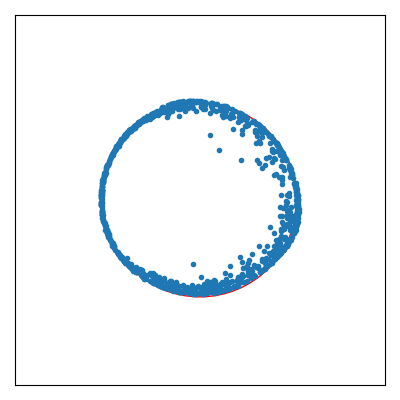} &
        \includegraphics[width=0.12\textwidth]{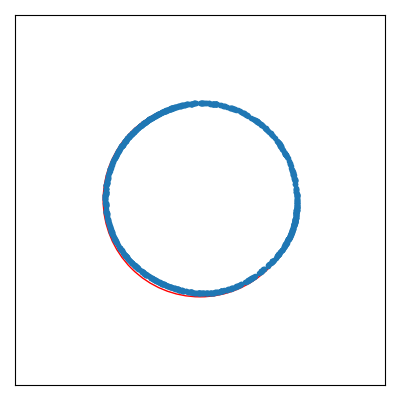} &
        \includegraphics[width=0.12\textwidth]{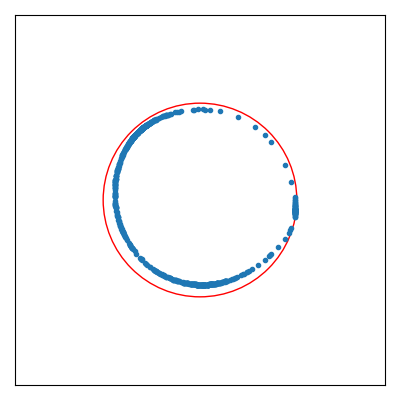} \\
{\footnotesize $-0.2224$} & {\footnotesize $-0.2675$} & {\footnotesize $-0.2999$} & {\footnotesize $-0.3138$} & {\footnotesize $-0.3184$} & {\footnotesize $-0.3194$} & {\footnotesize $-0.3189$} & {\footnotesize $-0.2972$} \\

        \includegraphics[width=0.12\textwidth]{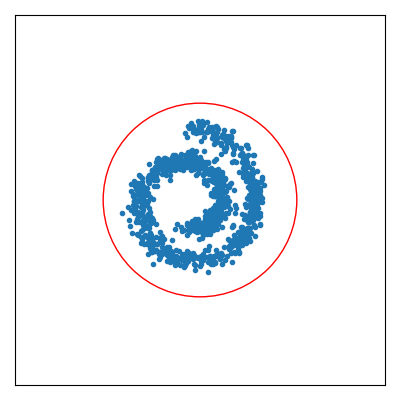} &
        \includegraphics[width=0.12\textwidth]{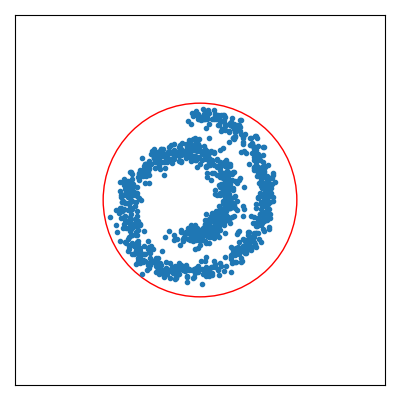} &
        \includegraphics[width=0.12\textwidth]{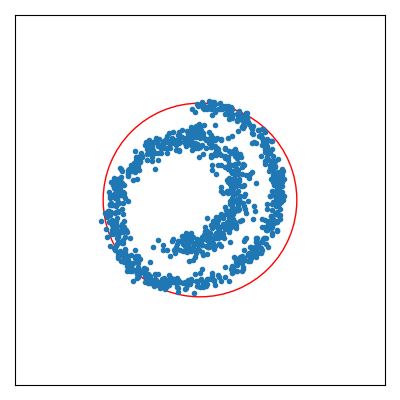} &
        \includegraphics[width=0.12\textwidth]{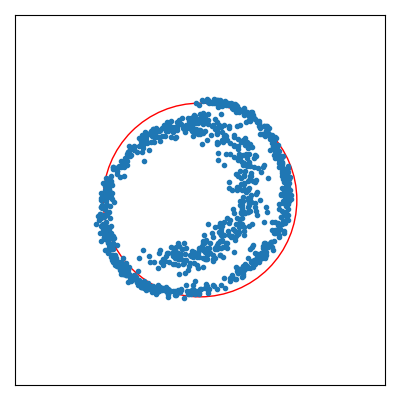} &
        \includegraphics[width=0.12\textwidth]{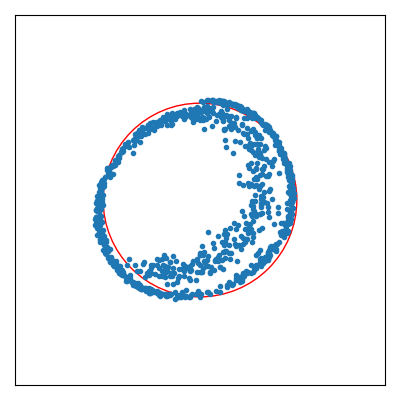} &
        \includegraphics[width=0.12\textwidth]{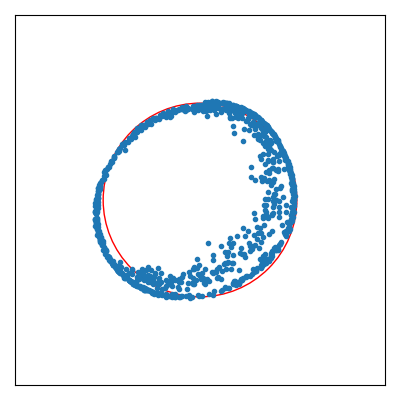} &
        \includegraphics[width=0.12\textwidth]{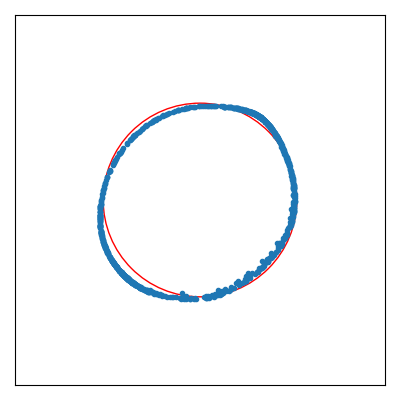} &
        \includegraphics[width=0.12\textwidth]{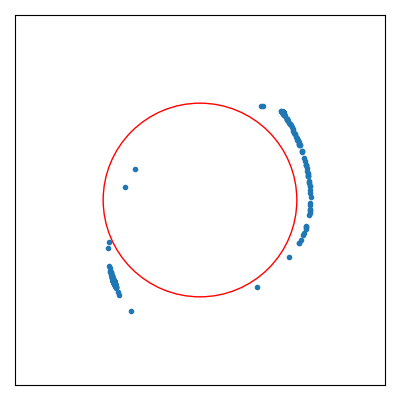} \\
{\footnotesize $-0.2224$} & {\footnotesize $-0.2657$} & {\footnotesize $-0.2950$} & {\footnotesize $-0.3084$} & {\footnotesize $-0.3125$} & {\footnotesize $-0.3122$} & {\footnotesize $-0.2853$} & {\footnotesize $0.0308$} \\

\multicolumn{8}{c}{\textbf{CASE (2)}} \\
\hline
\\

        \includegraphics[width=0.12\textwidth]{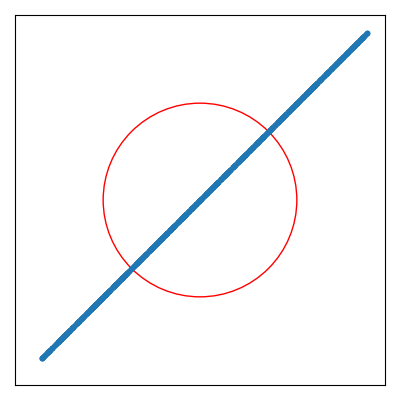} &
        \includegraphics[width=0.12\textwidth]{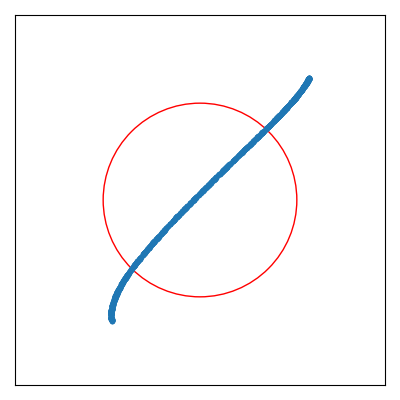} &
        \includegraphics[width=0.12\textwidth]{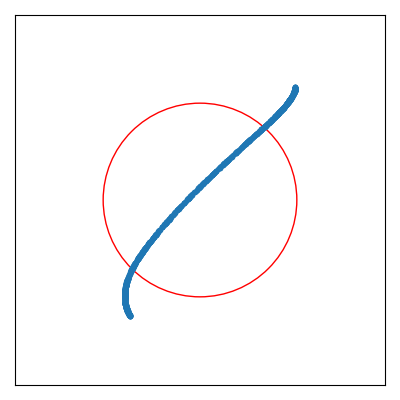} &
        \includegraphics[width=0.12\textwidth]{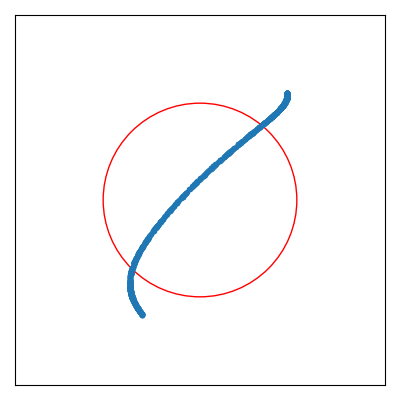} &
        \includegraphics[width=0.12\textwidth]{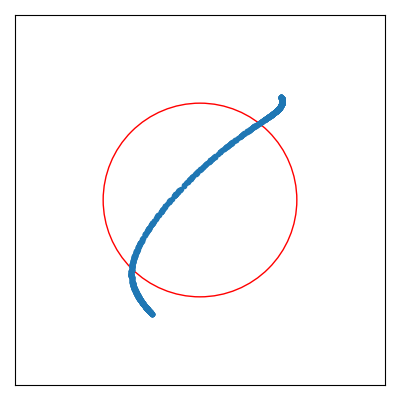} &
        \includegraphics[width=0.12\textwidth]{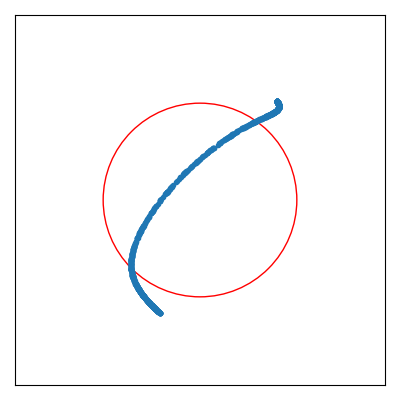} &
        \includegraphics[width=0.12\textwidth]{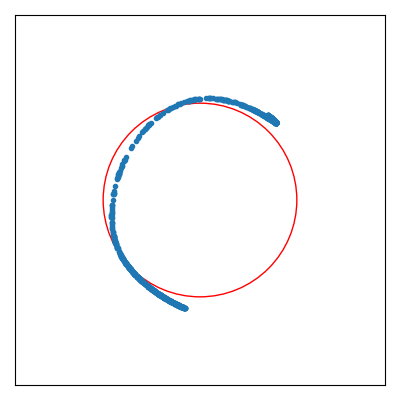} &
        \includegraphics[width=0.12\textwidth]{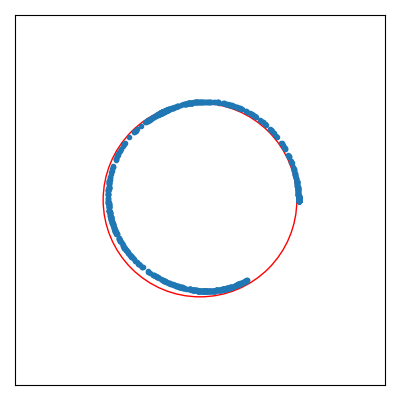} \\
        
{\footnotesize $-0.0344$} & {\footnotesize $-0.0749$} & {\footnotesize $-0.0906$} & {\footnotesize $-0.1122$} & {\footnotesize $-0.1330$} & {\footnotesize $-0.1517$} & {\footnotesize $-0.2159$} & {\footnotesize $-0.3160$} \\

        \includegraphics[width=0.12\textwidth]{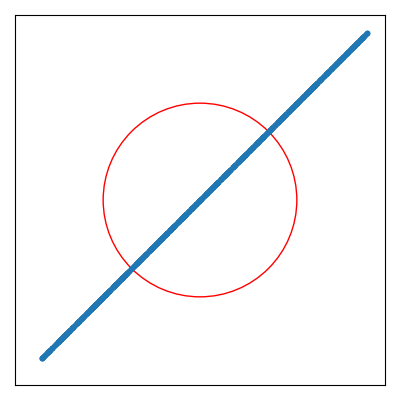} &
        \includegraphics[width=0.12\textwidth]{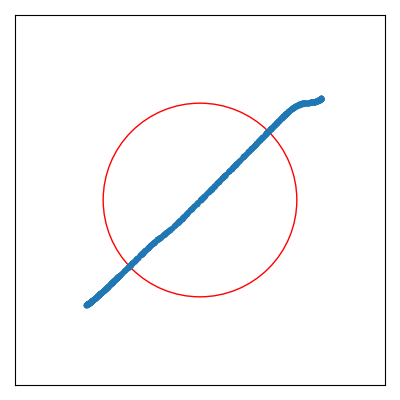} &
        \includegraphics[width=0.12\textwidth]{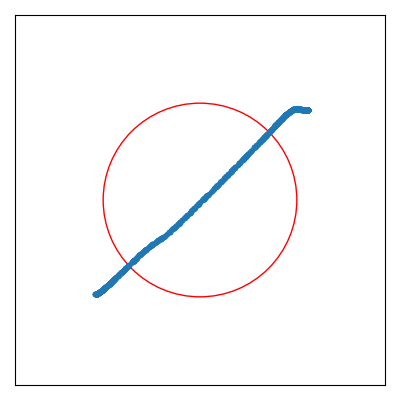} &
        \includegraphics[width=0.12\textwidth]{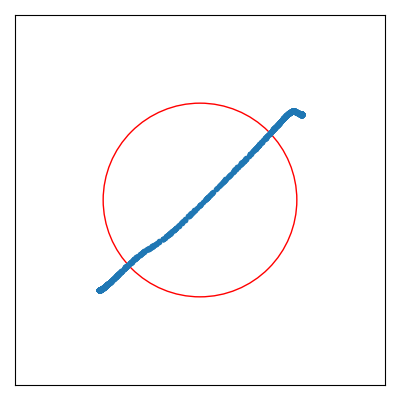} &
        \includegraphics[width=0.12\textwidth]{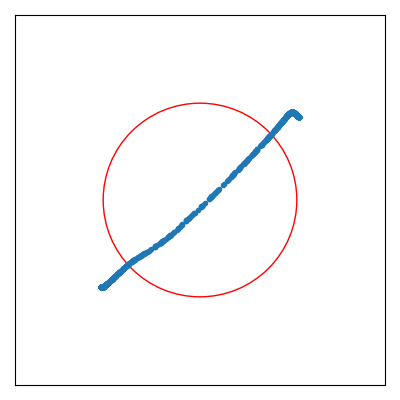} &
        \includegraphics[width=0.12\textwidth]{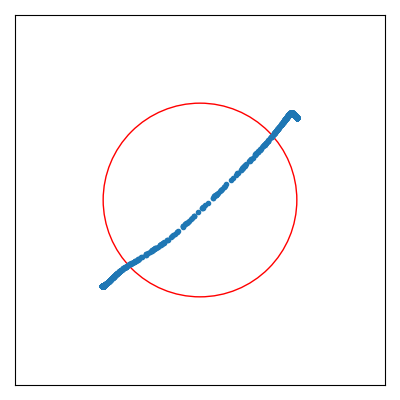} &
        \includegraphics[width=0.12\textwidth]{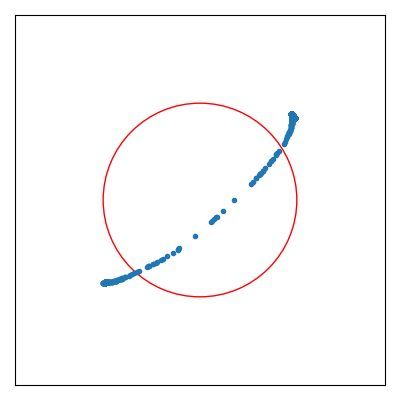} &
        \includegraphics[width=0.12\textwidth]{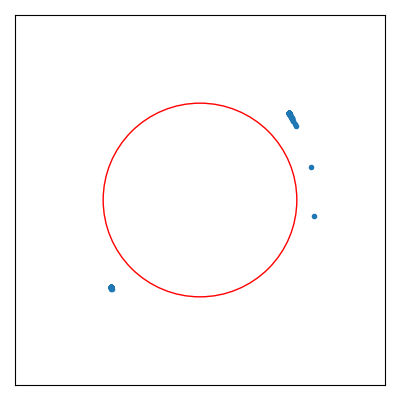} \\      
 {\footnotesize $-0.0344$} & {\footnotesize $0.0145$} & {\footnotesize $-0.0016$} & {\footnotesize $0.0083$} & {\footnotesize $0.0285$} & {\footnotesize $0.0516$} & {\footnotesize $0.1332$} & {\footnotesize $0.1349$} \\

\multicolumn{8}{c}{\textbf{CASE (3)}} \\
\hline
      
    \end{tabular}
    \caption{
    Comparison of predicted solutions of the aggregation equation ($p=0.5$, $q=3$) for three test cases.
Each column shows results at successive time steps $t$, with the corresponding energy $\mathcal{E}(\rho^t)$ displayed below (smaller values indicate closer proximity to equilibrium).
For each case, the first row shows our model’s predictions and the second row those of the baseline.
Red circles mark the equilibrium state.
}
\label{fig:fixTest}
\end{figure}

\begin{figure}[htbp]
    \centering
    \setlength{\tabcolsep}{0.5pt} 
    \begin{tabular}{cccccccc}
   
        \includegraphics[width=0.12\textwidth]{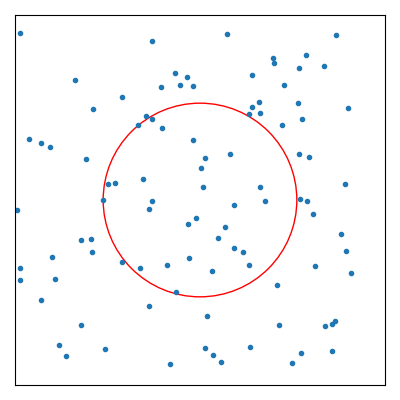} &
        \includegraphics[width=0.12\textwidth]{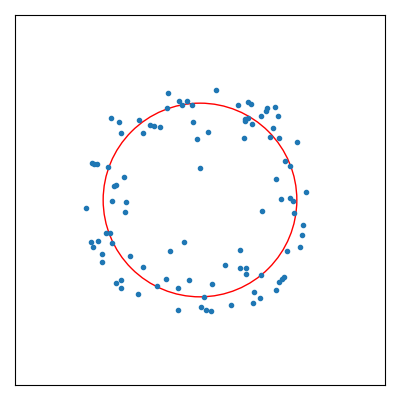} &
        \includegraphics[width=0.12\textwidth]{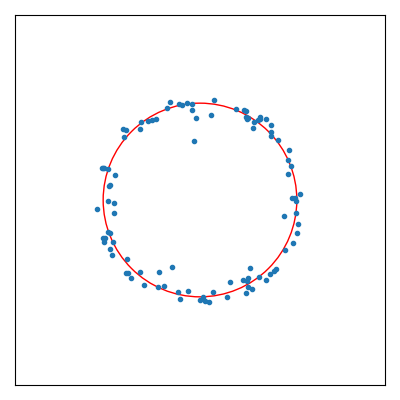} &
        \includegraphics[width=0.12\textwidth]{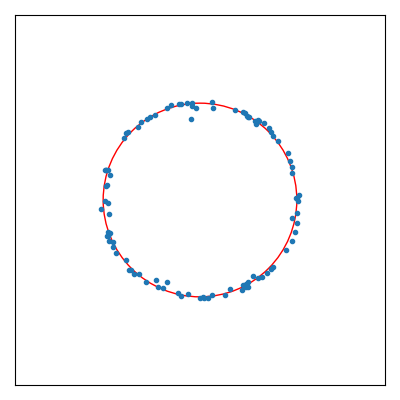} &
        \includegraphics[width=0.12\textwidth]{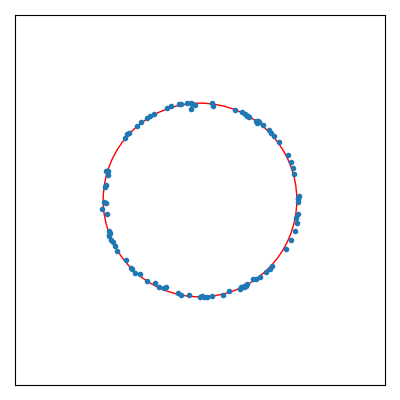} &
        \includegraphics[width=0.12\textwidth]{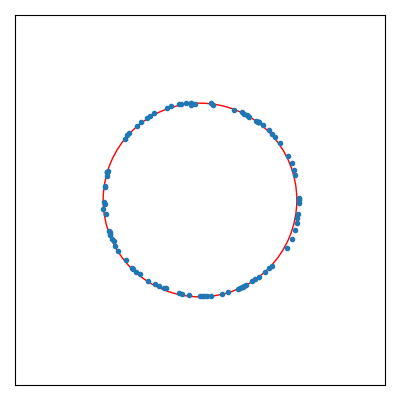} &
        \includegraphics[width=0.12\textwidth]{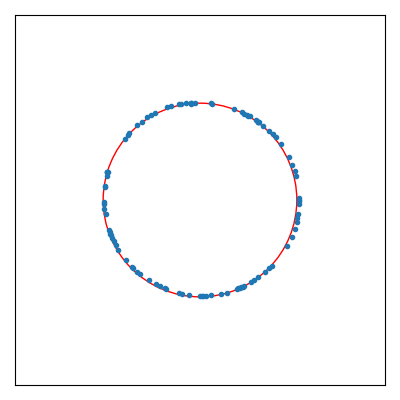} &
        \includegraphics[width=0.12\textwidth]{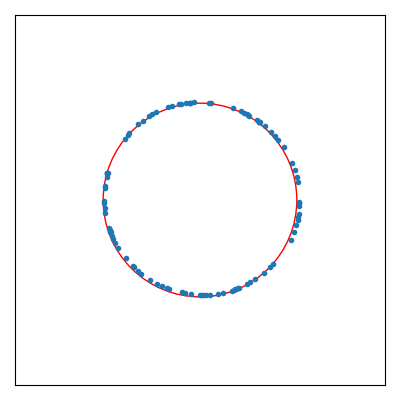} \\

\includegraphics[width=0.12\textwidth]{fig_exp/aggFIX_compare/1Epoch_0_-0.31145_-0.55301.png} &
        \includegraphics[width=0.12\textwidth]{fig_exp/aggFIX_compare/1Epoch_1_-0.31875_-0.63328.png} &
        \includegraphics[width=0.12\textwidth]{fig_exp/aggFIX_compare/1Epoch_2_-0.31997_-0.63826.png} &
        \includegraphics[width=0.12\textwidth]{fig_exp/aggFIX_compare/1Epoch_3_-0.32026_-0.63929.png} &
        \includegraphics[width=0.12\textwidth]{fig_exp/aggFIX_compare/1Epoch_4_-0.32031_-0.63952.png} &
        \includegraphics[width=0.12\textwidth]{fig_exp/aggFIX_compare/1Epoch_5_-0.32029_-0.63950.png} &
\includegraphics[width=0.12\textwidth]{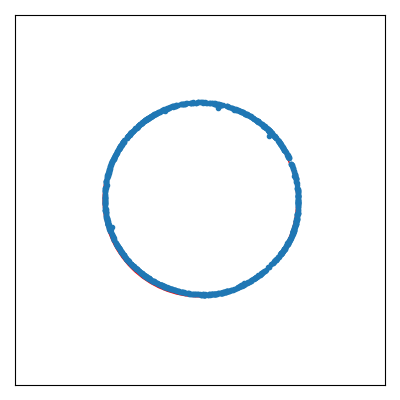} &
        \includegraphics[width=0.12\textwidth]{fig_exp/aggFIX_compare/1Epoch_10_-0.32000_-0.63892.png}  \\

         \includegraphics[width=0.12\textwidth]{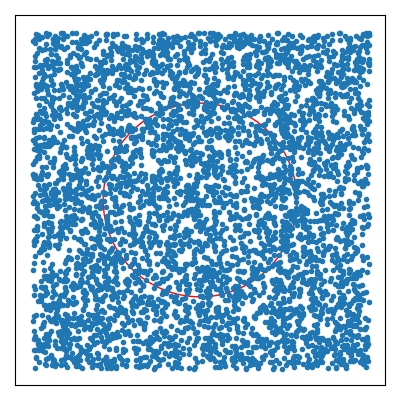} &
        \includegraphics[width=0.12\textwidth]{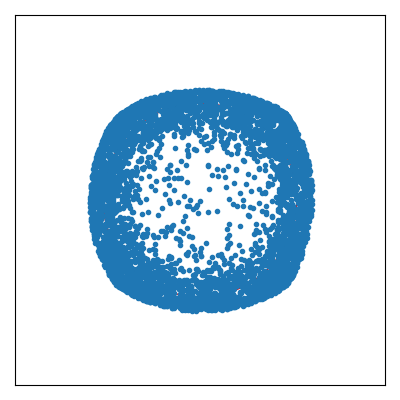} &
        \includegraphics[width=0.12\textwidth]{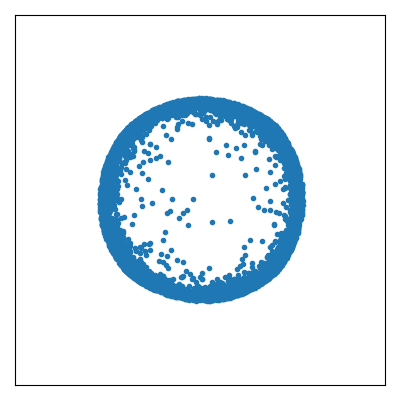} &
        \includegraphics[width=0.12\textwidth]{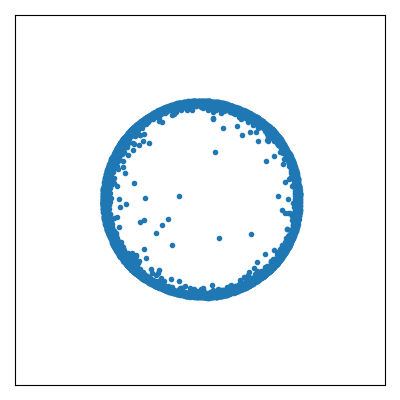} &
        \includegraphics[width=0.12\textwidth]{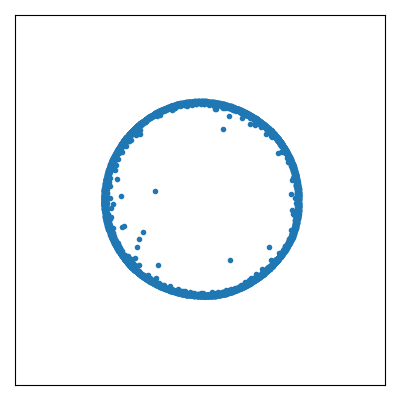} &
        \includegraphics[width=0.12\textwidth]{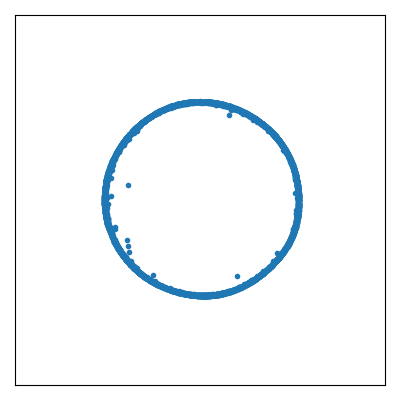} &
        \includegraphics[width=0.12\textwidth]{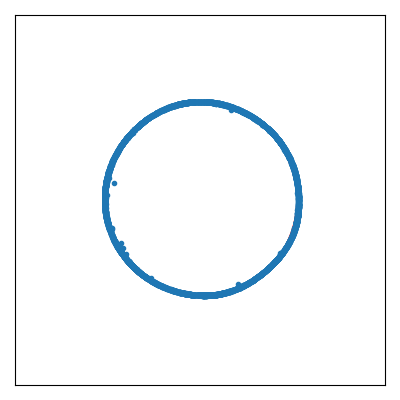} &
        \includegraphics[width=0.12\textwidth]{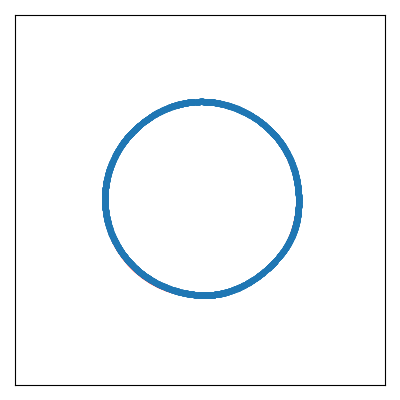} \\
      $t=0$ &  $t=1$ &  $t=2$ &  $t=3$ &  $t=4$ &  $t=5$ &  $t=6$ &  $t=10$
    \end{tabular}
    \caption{
Predicted solutions of the aggregation equation ($p=0.5$, $q=4$) at each time $t$ with input uniform distribution represented by 100 (first row), 1024 (second row) and 5000 (third row) sample points.
Red circles denote the theoretical equilibrium.
}
\label{fig:differentsample}
\end{figure}

More interestingly, the learned JKO operator demonstrates strong generalization ability: although it is trained on a single fixed initial density and only over the first five JKO steps, it can be applied to a wide range of unseen initial conditions and over much longer time horizons. This behavior is illustrated in the three test cases shown in \Cref{fig:fixTest}:
(1) the same uniform distribution on $[-1,1]^2$ but with independently resampled points;
(2) an initial density that differs moderately from the training distribution; and
(3) an initial density that differs substantially from it.
For the first two cases, the predicted WGFs converge to a stable ring by step~10, and the energy decreases monotonically with only minor fluctuations, likely due to sampling noise. The densities then remain nearly stationary over the subsequent 40 steps, indicating a stable equilibrium.
In the third case, the WGF converges more slowly but continues to approach equilibrium as the loss decreases.

As a comparison, we train a baseline model using five fixed densities from $\mathbb{D}_{717}$, the final training dataset generated by our algorithm, which is sufficiently close to the ideal dataset $\mathbb{D}^*$. The baseline model uses the same neural network architecture and optimizer as our method; the only difference lies in the training data. Our model is trained on a sequence of progressively more accurate datasets, whereas the baseline model is trained directly on the final, most accurate dataset.
As shown in \Cref{fig:fixTest}, the baseline model exhibits limited generalization capability, primarily due to the lack of data diversity. In Cases~1 and~2, it yields higher energy loss and fails to reach the equilibrium state. In Case~3 where the initial density differs substantially from the training data, it is unable to produce a reasonable WGF trajectory.

Finally, we point out that the trained JKO operator can be applied to input density represented by a different number of sample points from those used in training, due to the flexibility of the Transformer-based architecture.
\Cref{fig:differentsample} illustrates the predicted results when $100$, $1024$ or $5000$ points are sampled from the same uniform distribution.
In all cases, the solutions successfully converge to the equilibrium,
demonstrating that the attention mechanism captures the pairwise relationships among the sample points rather than relying on their absolute positions or quantities.

\subsubsection{Generalizable JKO operator}
\label{sec:aggregeneralize}

We train a single JKO operator for the parametrized family of interaction energies in \Cref{equ:interactionloss}, which is therefore expected to generalize across arbitrary $(p,q)$ pairs and diverse initial densities. To achieve this, the original network input, the density represented by its sample points, is augmented with its associated parameter pair by concatenating the $(p,q)$ values to every sample point. During training, $(p,q)$ pairs are drawn uniformly from the domain
\[
\Lambda_{p,q} = \{(p,q):\, 0 < p < 1,\; p < q < 10\},
\]
and initial densities are generated as uniform distributions over randomly placed rectangles or triangles within $[-1,1]^2$. The model is trained to perform $T = 10$ JKO steps with a fixed step size $\Delta t = 1$. In \Cref{alg:neurojko2stage}, the batch size is set to 36.

For attraction--repulsion interactions of the form \eqref{kernel}, smaller values of $p$ enhance short-range repulsion, while larger values of $q$ strengthen long-range attraction. Their interplay admits ring-shaped equilibria over a broad range of parameters, whose existence and stability depend sensitively on the pair $(p,q)$. Following the classification in \cite{kolokolnikov2011stability}, the parameter domain $\Lambda_{p,q}$ can be partitioned into four regimes associated with qualitatively distinct steady-state patterns:
\emph{(I) Stable rings}, 
When $pq>1$ and $q<q_3(p)$ where $q_3(p)$ denotes the critical curve in parameter space at which the ring equilibrium first becomes unstable as characterized in \cite{kolokolnikov2011stability}, the ring equilibrium remains rotationally symmetric and does not fragment under perturbations.
\emph{(II) Polygonal $N$-dot patterns}, 
When $q$ exceeds the instability threshold $q_3(p)$, symmetry-breaking instabilities cause the ring to fragment into a finite number of localized clusters arranged in a polygonal configuration.
\emph{(III) Diffuse clouds or annuli}, 
When $pq\le 1$ or $p\ge q$, the ring equilibrium loses stability or ceases to exist, leading to fully two-dimensional diffuse or annular density distributions.
\emph{(IV) Collapse/Monopoles}, 
When $0<q<p$, short-range attraction dominates and the dynamics collapse toward a single concentrated cluster.
The collapse regime is excluded from training and is instead used to assess the ability
of the trained operator
to identify equilibrium types not covered during training. 
This regime classification is illustrated in the rightmost column of \Cref{fig:ringtest}.

The predicted solutions for selected $(p,q)$ pairs and initial densities are shown in \Cref{fig:predflow_agg}. 
Although trained using simple  initial densities uniform on random rectangles or triangles, the learned operator generalizes to unseen initial densities, including Gaussian and star-shaped profiles, and reproduces the expected dynamical behaviors.
We recall that, for a fixed parameter pair $(p,q)$, the equilibrium pattern is independent of the initial condition. 
Consistent with this property, for $(p,q)=(0.5,3)$ in the Type~I regime, the predicted solutions converge to the same stable ring regardless of whether the initial density is Gaussian or a mixture of rectangular profiles, as illustrated in the first two rows of \Cref{fig:predflow_agg}.
In the Type~II regime, the parameter pair $(p,q)=(0.9,9)$ undergoes a symmetry-breaking instability, leading to a transition from a ring to a three-point configuration. 
For $(p,q)=(0.5,6)$, the predicted Wasserstein gradient flow exhibits a pentagonal breakup. 
In contrast, the continuum theory predicts that the ring should break into a three-point configuration for this parameter choice \cite{kolokolnikov2011stability}, indicating a minor prediction error in the trained JKO operator.
In the Type~III regime, for $(p,q)=(0.2,2)$ and $(0.5,1.5)$, the predicted flows evolve into irregular annular structures by approximately $t=10$. 
This behavior is consistent with the presence of high-mode instabilities that preclude the formation of stable ring equilibria.

\begin{figure}[htp]
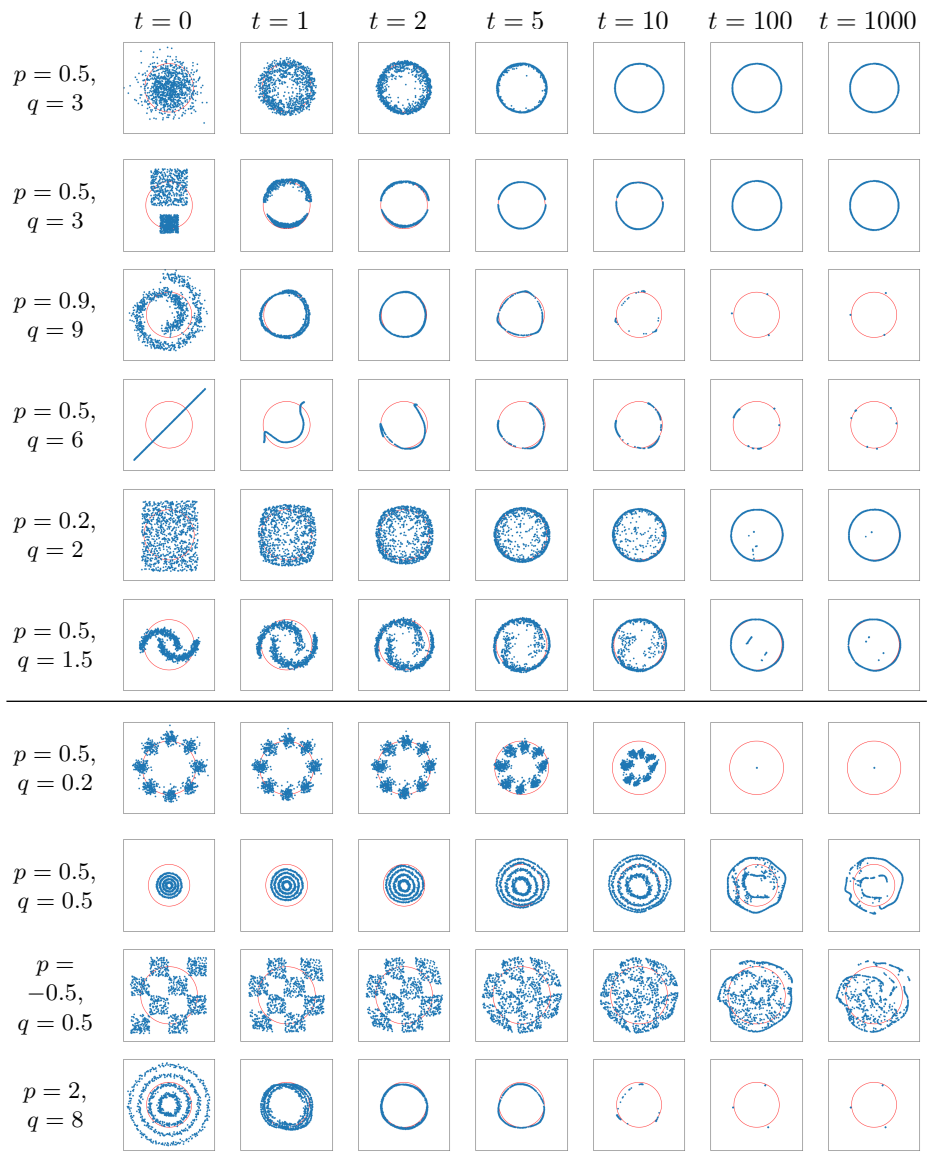

  \centering
  \renewcommand{\arraystretch}{0}    
  \setlength{\tabcolsep}{0pt}           
  \newcommand{\TimeRow}[3]{%
    \parbox[b][1.5cm][c]{0.1\textwidth}{\centering \small $p = #1$,\\ $q = #2$} &
      \includegraphics[width=0.12\textwidth]{fig_exp/aggregation/#3/Frame_000.png} &
      \includegraphics[width=0.12\textwidth]{fig_exp/aggregation/#3/Frame_001.png} &
      \includegraphics[width=0.12\textwidth]{fig_exp/aggregation/#3/Frame_002.png} &
      \includegraphics[width=0.12\textwidth]{fig_exp/aggregation/#3/Frame_005.png} &
      \includegraphics[width=0.12\textwidth]{fig_exp/aggregation/#3/Frame_010.png} &
      \includegraphics[width=0.12\textwidth]{fig_exp/aggregation/#3/Frame_100.png} &
     \includegraphics[width=0.12\textwidth]{fig_exp/aggregation/#3/Frame_999.png}\\
  }

\begin{tabular}{cccccccc}  

  &    $t=0$& $t=1$&$t=2$&$t=5$&$t=10$&$t=100$&$t=1000$\\    
      \TimeRow{0.5}{3}{p0.5_q3} \vspace{-0.1cm}
      \TimeRow{0.5}{3}{p0.5_q3_2}  \vspace{-0.1cm} 
      \TimeRow{0.9}{9}{p0.9_q9} \vspace{-0.1cm}
      \TimeRow{0.5}{6}{p0.5_q6}  \vspace{-0.1cm}
      \TimeRow{0.2}{2}{p0.2_q2} \vspace{-0.1cm}      
      \TimeRow{0.5}{1.5}{p0.5_q1.5}  
\midrule
      \TimeRow{0.5}{0.2}{p0.5_q0.2}   \vspace{-0.1cm}
      \TimeRow{0.5}{0.5}{p0.5_q0.5} \vspace{-0.1cm}
      \TimeRow{-0.5}{0.5}{p-0.5_q0.5}  \vspace{-0.1cm}
      \TimeRow{2}{8}{p2_q8}  
  \end{tabular}
 
  \caption{
  Predicted solutions of the aggregation equation using the generalizable JKO operator across varying $(p,q)$ pairs and initial densities.
The horizontal divider separates cases within the training range ($0 \le p < 1$, $p < q \le 10$) from cases outside this range.
Red circles indicate the potential ring solution. }  
  \label{fig:predflow_agg}
\end{figure}

The learned operator also generalizes well to $(p,q)$ values outside the training range, as shown in the second part of ~\Cref{fig:predflow_agg}.
For $(p,q) = (0.5,0.2)$, the predicted evolution collapses to a single point, consistent with predominantly attractive interactions (Type~I equilibrium). 
Although this equilibrium type was not included in training, the operator still converges to the correct steady state. 
For $(p,q) = (0.5,0.5)$, the interaction force vanishes, and the exact dynamics keep particles stationary at each JKO step; 
the learned operator, however, exhibits a slow drift toward irregular clusters, reflecting approximation error at this degenerate parameter setting. 
The case $(p,q) = (-0.5,0.5)$ produces a compact, disk-like equilibrium, indicating strong short-range repulsion with moderate attraction at intermediate distances. 
Finally, $(p,q) = (2,8)$ shows a rapid breakup of a transient ring into three clusters, consistent with the expected mode-3 instability. 
Overall, the predicted WGFs follow the expected dynamics with only minor approximation errors.

\begin{figure}[htp]
    \centering
    \setlength{\tabcolsep}{0pt} 
    \begin{tabular}{ccccc}   
        \includegraphics[width=0.19\textwidth]{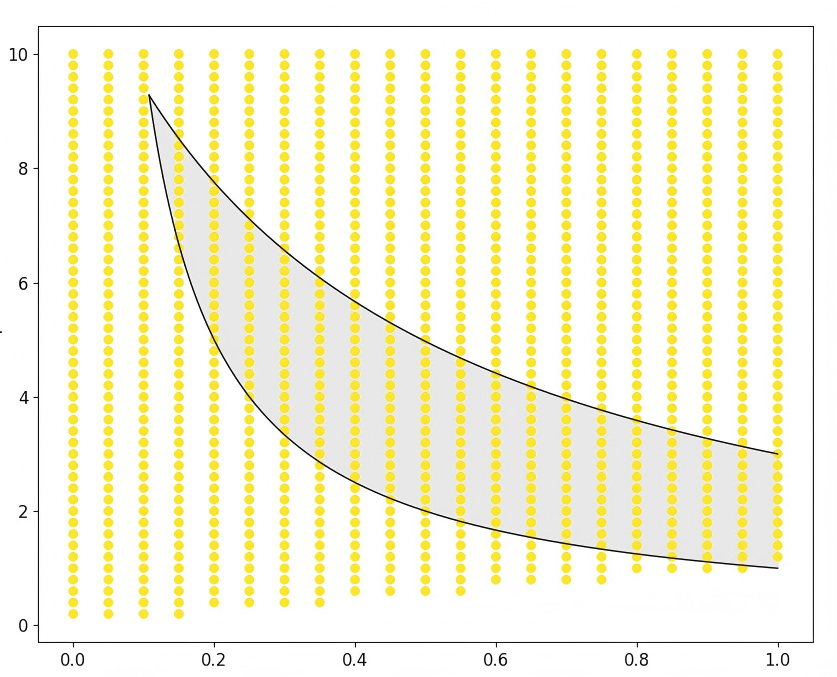} &
        \includegraphics[width=0.185\textwidth]{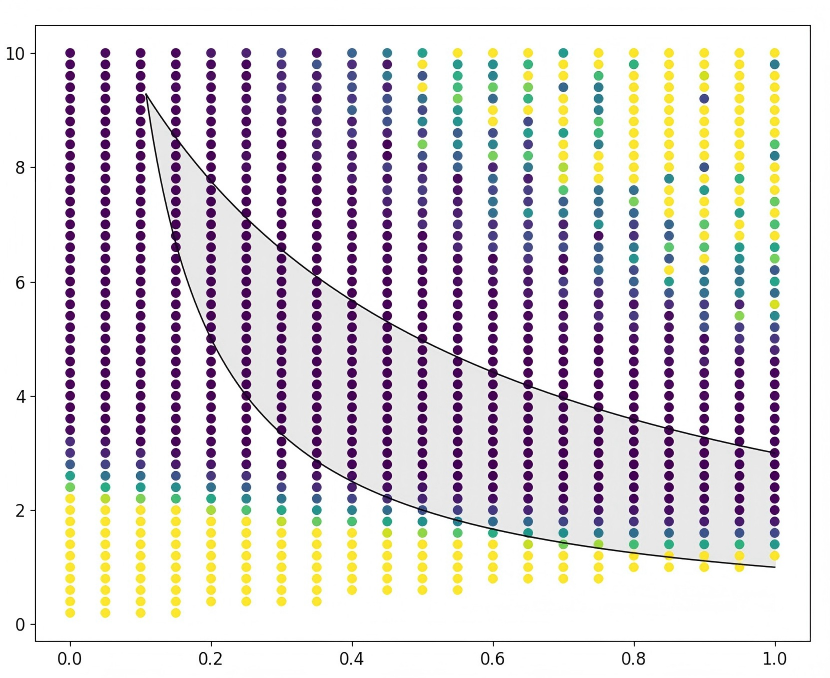} &
        \includegraphics[width=0.19\textwidth]{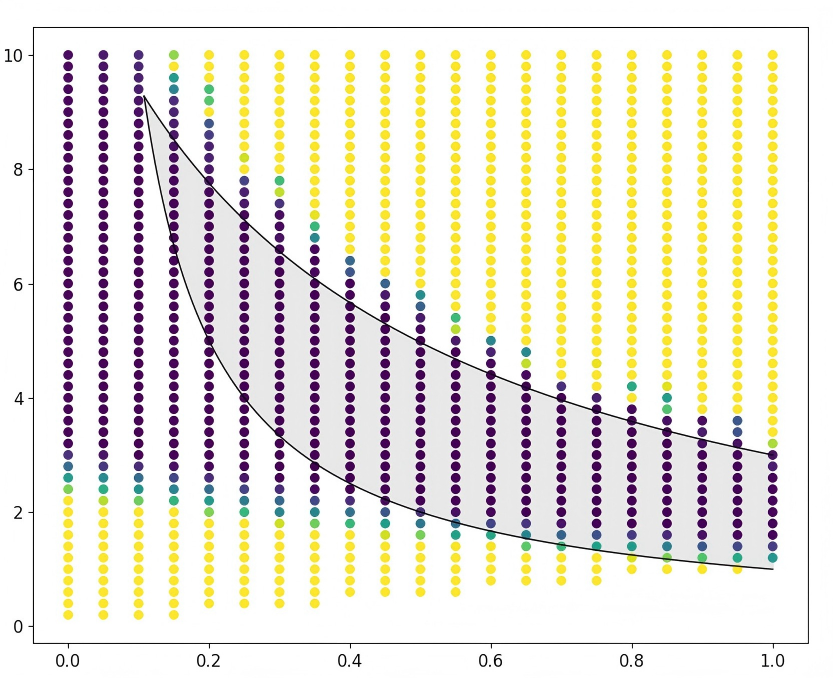} &
        \includegraphics[width=0.19\textwidth]{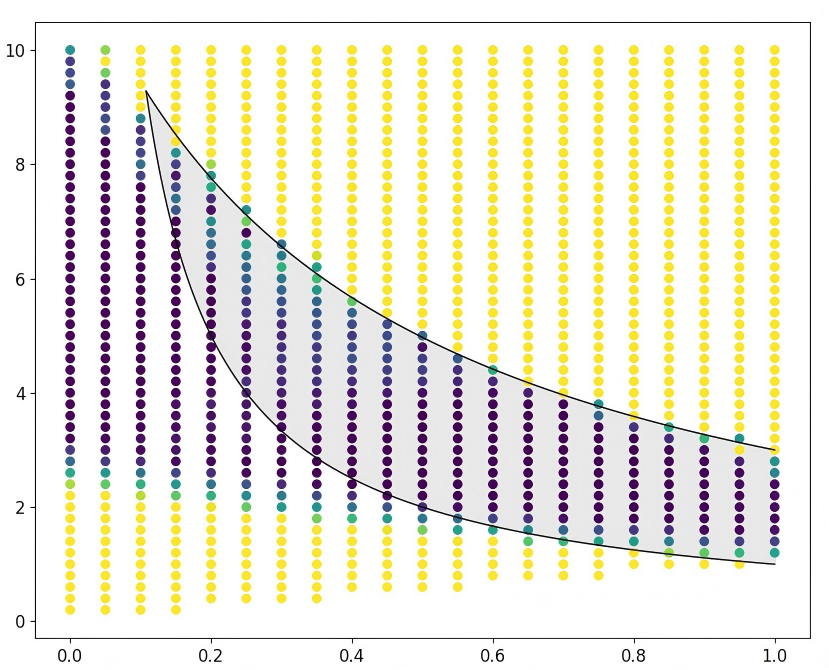} &
        \includegraphics[width=0.19\textwidth]{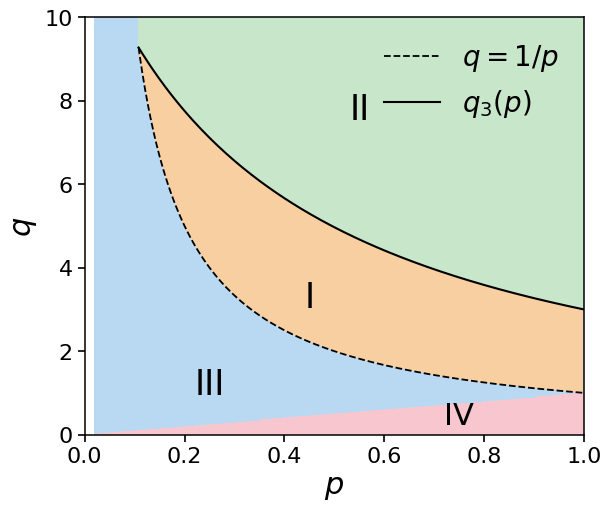} \\
      $t=0$ &  $t=10$ &  $t=50$ &  $t=1000$ &  equilibrium 
    \end{tabular}
    \caption{
    Ring test over 1005 $(p,q)$ pairs within the training domain. 
 For each plot,
 every point represents a predicted density for the corresponding $(p,q)$ value, 
The color indicates the Chamfer distance between the predicted density and its potential ring solution at $t = 0, 10, 50,$ and $1000$ (yellow = large, dark blue = small). 
Type~II regions are shaded for clarity. 
The final panel summarizes the four equilibrium regimes (Types~I–IV). 
    }
\label{fig:ringtest}
\end{figure}

Now, we evaluate the prediction performance of the trained operator across the entire training domain.
Since different equilibrium types correspond to distinct behaviors of ring formation and breakup, we perform a \emph{ring test} over 1005 $(p,q)$ pairs spanning the training domain.
For each parameter pair, we recursively apply the trained operator starting from the same uniform initial density on $[-1,1]^2$ and predict the current state at times $t = 10, 50, 1000$. 
Each predicted density is represented by 1024 sample points and compared against an ideal ring configuration, also represented by 1024 uniformly sampled points.
We quantify the deviation from a ring structure using the Chamfer distance, defined as the sum of squared nearest-neighbor distances between the two point clouds.
Small Chamfer distances indicate close agreement with a ring configuration, while larger values reflect significant deviations.
The results are shown in \Cref{fig:ringtest}. 
At $t=0$, all $(p,q)$ pairs exhibit large Chamfer distances, as the initial density is non-ring-like.
In the Type~I regime, the distance rapidly decreases by $t=10$ and remains small thereafter, indicating convergence to a stable ring equilibrium.
In the Type~II regime, the density first approaches a ring configuration and then breaks into clusters, characterized by a decrease followed by an increase in the Chamfer distance.
This transition occurs more slowly for smaller $(p,q)$ values, which remain ring-like at $t=10$ and deviate only by $t=50$.
For Type~III parameter pairs, where stable rings fail to form, the Chamfer distance remains large even at $t=1000$ particularly in the lower-left region of the parameter domain.
We clarify that larger $q$ values induce stronger attraction, leading to faster contraction into annular structures.
Since annular configurations are geometrically close to rings, this results in comparatively smaller Chamfer distances in the upper-left region of the plot.
Overall, the observed behaviors are consistent with theoretical predictions for the aggregation model.


\begin{figure}[htp]
  \centering
  \renewcommand{\arraystretch}{0}  
  \setlength{\tabcolsep}{0pt}      

  \newcommand{\Stack}[1]{%
    \begin{tabular}{c}    
      \includegraphics[width=0.24\textwidth]{fig_exp/aggregation/V/#1_Vpred_colored.png}\\[-1pt]
      \includegraphics[width=0.24\textwidth]{fig_exp/aggregation/V/#1_Vref_colored.png}\\
    \end{tabular}%
  }

  \begin{tabular}{cccc}
    \Stack{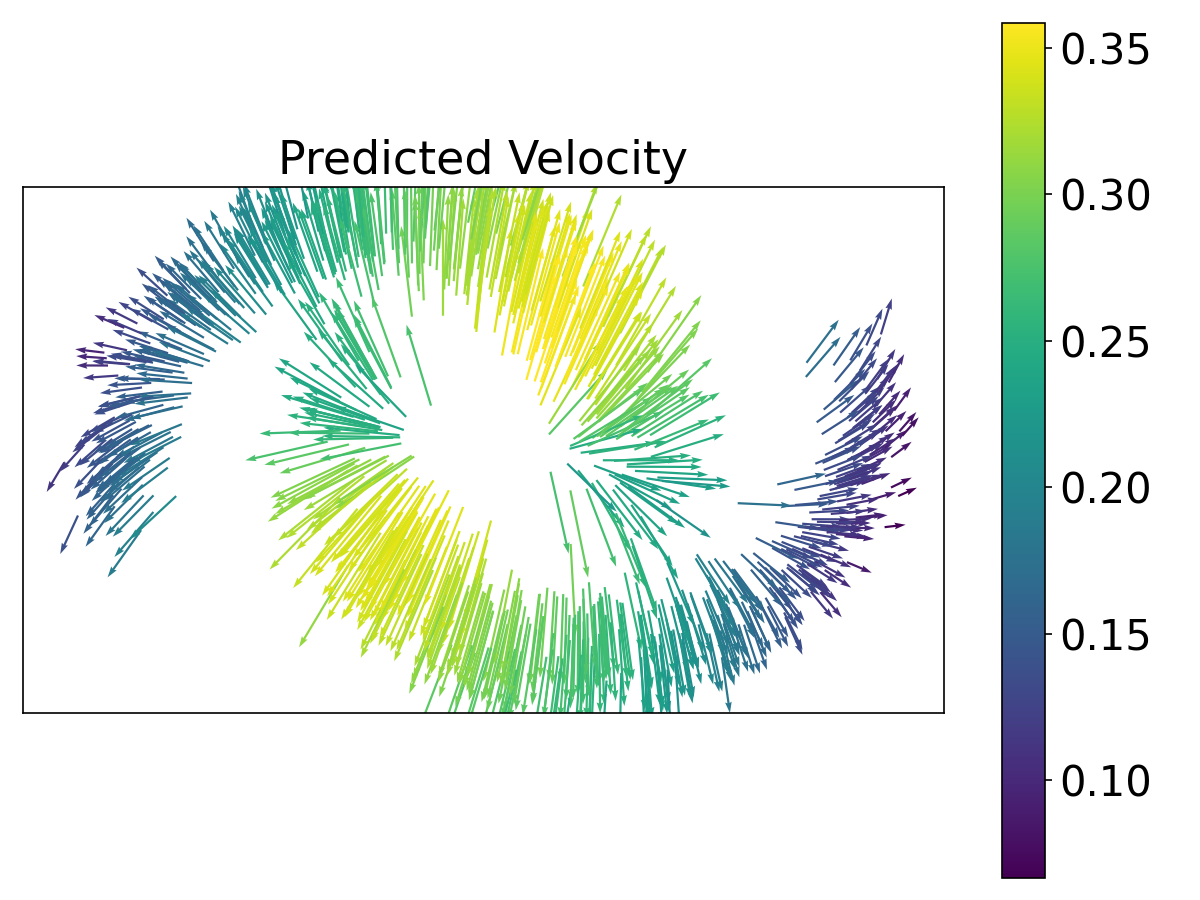} &
    \Stack{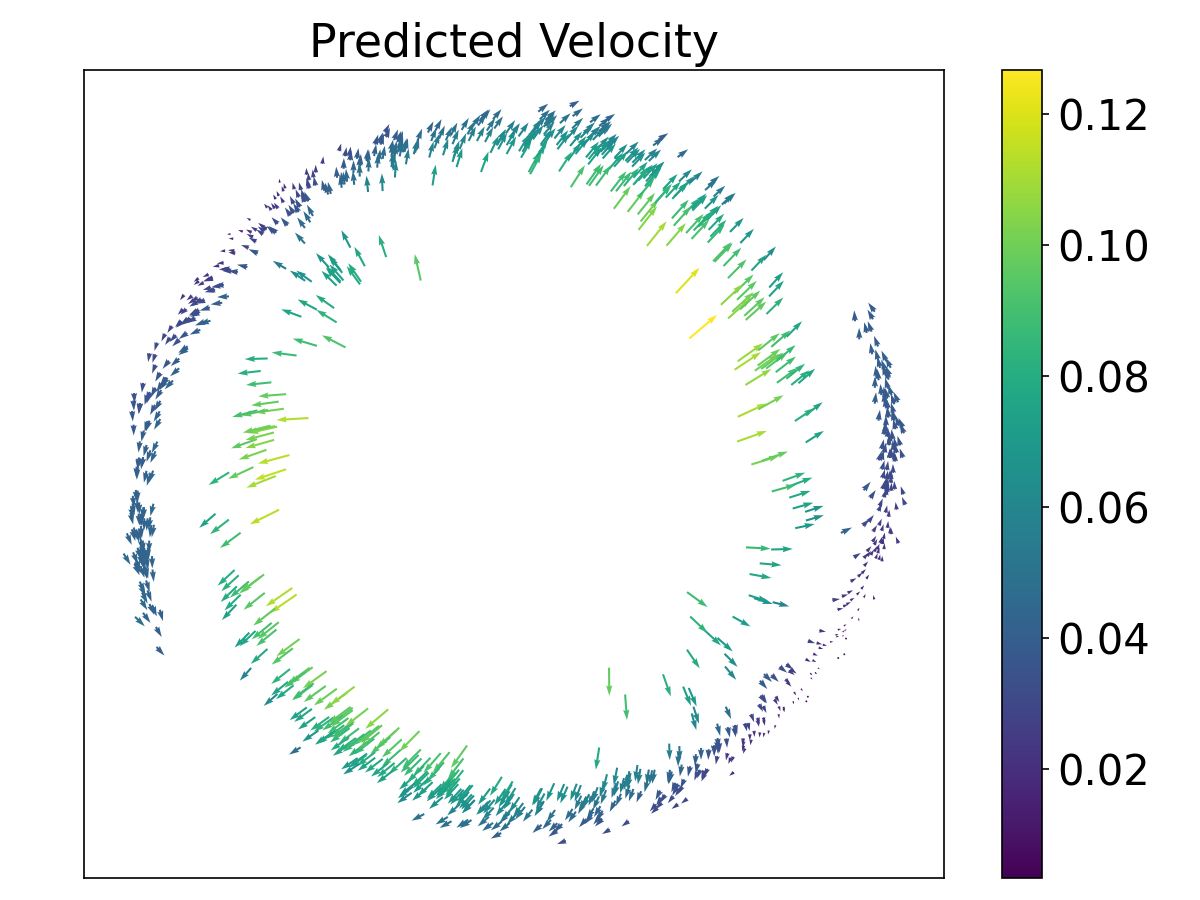} &
    \Stack{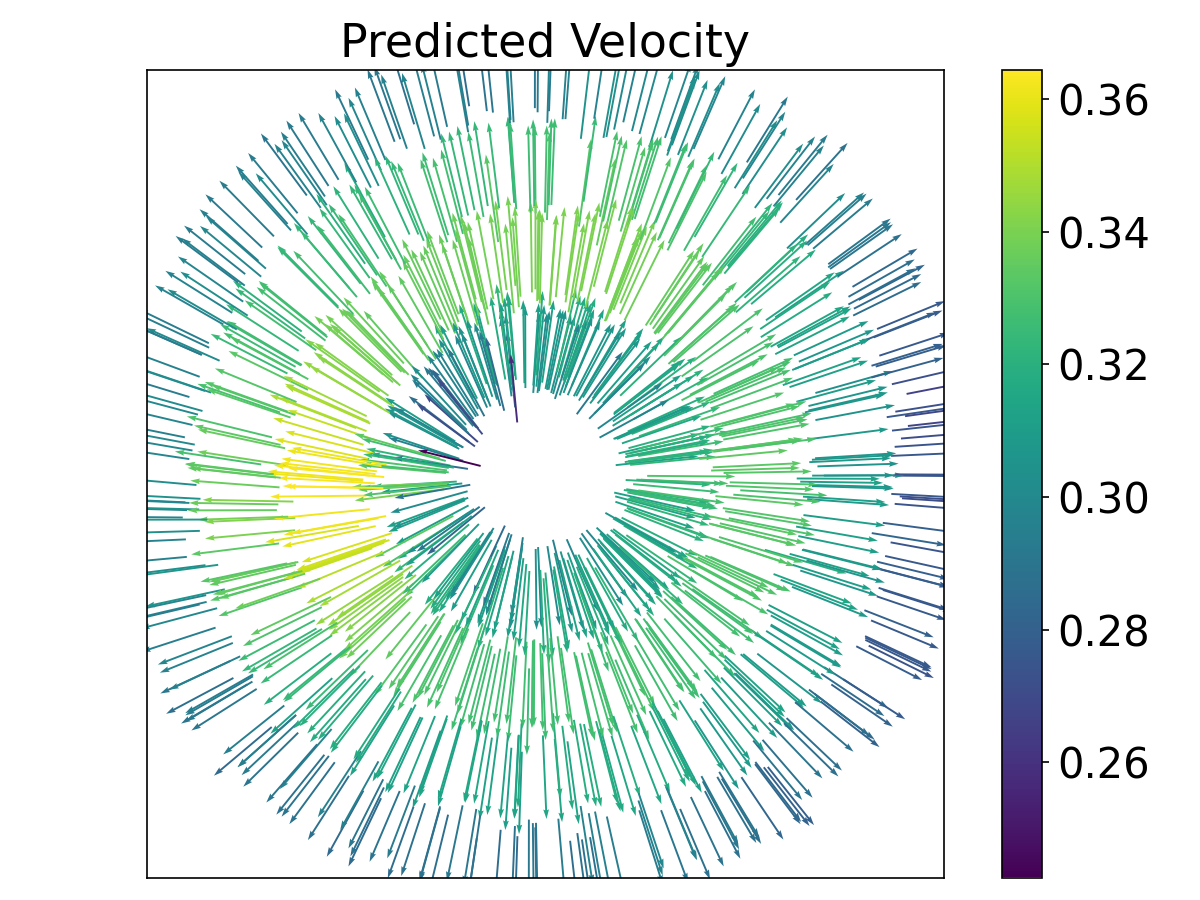} &
    \Stack{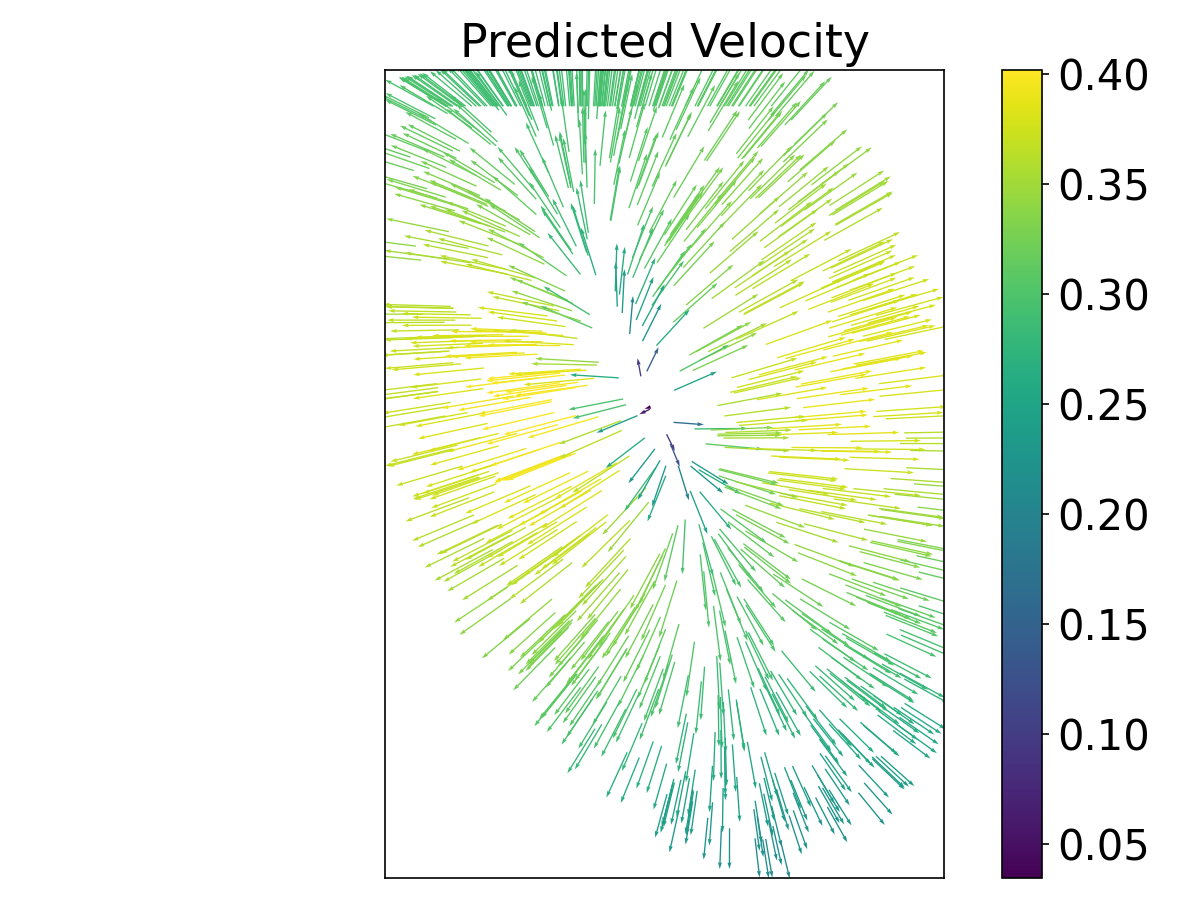}
  \end{tabular}

  \caption{Comparison of predicted JKO velocity fields (top) with reference velocity fields from Equation~\eqref{equ:particle-ode} (bottom) for randomly selected densities with $(p,q) = (0.5,3)$. The color indicates the velocity magnitude.}
  \label{fig:aggV}
\end{figure}

Finally, we validate the learned JKO velocity field $\mathbf{V}_{\JKO}$ 
by comparing it with the exact velocity field from the particle ODEs in~\Cref{equ:particle-ode}. 
In theory, $\mathbf{V}_{\JKO}$ approximates the WGF velocity when the stepsize~$\Delta t$ is small. 
Across different input densities, the predicted velocities closely match the reference fields (see~\Cref{fig:aggV}), 
demonstrating that the learned JKO operator accurately captures the underlying dynamics of the aggregation equation.

\subsection{Porous Medium Equation}
We consider the porous medium equation given by: 
\begin{equation}\label{eq:porous-pde}
    \partial_t \rho(t,x) = \Delta \rho(t,x)^m, 
    \qquad m>1,\; x\in\mathbb{R}^d\,,
\end{equation}
which can be viewed as the Wasserstein gradient flow of the internal energy:
\begin{equation}
    \label{equ:porousenergy}
        \mathcal{E}(\rho)
    = \frac{1}{m-1}\int_{\mathbb{R}^d}\rho(x)^m\,dx,  \qquad m >1\,.
\end{equation}
In the whole space $\mathbb{R}^d$, finite-mass solutions admit no nontrivial steady states; instead, they spread in a self-similar manner, preserving total mass while decaying pointwise to zero. The corresponding self-similar Barenblatt solution for a Dirac delta initial condition is given by
\begin{equation}\label{equ:porousBB}
    \rho(t,x)
    = (t+t_0)^{-\alpha}
      \Bigl(C - \beta \|x\|^2 (t+t_0)^{-2\alpha/d}\Bigr)_{+}^{1/(m-1)},
\end{equation}
where
\(
    \alpha = \frac{d}{d(m-1)+2}, 
    \beta = \frac{(m-1)\alpha}{2dm},
\)
and the parameters $C>0$ and $t_0>0$ control the total mass and the temporal shift, respectively. It is clear that for $\rho$ of the form \eqref{equ:porousBB}, the corresponding velocity field can be computed as:
\[
 v(t,x)
    = -\,\nabla_x\!\left(\frac{\delta \mathcal{E}}{\delta \rho}\right)
    = -\,\nabla_x\!\Bigl(\tfrac{m}{m-1}\rho^{\,m-1}\Bigr) =
    \frac{\alpha}{d}\,\frac{x}{t+t_0},
\]
which is radial, linear in $x$, and decays in time like $(t+t_0)^{-1}$ for each fixed $x$. Its divergence is
\begin{equation}\label{equ:porousdiv}
    \operatorname{div} v(t,x)
    = \frac{d}{d(m-1)+2}\,\frac{1}{t+t_0}\,,
\end{equation}
which is spatially homogeneous.

\paragraph{Training setup}
We train separate JKO operators for different choices of the parameters $d$, $m$ (both appearing in \eqref{equ:porousenergy}), and $\Delta t$. During training, the initial densities $\rho(0,\cdot)$ are given by the Barenblatt profile \eqref{equ:porousBB} with $t_0 = 10^{-3}$ and with the parameter $C$ sampled uniformly from $[0.1, 1]$. All models are trained for $T = 40$ JKO steps (i.e., evolving \eqref{eq:porous-pde} from $t = 0$ to $t = 40\Delta t$). The batch sizes of initial densities are set to $6$, $4$, $2$, and $1$ for $d = 1$, $2$, $5$, and $10$, respectively.

The neural network input is a density represented by its sample points. 
To improve accuracy, the input is augmented by concatenating its density values to each sample point. 
Radial sampling is used to generate isotropic point clouds. 
To mitigate high-dimensional mass concentration where few sample points lie near the origin, we introduce a small number of landmark points near the origin, and also recenter the sample points to achieve approximately zero-mean sampling for the initial density. 
For generalization tests, the trained models are iteratively applied for 100 steps, which is beyond the training $40$ steps.

\begin{figure}[h]
\centering
\begin{subfigure}[b]{0.4\textwidth}
    \includegraphics[width=\textwidth]{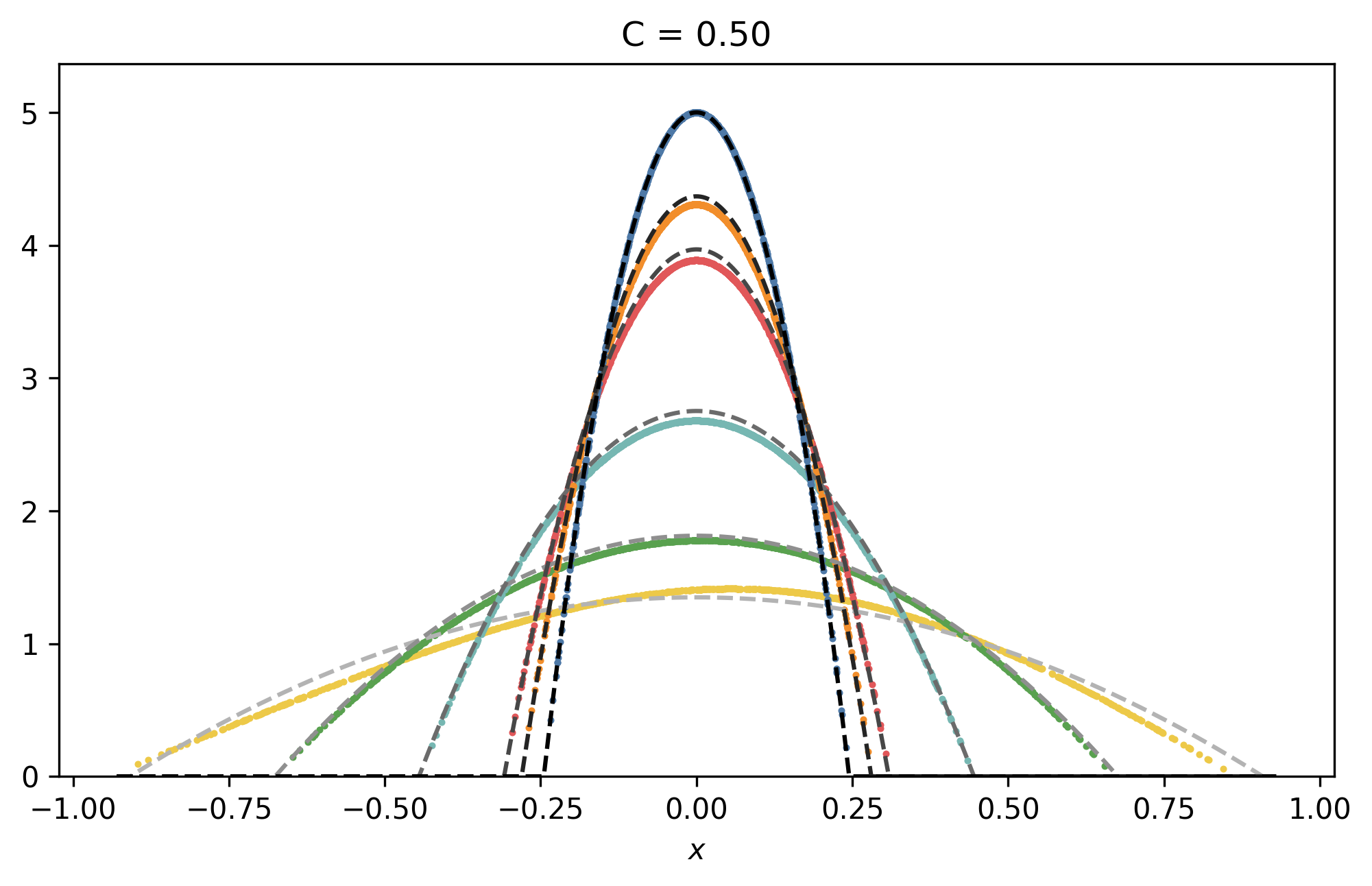}
    \caption*{$d=1$}
\end{subfigure}
\begin{subfigure}[b]{0.28\textwidth}
    \includegraphics[width=\textwidth]{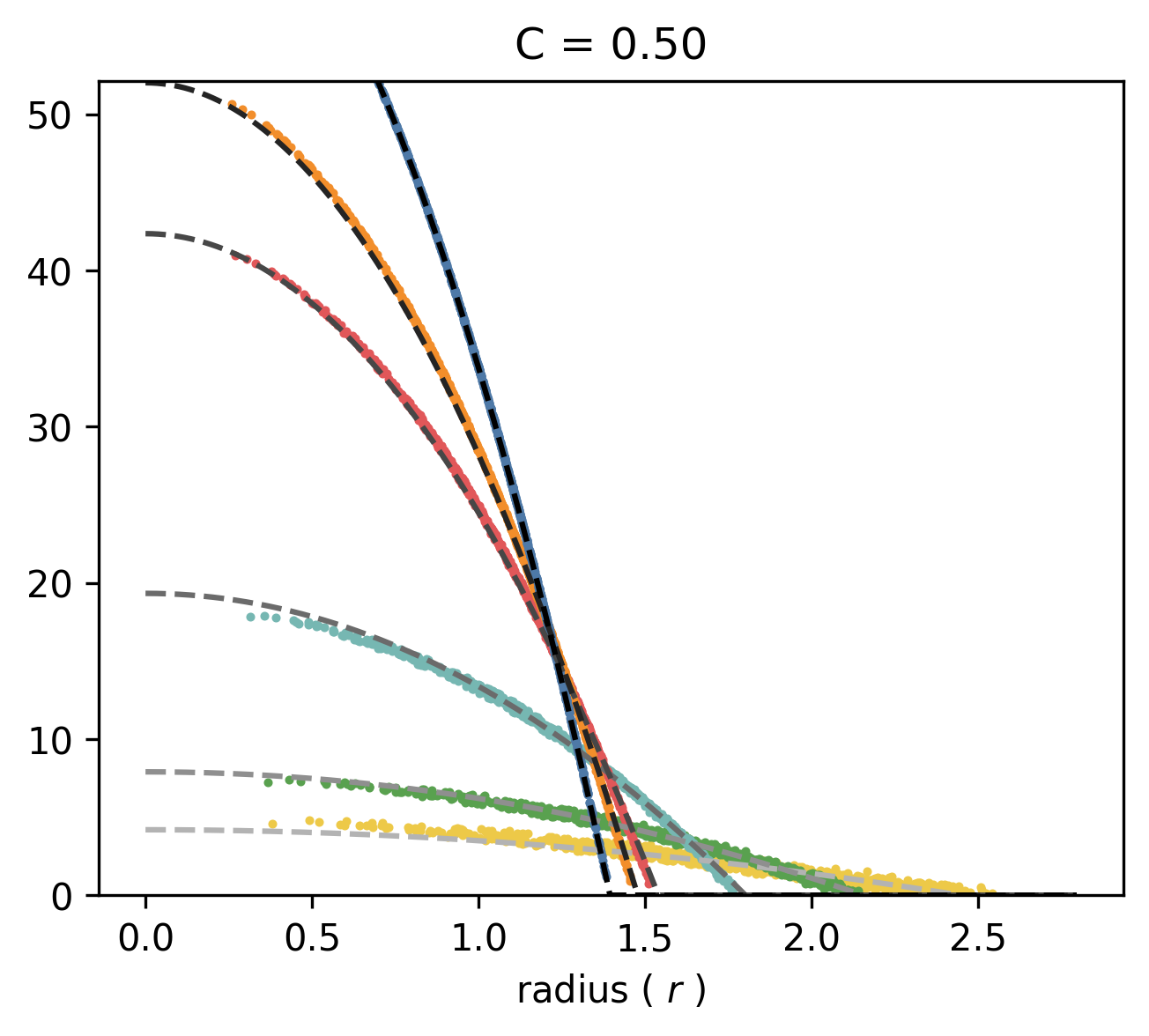}
    \caption*{$d=5$}
\end{subfigure}
\hfill
\begin{subfigure}[b]{0.28\textwidth}
    \includegraphics[width=\textwidth]{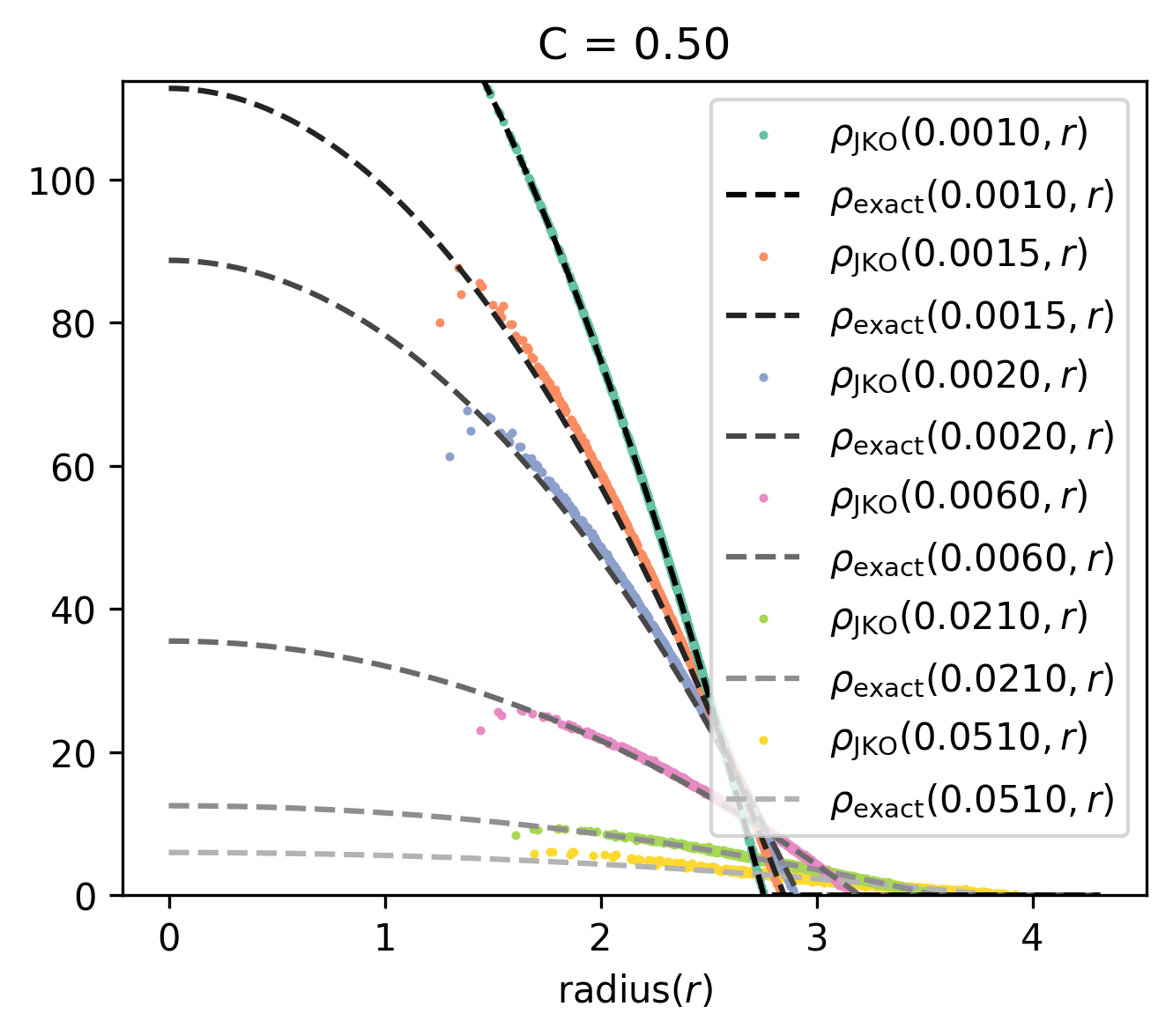}
    \caption*{$d=10$}
\end{subfigure}
\caption{Predicted porous-medium flows (solid) versus exact Barenblatt solutions (dashed) starting from $\rho^0$ with $C=0.5$ across data dimensions $d$ with fixed $\Delta t=0.0005$ and $m=2$.}
\label{fig:porousflowdifferentd}
\end{figure}

\begin{figure}[h!]
\centering
\begin{subfigure}[b]{0.32\textwidth}
    \includegraphics[width=\textwidth]{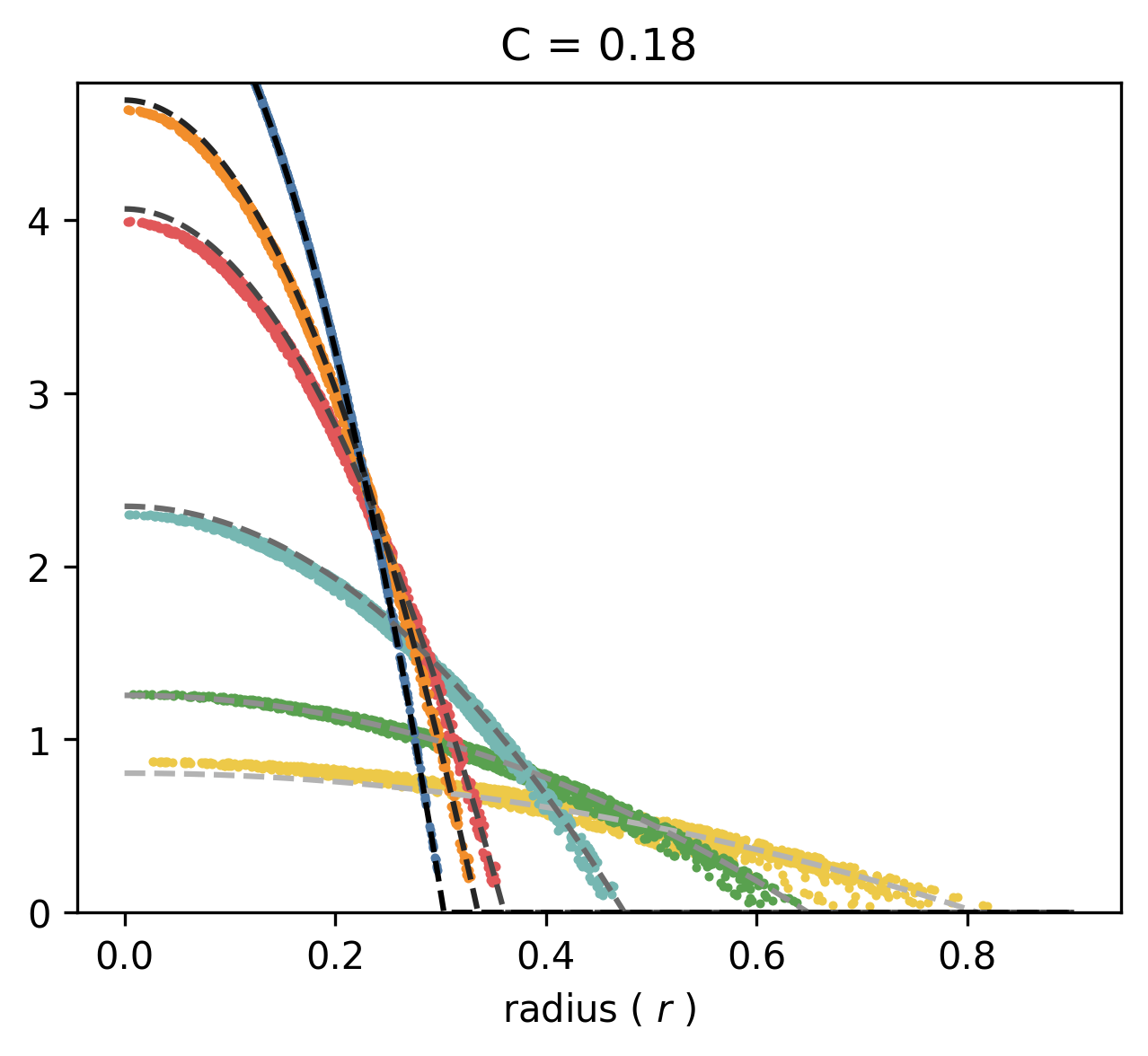}
\end{subfigure}
\begin{subfigure}[b]{0.32\textwidth}
    \includegraphics[width=\textwidth]{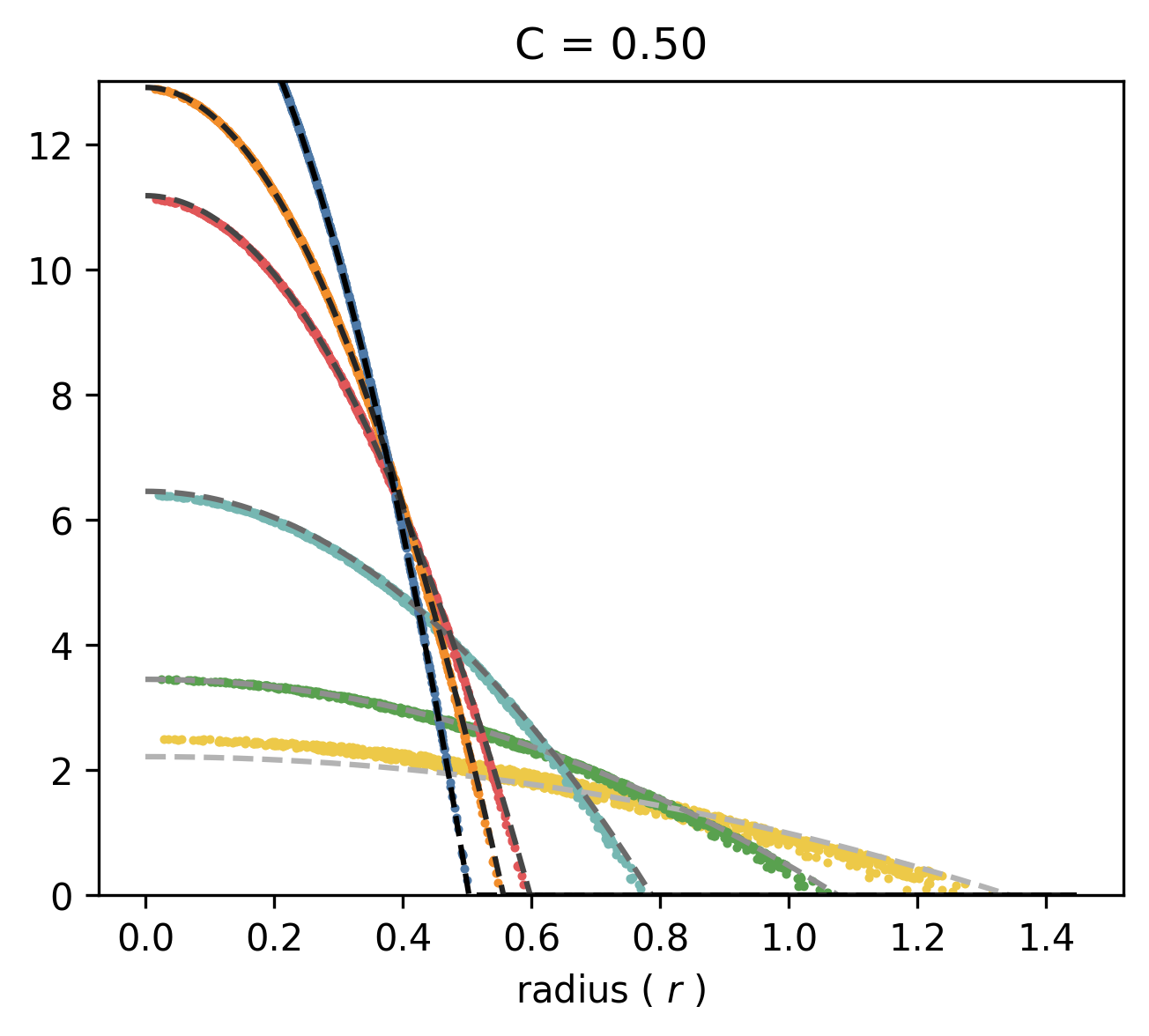}
\end{subfigure}
\begin{subfigure}[b]{0.32\textwidth}
    \includegraphics[width=\textwidth]{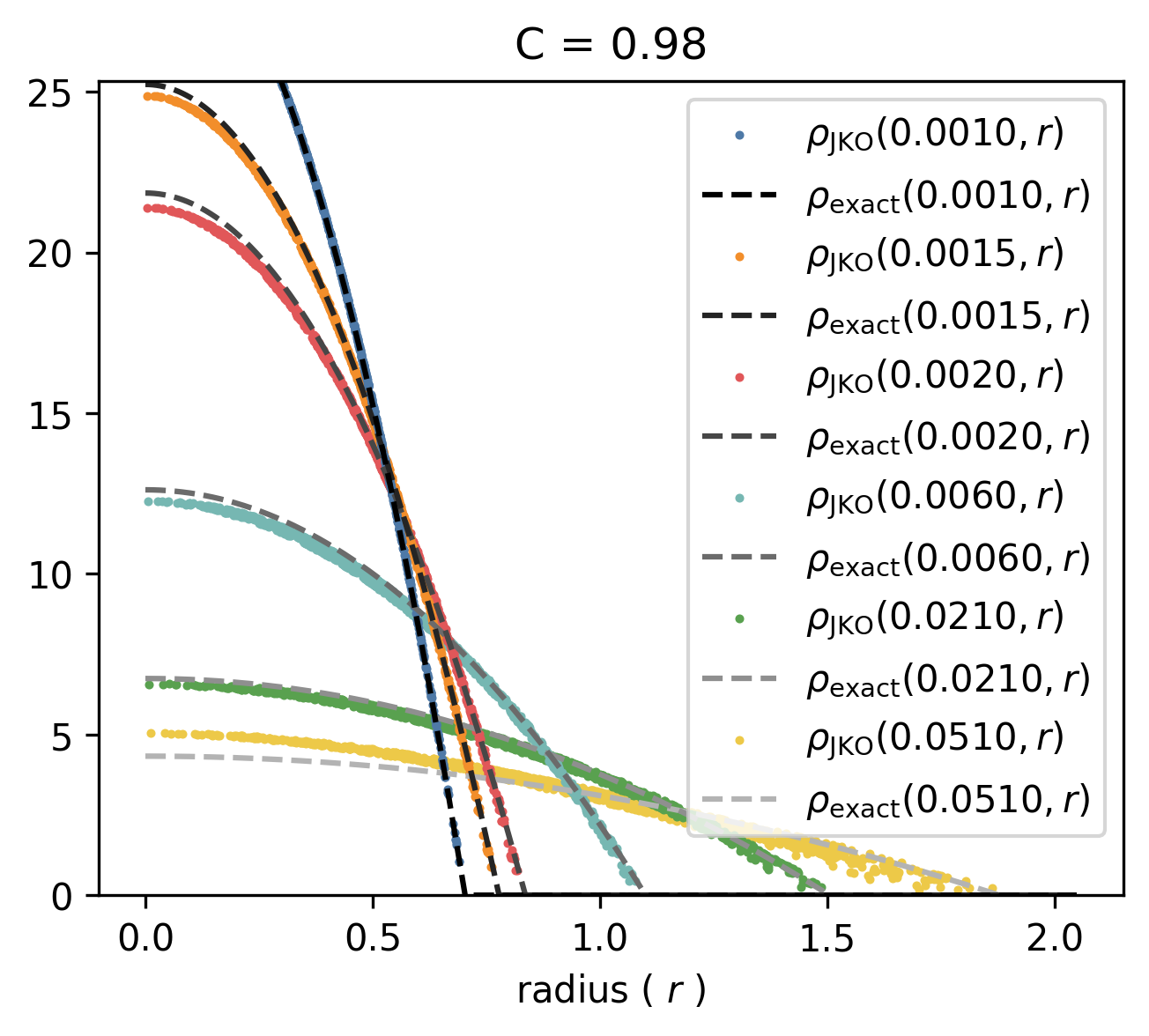}
\end{subfigure}
\\
\begin{subfigure}[b]{0.32\textwidth}
    \includegraphics[width=\textwidth]{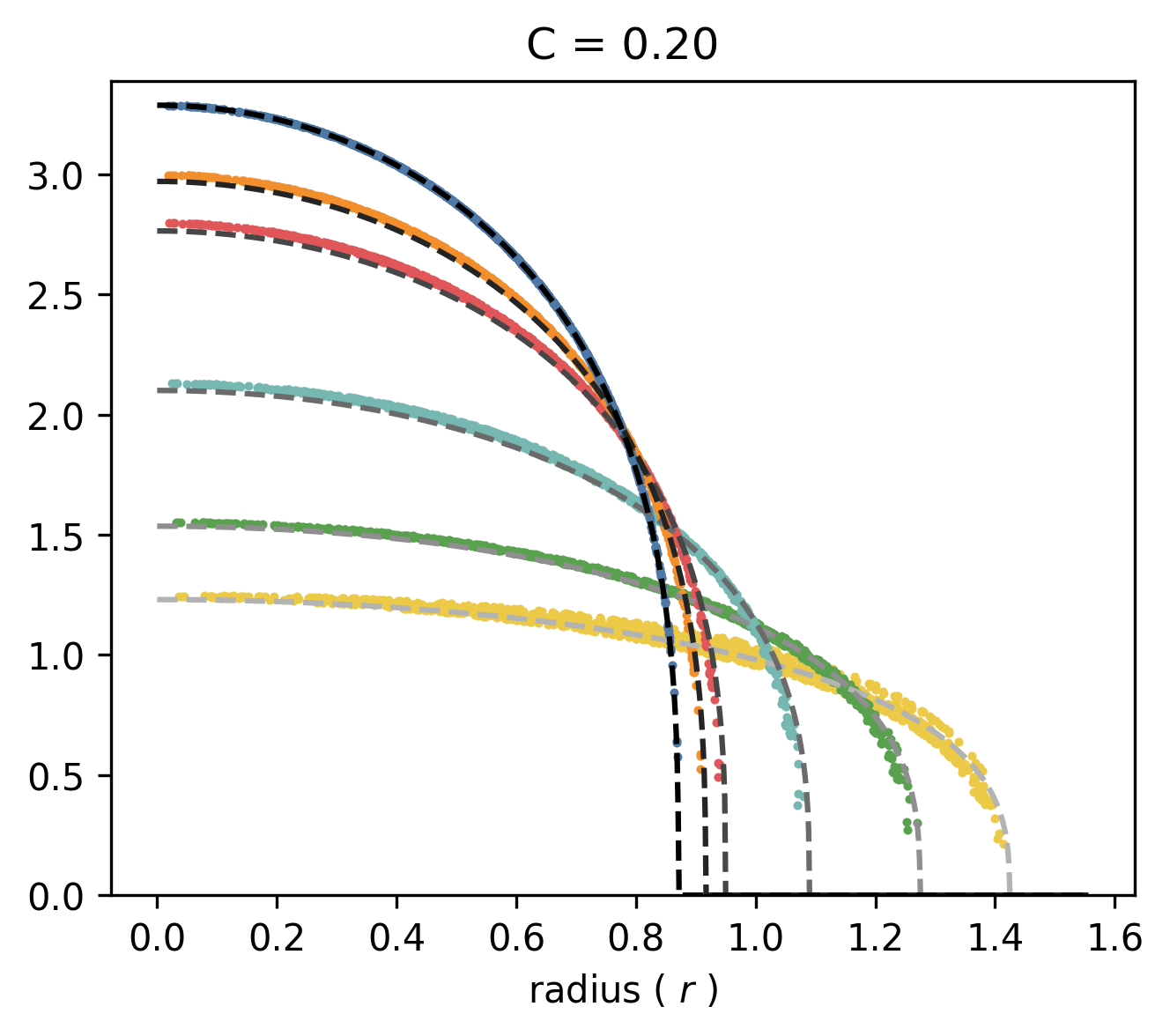}
\end{subfigure}
\begin{subfigure}[b]{0.32\textwidth}
    \includegraphics[width=\textwidth]{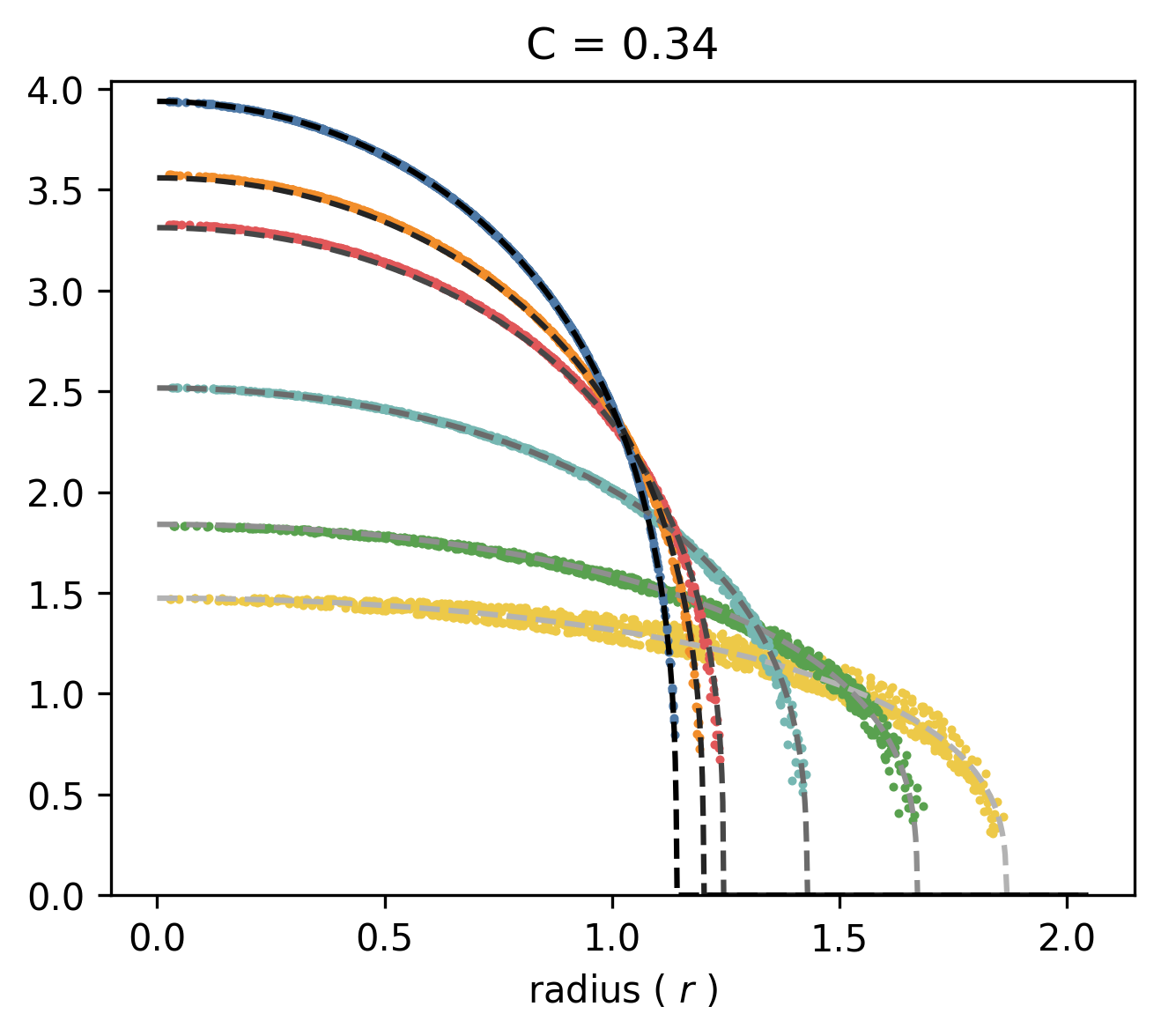}
\end{subfigure}
\begin{subfigure}[b]{0.32\textwidth}
    \includegraphics[width=\textwidth]{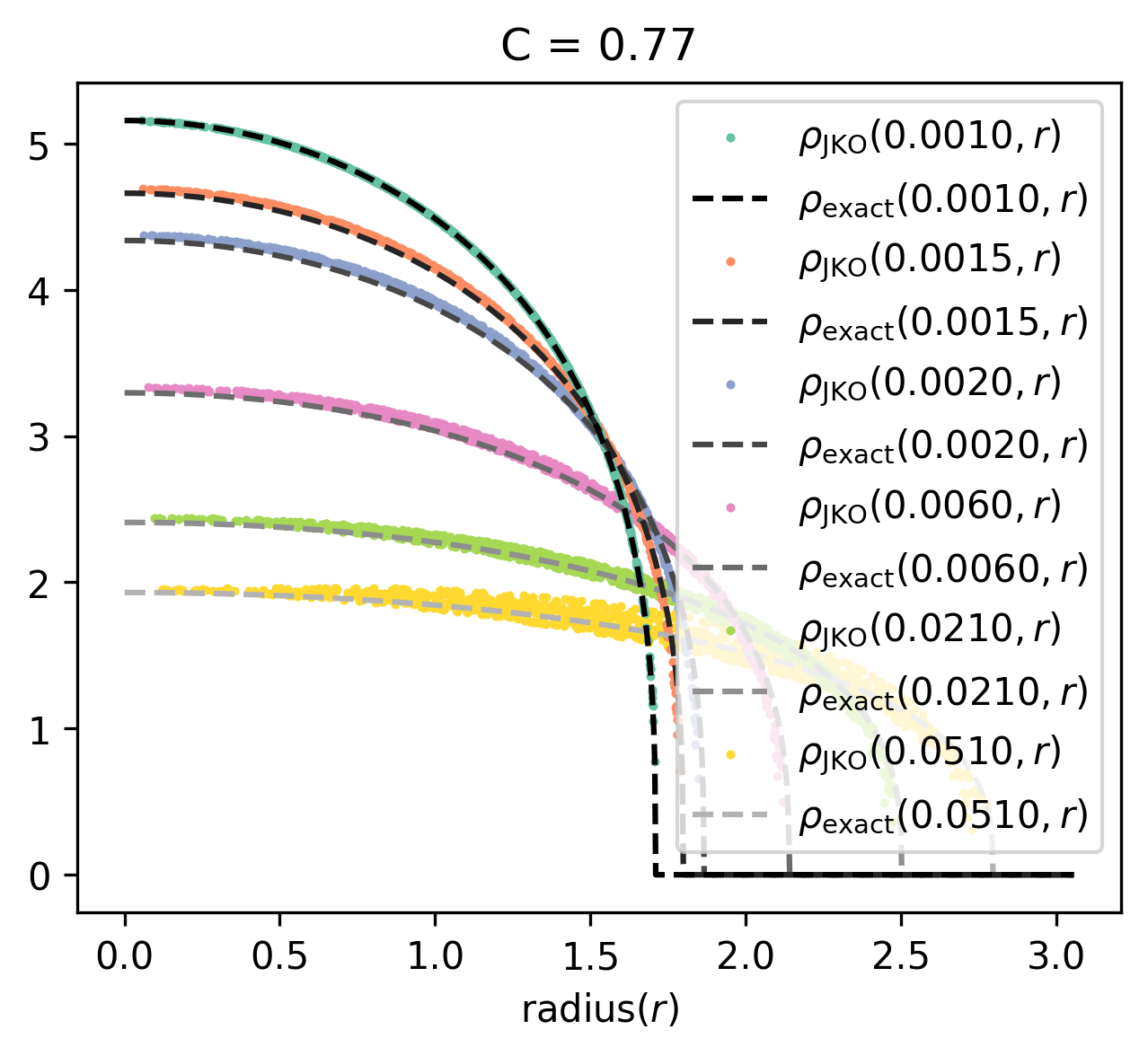}
\end{subfigure}
\caption{
Predicted porous medium flows (solid) versus exact Barenblatt solutions (dashed) for $\Delta t=0.0005$ and $d=2$. 
Each row shows three predicted
evolutions starting from $\rho^0$  
with different $C$ values, using the same JKO operator ($m=2$ top, $m=4$ bottom). 
Although trained only up to $t\le0.021$, the operator generalizes well beyond this range. 
The vertical axis is truncated to highlight the sharp density peaks.}
\label{fig:porousflowdim2}
\end{figure}

\begin{figure}[h!]
\centering
\begin{subfigure}[b]{0.32\textwidth}
    \includegraphics[width=\textwidth]{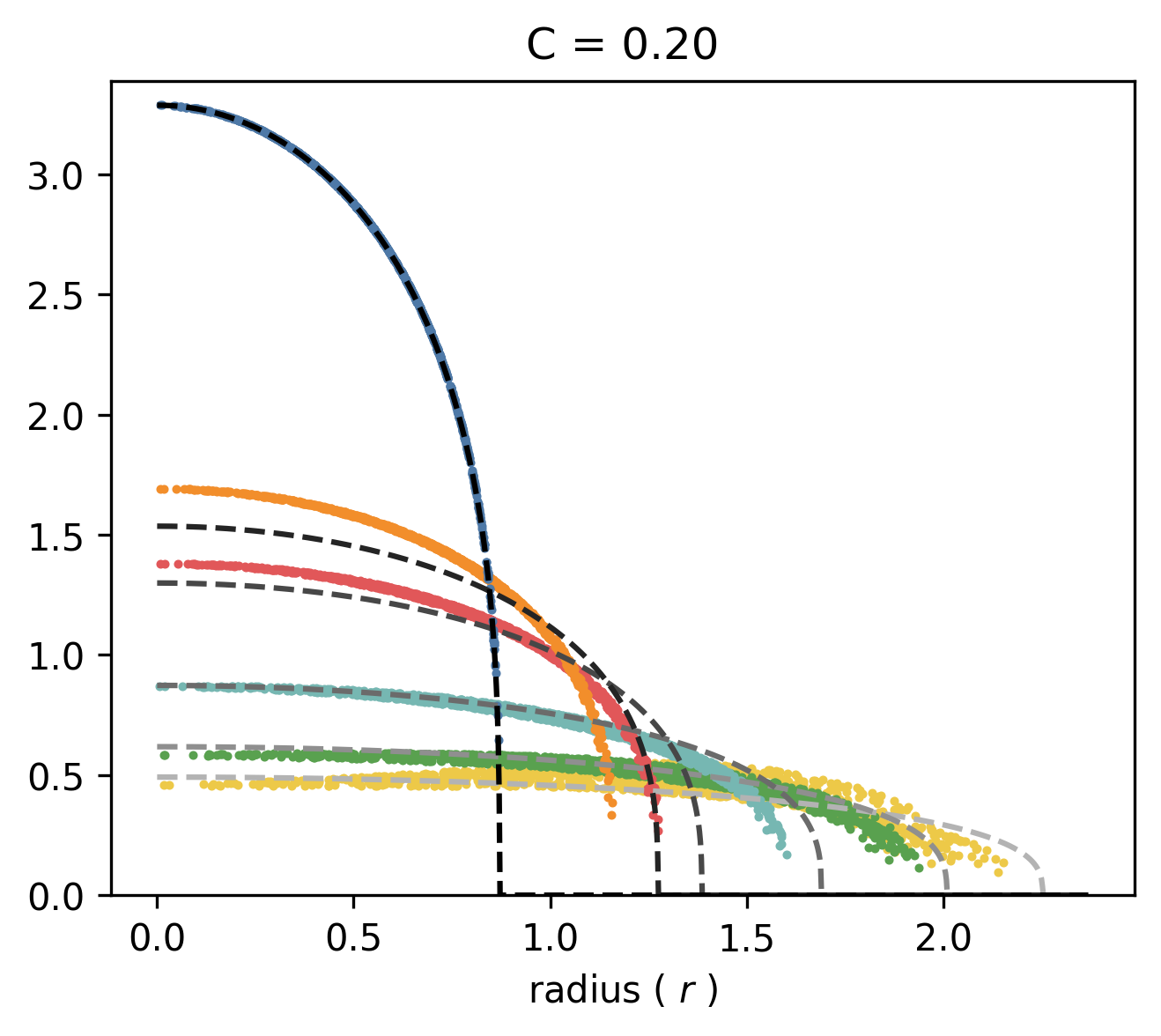}
\end{subfigure}
\begin{subfigure}[b]{0.32\textwidth}
    \includegraphics[width=\textwidth]{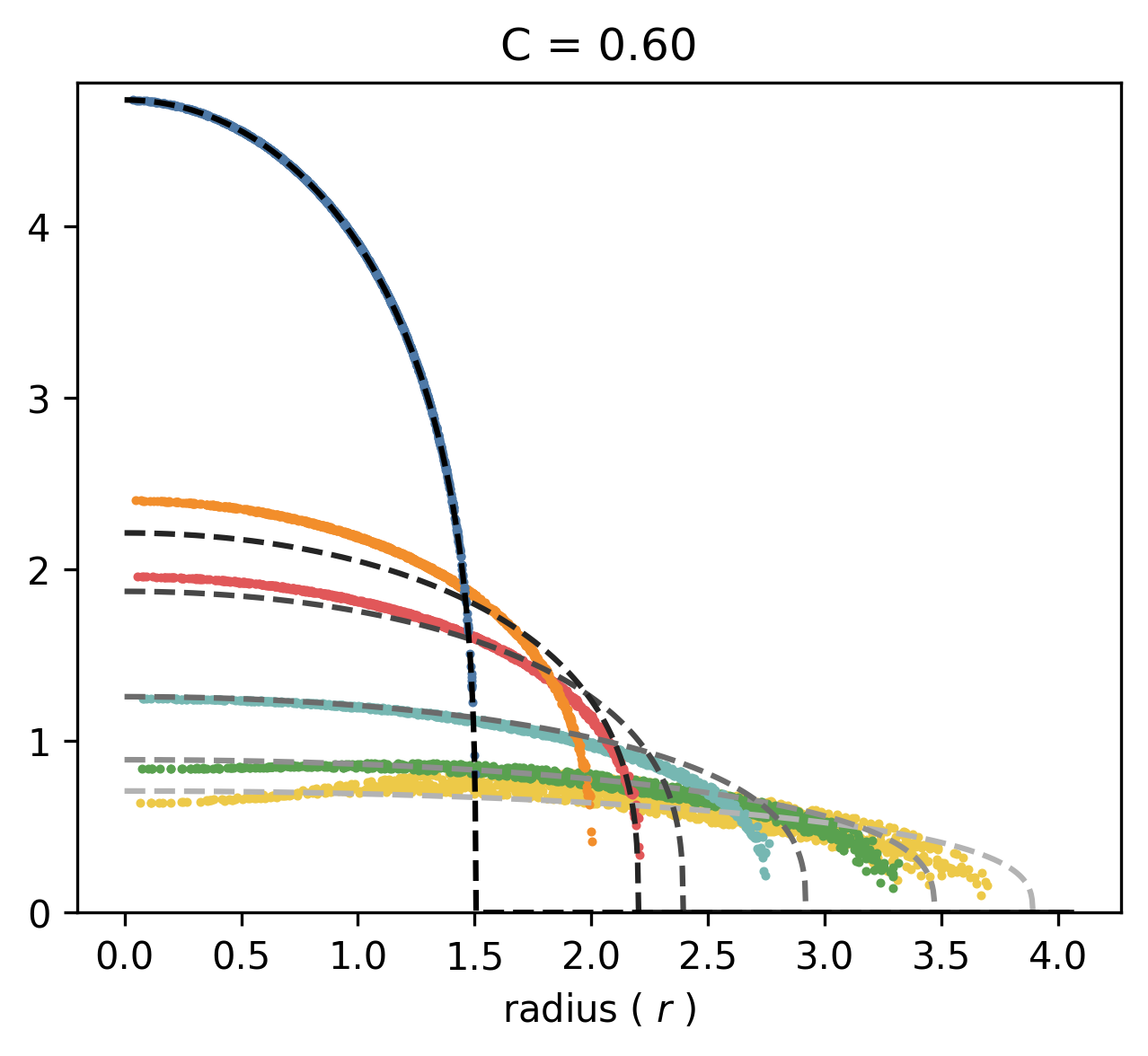}
\end{subfigure}
\begin{subfigure}[b]{0.32\textwidth}
    \includegraphics[width=\textwidth]{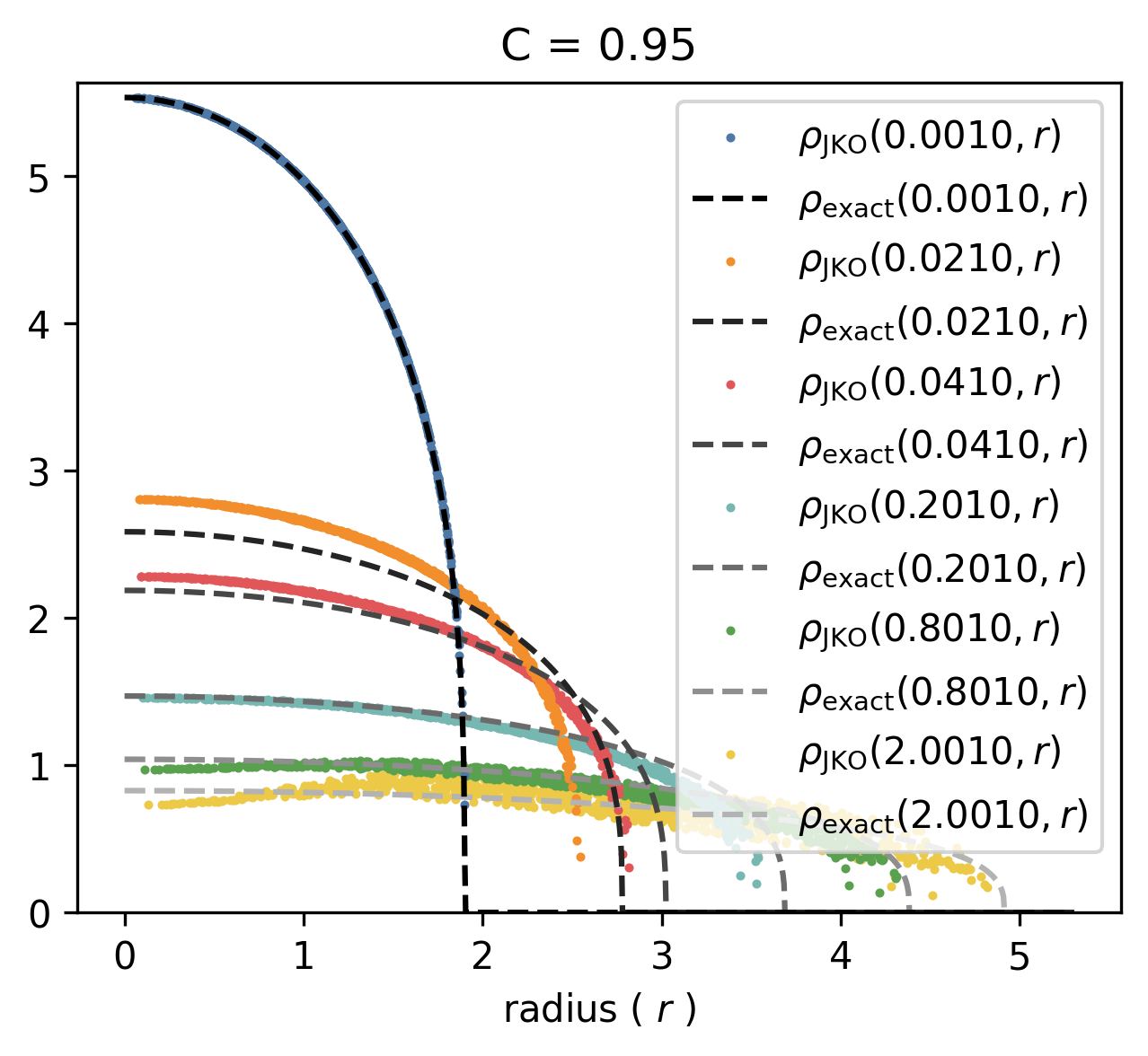}
\end{subfigure}
\caption{Comparison of predicted porous medium flows (solid) and exact Barenblatt solutions (dashed) for $\Delta t=0.02$, $d=2$, and $m=4$, across different values of $C$.}

\label{fig:porousflowdim2_2}
\end{figure}

\paragraph{Prediction performance}
To assess accuracy, we compare our predicted solutions with the exact Barenblatt profiles, as shown in Figures~\ref{fig:porousflowdim2}, \ref{fig:porousflowdim2_2}, and~\ref{fig:porousflowdifferentd}. In the one-dimensional case (first column of \Cref{fig:porousflowdifferentd}), the predictions match the analytic solution closely when a sufficiently small stepsize, $\Delta t = 0.0005$, is used. For $d>1$, densities are plotted in radial coordinates across all figures.
Across different stepsizes and choices of $C$ in the initial density, the predicted solutions remain nearly isotropic, consistent with the Barenblatt form. With a small stepsize (Figures~\ref{fig:porousflowdim2} and~\ref{fig:porousflowdifferentd}), the results match the exact profiles well across all tested $C$ values and dimensions $d$.

In \Cref{fig:porousflowdim2_2}, we further shows results with a larger step size $\Delta t = 0.02$ and $m=4$, a setting where 
the optimization problem is too ill-conditioned for traditional numerical solvers such as primal–dual or back-and-forth methods.
In contrast, our predictions remain stable and smooth.

We also quantitatively evaluate the errors between our predicted solutions and the exact Barenblatt profiles using relative $L_1$ and $L_\infty$ norms:
\[
L_1 \text{ Error}
 = 
\frac{
  \sum_i
    \bigl|\rho_{\mathrm{JKO}}(x_i,t)-\rho_{\mathrm{exact}}(x_i,t)\bigr|
}{
 \sum_i
\rho_{\mathrm{JKO}}(x_i,t)
}, \quad
L_{\infty} \text{ Error}
 =
\frac{
  \sup_{i}\bigl|\rho_{\mathrm{JKO}}(x_i,t)-\rho_{\mathrm{exact}}(x_i,t)\bigr|
}{
  \sup_{i}\rho_{\mathrm{exact}}(x_i,t)
},
\]
computed for each $t = 1, 2, \dots, T$. Here, ${x_i} \sim \rho_{\mathrm{JKO}}(x,t)$, and $T = 40$ and $T = 100$ correspond to the training and testing horizons, respectively. That is, each operator is trained to perform the first 40 JKO steps, and once trained, it is applied to compute 100 JKO steps.

The total error consists of four components: (i) the $\mathcal O(\Delta t)$ error between the JKO solution and WGF,
(ii) the statistical error  from approximating the JKO problem with a finite number of particles,
(iii) the training error, and (iv) the approximation error due to network capacity. 
Table~\ref{tab:porouserrors} reports the results.
For $m=2$, the prediction errors scale approximately as $\mathcal{O}(\Delta t)$ for moderately small step sizes, colored blue and orange, consistent with theoretical expectations.
For extremely small $\Delta t$, errors increase slightly, likely due to training noise or limited numerical precision.
Because the Barenblatt solutions are radially symmetric, errors remain comparable across dimensions $d=1,2,5,10$ for $\Delta t=0.0005$.
In higher dimensions, the scarcity of sample points near the origin leads to less accurate reconstruction in that region (see~\Cref{fig:porousflowdifferentd}) and consequently larger $L_\infty$ errors (Table~\ref{tab:porouserrors}).

\begin{table}[htp]
\centering
\scriptsize
\begin{tabular}{
    ccl
    llll
}
\toprule
$\mathbf{m}$ & 
{\textbf{dim}} & 
{$\Delta t$} & 
\multicolumn{2}{c}{\textbf{$T=40$(training)}} & 
\multicolumn{2}{c}{\textbf{$T=100$(testing)}} \\
\cmidrule(lr){4-5} \cmidrule(lr){6-7}
& & & 
{$L_1$ Error} & {$L_{\infty}$ Error} & 
{$L_1$ Error} & {$L_{\infty}$ Error} \\ 
\midrule

$2$  &  \textbf{1} & \textbf{0.0005} &  \textbf{\color{cyan} 0.022}(0.006) &  \textbf{0.033}(0.009) &  \textbf{0.026}(0.010) &  \textbf{0.050}(0.020) \\

$2$  & 1 & 0.001  &  {\color{cyan}0.041}(0.011) & 0.054(0.010) & 0.050(0.014) & 0.073(0.022) \\
$2$ & 1 & 0.002 & {\color{cyan}0.089}(0.006) & 0.085(0.009) & 0.095(0.019) & 0.121(0.041) \\
$2$ & 1 & 0.005  &  {\color{cyan}0.225}(0.014) & 0.189(0.015) & 0.235(0.018) & 0.211(0.029) \\
\midrule
$2$ & 2 & 0.00001 & 0.015(0.007) & 0.046(0.022) & 0.032(0.018) & 0.086(0.043) \\
$2$ & 2 & 0.0001 & 0.017(0.009) & 0.038(0.016) & 0.033(0.021) & 0.074(0.044) \\

$2$  & \textbf{2} &  \textbf{0.0005} &   \textbf{\color{orange}0.020}(0.006) &  \textbf{0.059}(0.022) &  \textbf{0.035}(0.018) &  \textbf{0.117}(0.060) \\

$2$  &  2 & 0.001  & {\color{orange}0.038}(0.008) & 0.079(0.022) & 0.053(0.021) & 0.107(0.036) \\
$2$   & 2 & 0.002  & {\color{orange}0.090}(0.015) & 0.099(0.015) & 0.100(0.014) & 0.128(0.034) \\
$2$ & 2 & 0.005 & 0.164(0.036) & 0.195(0.017) & 0.182(0.039) & 0.245(0.074) \\
$2$ & 2 & 0.010 & 0.269(0.028) & 0.244(0.022) & 0.269(0.026) & 0.299(0.064) \\
$2$ & 2 & 0.020 & 0.422(0.019) & 0.300(0.024) & 0.388(0.056) & 0.364(0.070) \\
\midrule

$2$  &  \textbf{5} &  \textbf{0.0005} &  \textbf{0.027}(0.011) &  \textbf{0.071}(0.029) &  \textbf{0.047}(0.028) &  \textbf{0.107}(0.058) \\
$2$  &  \textbf{10} &  \textbf{0.0005} &  \textbf{0.030}(0.008) &  \textbf{0.373}(0.240) &  \textbf{0.053}(0.026) &  \textbf{0.369}(0.227) \\

\midrule
$4$ & 2 & 0.0005 & 0.011(0.003) & 0.128(0.043) &  0.015(0.005) & 0.162(0.047)\\
$4$ & 2 & 0.001  & 0.021(0.005) & 0.194(0.050) &0.034(0.015) & 0.241(0.062) \\
$4$ & 2 & 0.002  & 0.026(0.004) & 0.216(0.046) & 0.036(0.011) & 0.248(0.062) \\
$4$ & 2 & 0.020 & 0.051(0.010) & 0.474(0.028) &  0.062(0.010) & 0.273(0.051) \\
$4$ & 2 & 0.200 & 0.128(0.022) & 0.441(0.027) & 0.120(0.019) & 0.476(0.041) \\
$4$ & 2 & 2.000 &  0.254(0.041) & 0.556(0.017)& 0.222(0.042) & 0.571(0.024) \\
\bottomrule
\end{tabular}

\captionof{table}{Error statistics for our predictions versus the exact Barenblatt solutions (relative $L_1$ and $L_\infty$) are reported as the mean (standard deviation) over all $5 \times T$ sample points, corresponding to $5$ random choices of $C$, each with $T$ time steps.}
\label{tab:porouserrors}
\end{table}

\begin{figure}[h!]
\centering

\begin{subfigure}[b]{0.193\textwidth}
    \includegraphics[width=\textwidth]{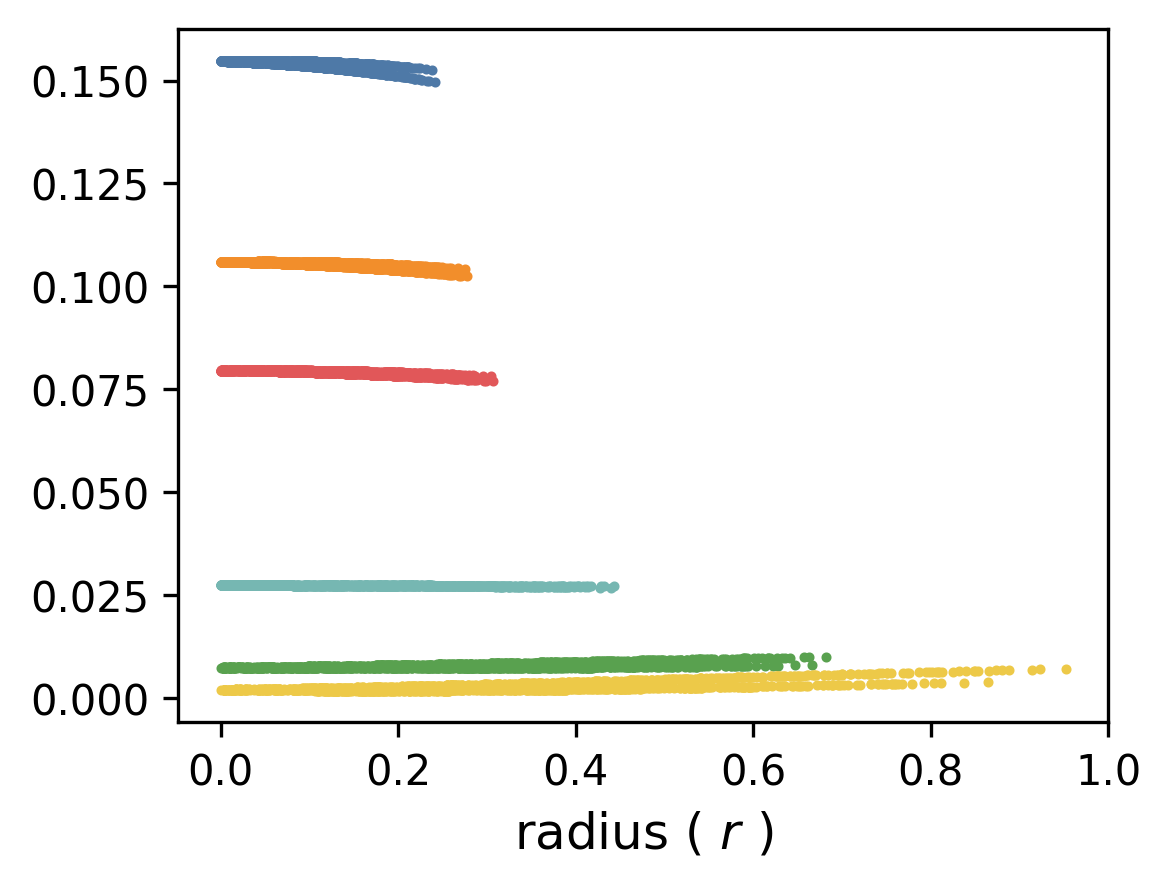}
    \caption*{$d=$1}
\end{subfigure}
\begin{subfigure}[b]{0.193\textwidth}
    \includegraphics[width=\textwidth]{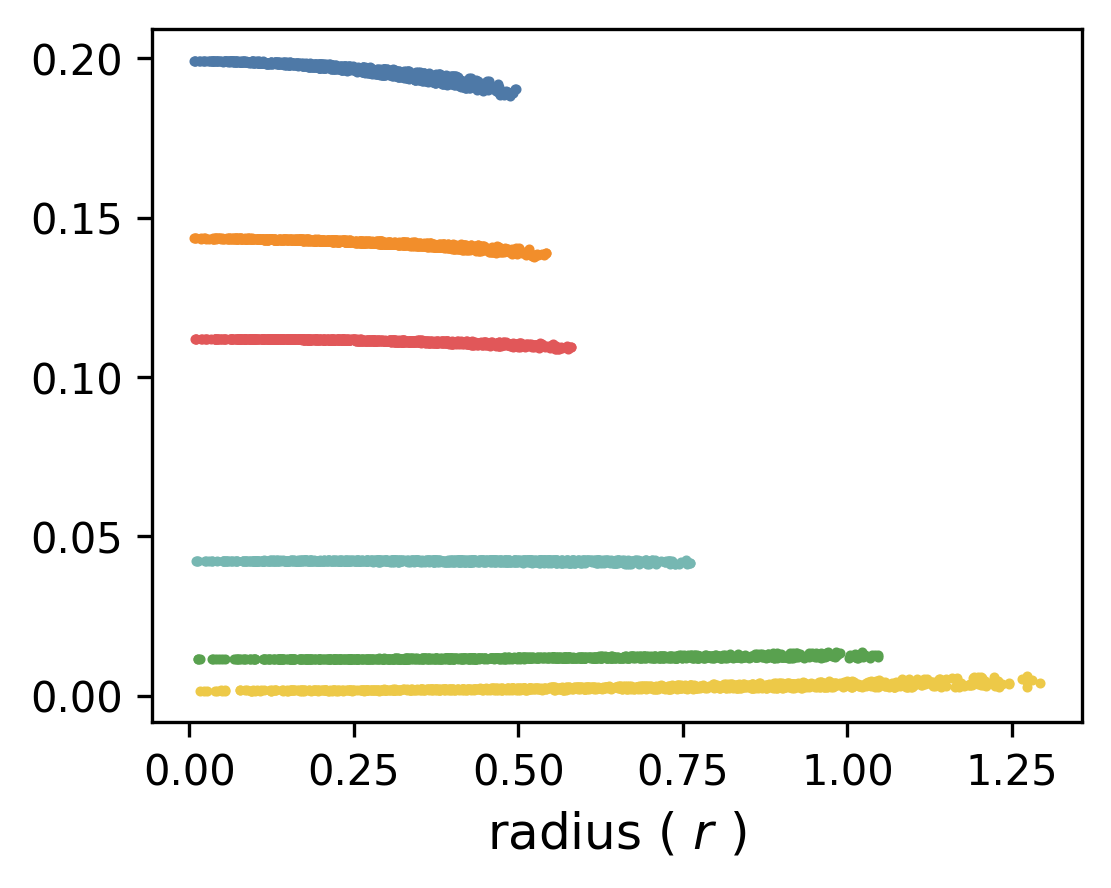}
    \caption*{$d=$2}
\end{subfigure}
\begin{subfigure}[b]{0.193\textwidth}
    \includegraphics[width=\textwidth]{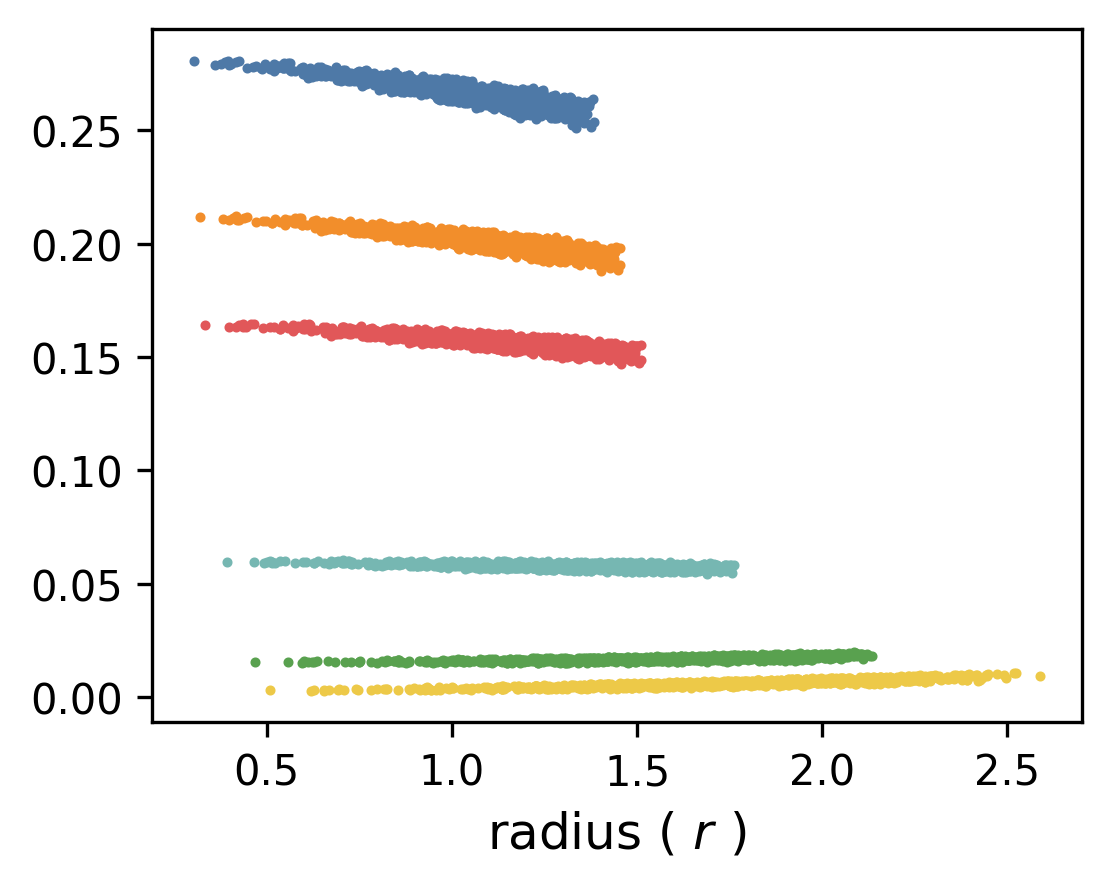}
    \caption*{$d=$5}
\end{subfigure}
\hfill
\begin{subfigure}[b]{0.193\textwidth}
    \includegraphics[width=\textwidth]{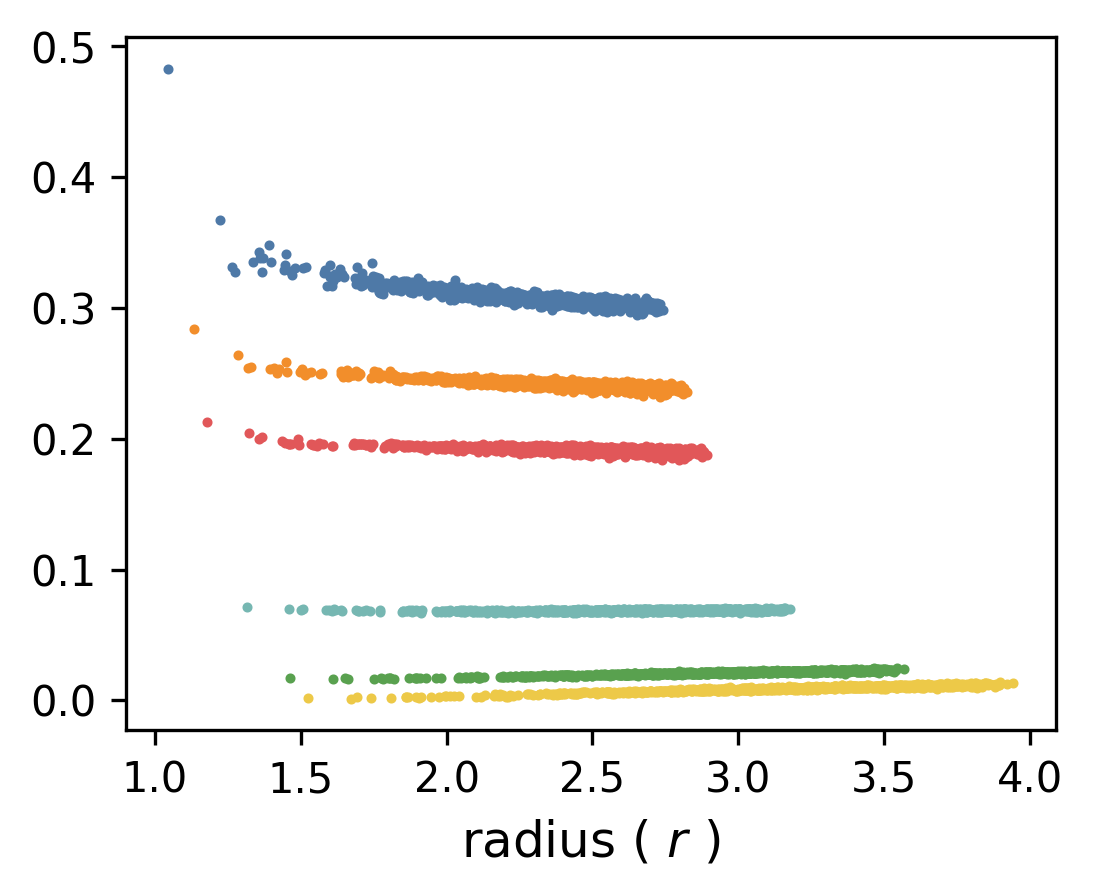}
    \caption*{$d=$10}    
\end{subfigure}
\begin{subfigure}[b]{0.193\textwidth}
    \includegraphics[width=\textwidth]{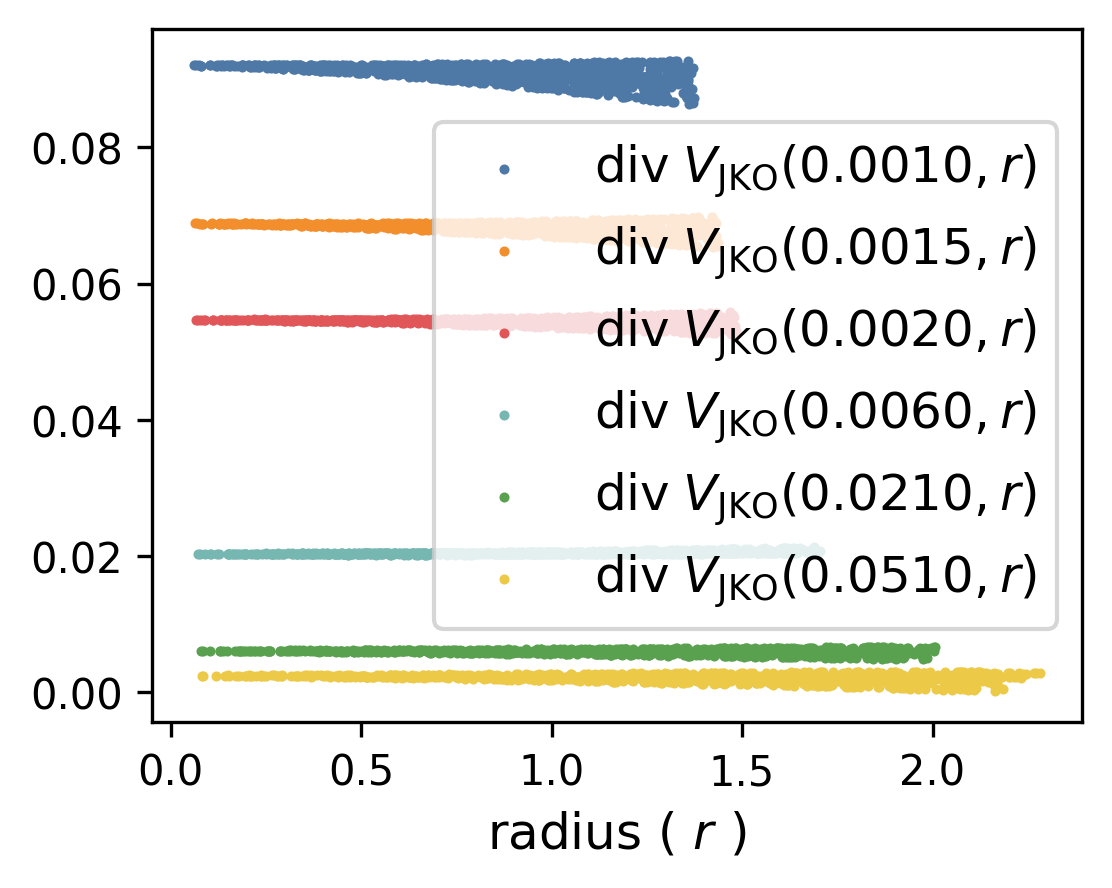}
    \caption*{$d=2(m=4)$}
\end{subfigure}
\caption{Divergence of the predicted velocity fields across data dimensions $d$ with fixed $C=0.5, \Delta t=0.0005$ and $m=2$. }

\label{fig:porousdiv}
\end{figure}

To further validate that our neural operator is learned correctly, we examine the divergence of the predicted velocity field, $\operatorname{div}\mathbf{V}_{\mathrm{JKO}}(t,\cdot)$. For sufficiently small $\Delta t$, it is expected to approximate the exact divergence given in \eqref{equ:porousdiv}. As shown in Figure~\ref{fig:porousdiv}, the predicted divergence is nearly spatially homogeneous, decays as $(t+t_0)^{-1}$, increases with the spatial dimension $d$, and decreases with the nonlinearity $m$, all in agreement with theoretical predictions. In higher dimensions, the scarcity of sample points near the origin leads to larger deviations at small radii, where the predicted divergence is locally overestimated.

\begin{figure}[h]
  \centering
  \renewcommand{\arraystretch}{0}    
  \setlength{\tabcolsep}{0pt}           
  \newcommand{\TimeRow}[7]{
      \includegraphics[width=0.115\textwidth]{fig_exp/#1/0.png} &
      \includegraphics[width=0.115\textwidth]{fig_exp/#1/#2.png} &
      \includegraphics[width=0.115\textwidth]{fig_exp/#1/#3.png} &
      \includegraphics[width=0.115\textwidth]{fig_exp/#1/#4.png} &
      \includegraphics[width=0.115\textwidth]{fig_exp/#1/#5.png} &
      \includegraphics[width=0.115\textwidth]{fig_exp/#1/#6.png} &
      \includegraphics[width=0.115\textwidth]{fig_exp/#1/#7.png} &
     \includegraphics[width=0.115\textwidth]{fig_exp/#1/target.png}\\
  }

\begin{tabular}{cccccccc}  
       \TimeRow{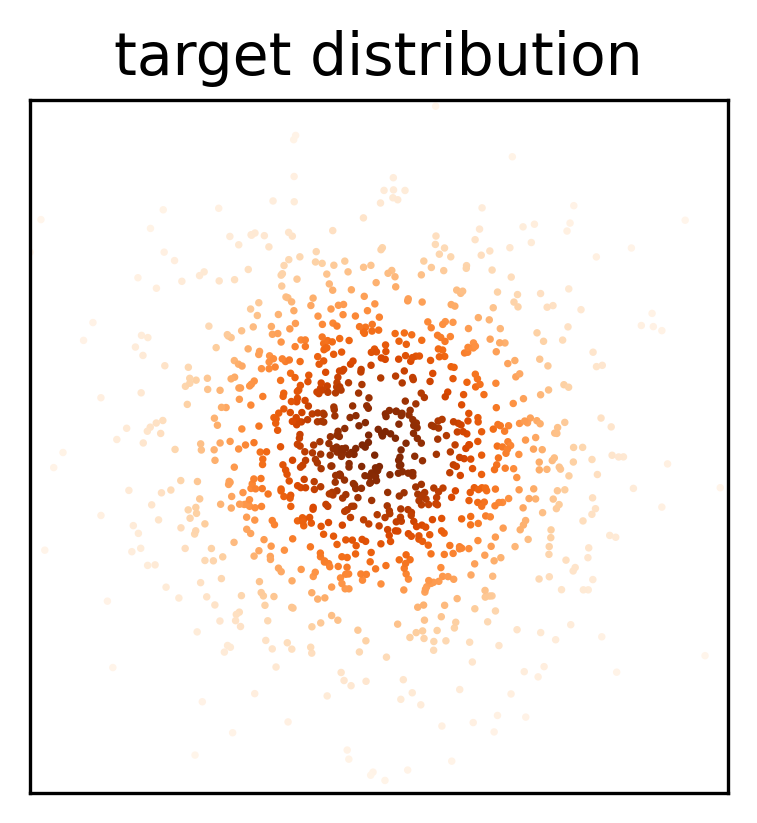}{1}{2}{3}{10}{50}{100}
       \TimeRow{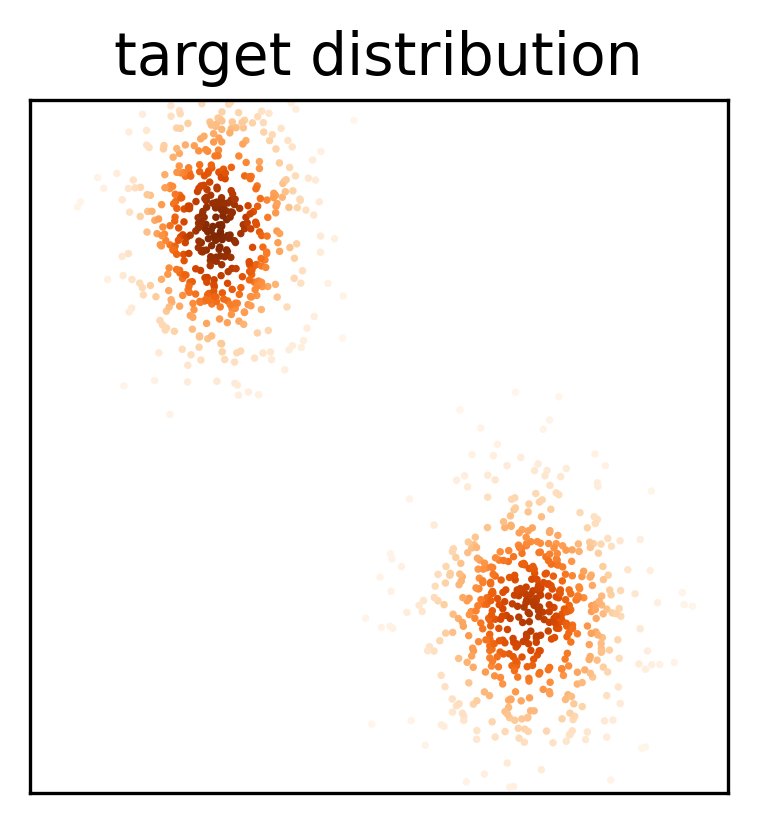}{1}{2}{3}{10}{50}{100}
       \TimeRow{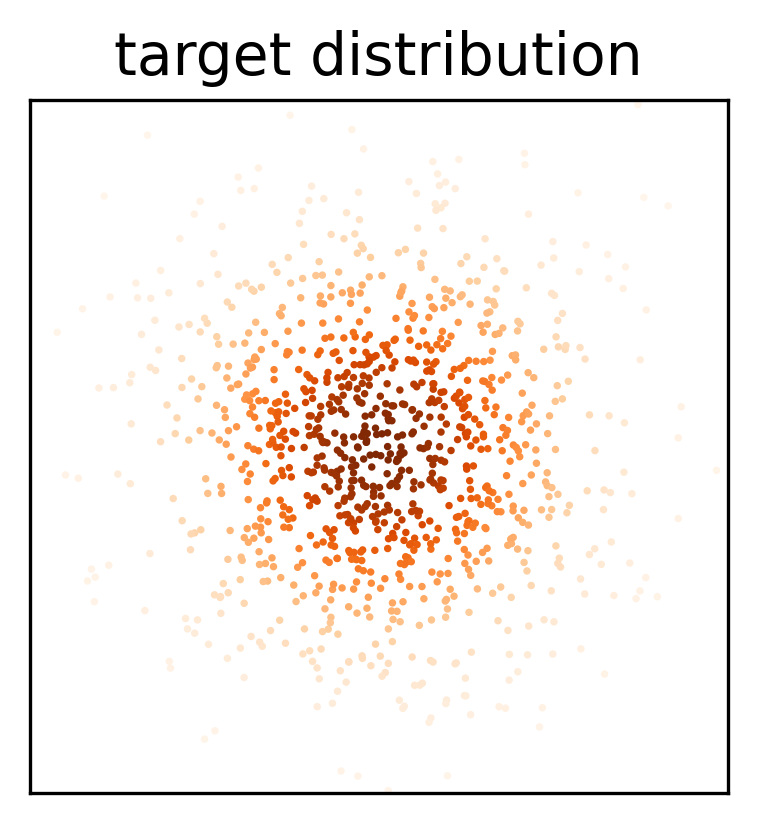}{1}{2}{3}{10}{50}{100}
       \TimeRow{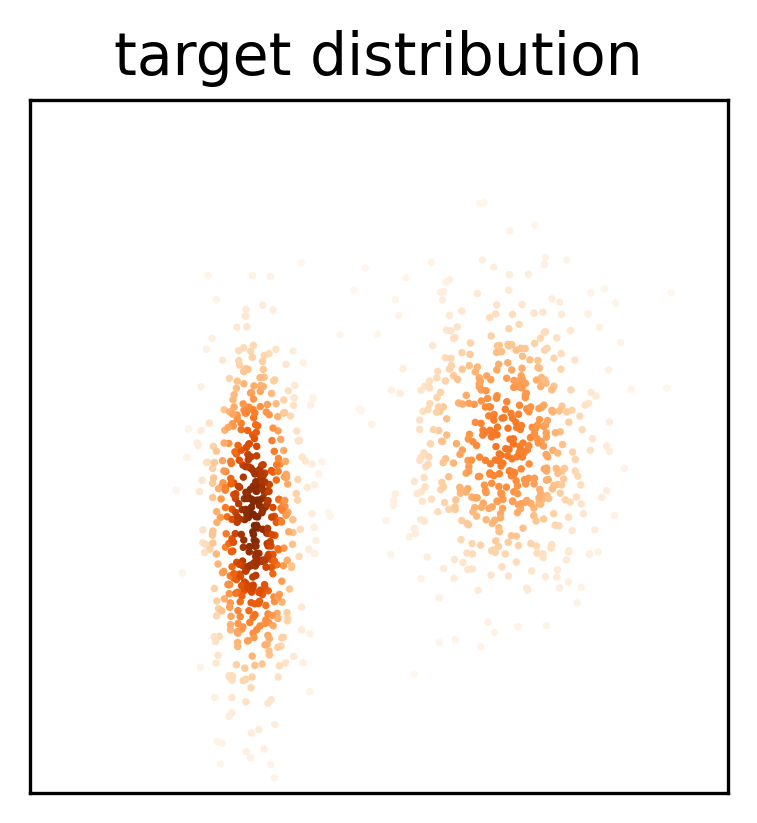}{1}{2}{3}{10}{50}{100}
       \TimeRow{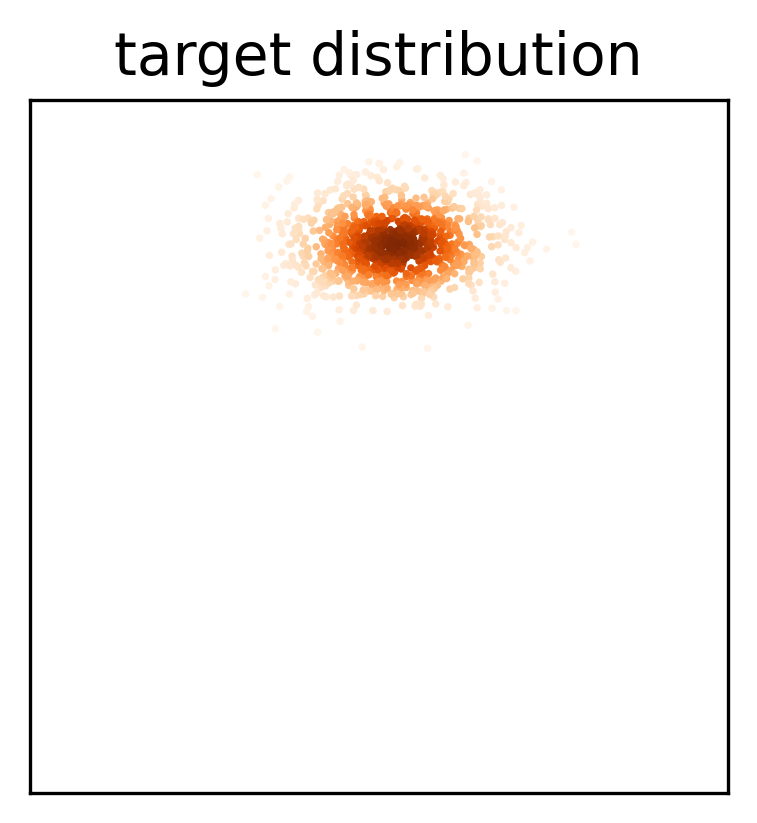}{1}{2}{3}{10}{50}{100}
    
      \vspace{0.2cm}
\\      
\hline
\\
\vspace{0.2cm}\\

    \TimeRow{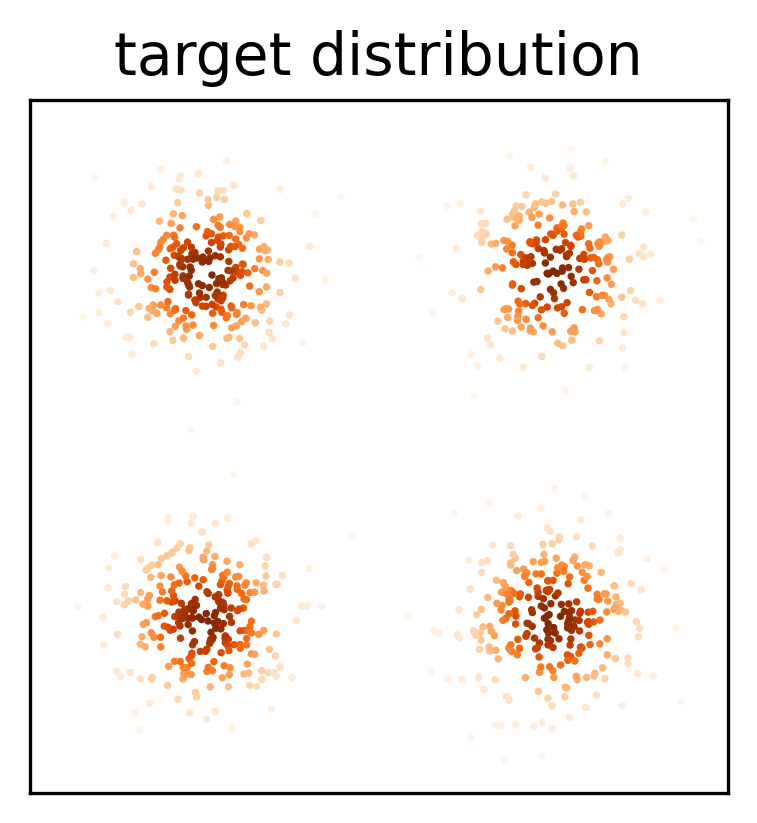}{1}{2}{3}{10}{20}{100}
    \TimeRow{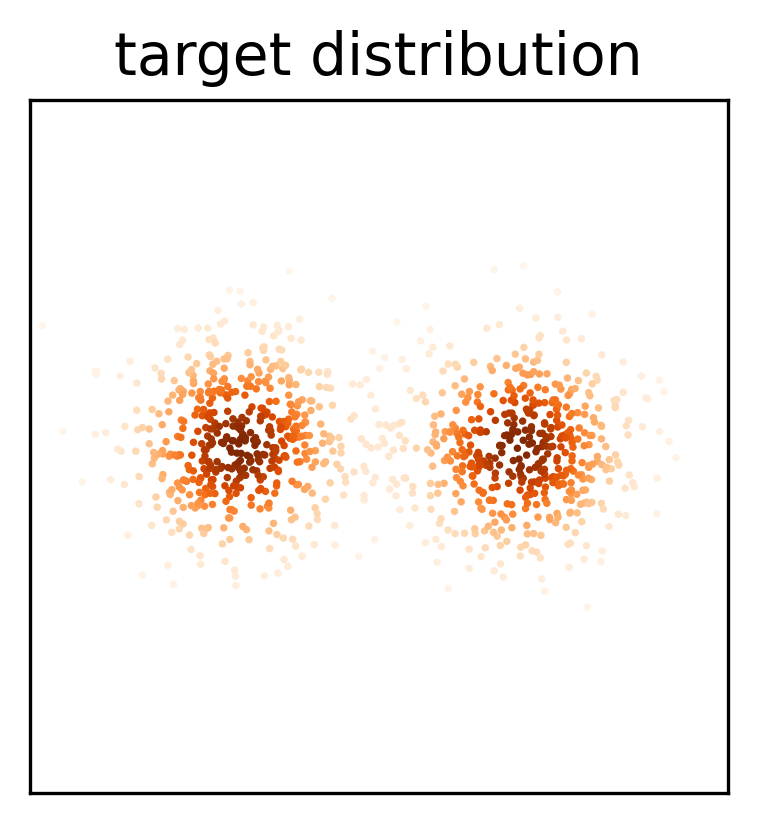}{2}{5}{10}{20}{35}{100}
     \TimeRow{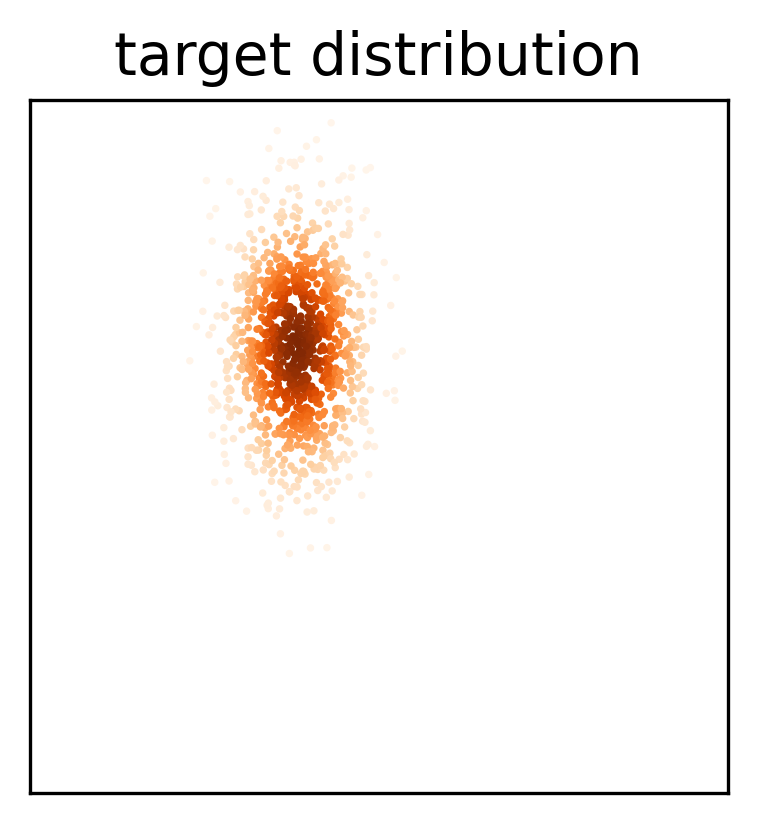}{2}{5}{10}{20}{50}{100}
      \TimeRow{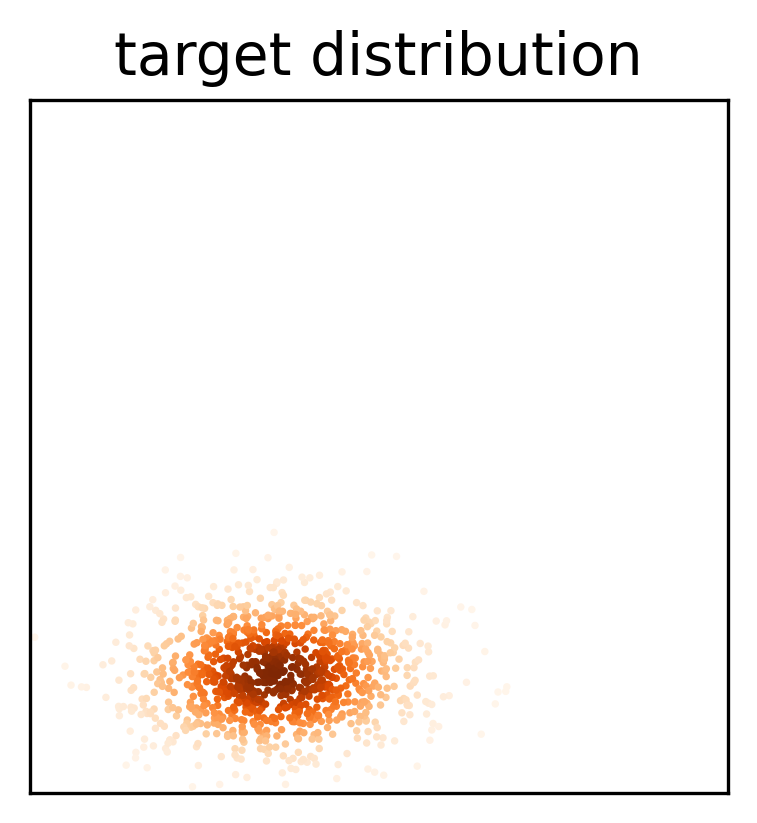}{2}{5}{10}{20}{40}{100}
  \end{tabular}
 \caption{
Predicted Fokker-Planck equation for different pairs of initial and target distributions. 
Point colors indicate the density values at the corresponding locations, with darker colors representing higher densities. 
The horizontal divider separates in-distribution pairs ( top) from out-of-distribution pairs ( bottom).
}

  \label{fig:predflow_kl}
\end{figure}

\subsection{Fokker-Planck equation} In this section, we consider the Fokker-Planck equation of the form: 
\begin{align} \label{FP}
    \partial_t \rho = \nabla_x \cdot \!\bigl(\rho\nabla_x \Phi(x) + \nabla_x \rho \bigr),
\end{align}
whose corresponding energy is the relative entropy: 
\[
\mathrm{D_{KL}} (\rho \,\|\, \rho_{\text{target}})
 = \int_{\Omega} \rho(x) \log
 \frac{\rho(x)}{\rho_{\text{target}}(x)} \,\mathrm{d}x,  \qquad
    \rho_{\mathrm{target}}(x) \propto e^{-\Phi(x)},
\]
where $\Phi(x)$ is often referred to as the potential function of $\rho_{\mathrm{target}}$. Our goal is to train a JKO operator capable of solving the Fokker–Planck equation \eqref{FP} for a family of initial and target distributions.

\paragraph{Training setup}
Since the energy functional is determined by  $\rho_{\mathrm{target}}$, we assume the analytic form of $\rho_{\mathrm{target}}$ is also known and can be directly evaluated during training.
To allow the operator to adapt to different target distributions, 
$\rho_{\mathrm{target}}$ is also provided to the network as part of the input. Specifically, 
the network input consists of the current density paired with its corresponding target distribution, both are represented by their corresponding sample points,
with density values concatenated to each point.
Once trained, the JKO operator can infer the target distribution from its sample representation and generate the corresponding gradient flow from the initial to the target distribution.  

For dimension $d=2$, the training dataset includes (i) the standard Gaussian and (ii) mixtures of two Gaussians, where  the standard deviations are sampled from $(0.2,0.9)$ and the mixture means are sampled uniformly from $\left[-2,2\right]^2$. 
When applying \Cref{alg:neurojko2stage}, 
The batch size  is $3$,
which means 3 initial and target distribution pairs are drawn in each outer iteration.
In each pair, the initial and target distributions are assigned so that one is a standard Gaussian and the other a Gaussian mixture.
This setup allows the model to learn bi-directional gradient flows between these two kinds of distributions.  
We train the network with a fixed  time step $\Delta t=0.02$, and $T=40$ JKO steps, corresponding to gradient flow time from $0$ to $0.8$.  

\paragraph{Prediction performance}
Prediction results are shown in \Cref{fig:predflow_kl}, with the last column representing the target distribution. The trained JKO operator is applied iteratively for up to 100 steps (i.e., $t=2$) from each initial distribution. The predicted gradient flows are smooth, and the relative entropy decreases along the trajectory as indicated by the values shown on each subplot. 
When the target distribution is the standard Gaussian, JKO operator generalizes well beyond the training time domain $t\leq 0.8$, and the relative entropy steadily decrease and approach zero as $t \to 2$. 
However, when the target distribution is a Gaussian mixture, the predictions become less stable near equilibrium.
Although the KL energy remains convex, Gaussian mixture targets induce a more intricate optimal transport geometry.
As a result, even small approximation errors in the learned JKO operator can lead to oscillatory behavior near the steady state.


To further assess generalization, we test the model on unseen initial–target pairs (shown in the second half of \Cref{fig:predflow_kl}), such as flows between mixtures of two and four Gaussians, or between single Gaussians with randomly chosen means and variances. In these settings, the learned flows remain mostly smooth and reasonable.

\section{Conclusion}
\label{sec:conclusion}
In this work, we introduced the Learn-to-Evolve framework for Wasserstein gradient flows, which constructs a neural operator that maps the solution of a parametrized family of gradient flow equations from one time step to the next over long time horizons. We presented the basic framework, detailed the neural network architecture, and described implementation aspects. We also provided theoretical justification for the convergence of our approach and validated it extensively across a range of examples.

The proposed framework, which leverages the JKO formulation of Wasserstein gradient flows, can be extended beyond the JKO setting to other time-evolving systems.
Furthermore, by explicitly incorporating the timestep as an additional network input, the model can utilize temporal information to adapt to varying step sizes, potentially improving accuracy and stability across different temporal resolutions. Several important directions remain open for future investigation.
First, while the presented results demonstrate excellent performance of the learned JKO operator, it remains unclear how much of this success stems from the intrinsic noise tolerance of the JKO scheme itself. A systematic comparison between training the Learn-to-Evolve framework on JKO (proximal) operator and on other time-evolution models would provide valuable insight. 
Second, more advanced training strategies are needed to enrich the data spectrum and improve both the stability and generalization of training.
Third, developing a deeper theoretical understanding of the framework’s generalization properties, especially given that the training data are not fixed but evolve with the learning dynamic, represents a compelling and nontrivial challenge for future work.

\section*{Acknowledgments}
We would like to acknowledge the support of Chenfanfu Jiang's GPU, and helpful discussions about the 
contractiveness of the JKO operator
with Giulia Cavagnari.  LW is supported in part by NSF grant DMS-1846854, DMS-2513336, and the Simons Fellowship. DN was partially supported by NSF DMS 2408912. R. Lai is supported in part by NSF DMS 2401297.

\bibliographystyle{siamplain}
\bibliography{ref}

\appendix

\section{Proof of Lemma 4.1}

\begin{proof}
Let $A = \nabla_x \bfT(x) = I + \nabla_x \bfV(x)$, which is positive definite by assumption.  
Since $A$ is positive definite, we have
\[
\operatorname{tr}(I - A^{-1}) \leq \log\det A \leq \operatorname{tr}(A - I) = \mathrm{div}\,\bfV(x).
\]
For the lower bound side, set $C = \nabla_x \bfV(x)$.  
When $\|C\| < 1$, the Neumann series gives
\(
A^{-1} = I - C + C^2 - C^3 + \cdots ,
\)
and thus
\[
\operatorname{tr}(I - A^{-1}) 
= \mathrm{div}\,\bfV(x) - \operatorname{tr}(C^2) + \operatorname{tr}(C^3) - \cdots .
\]
Since $\|\bfV(x)\| = \mathcal{O}(\Delta t)$, under mild regularity we have 
$\|\nabla_x \bfV(x)\| = \mathcal{O}(\Delta t)$ and $\|C\|_F^2 = \mathcal{O}(\Delta t^2)$, so all higher-order terms are $\mathcal{O}(\Delta t^2)$.  
Therefore,
\(
\operatorname{tr}(I - A^{-1})  = \mathrm{div}\,\bfV(x) + \mathcal{O}(\Delta t^2),
\)
followed by \(\log\det A \leq  \mathrm{div}\,\bfV(x)  + \mathcal{O}(\Delta t^2).\)
Exponentiating yields \Cref{equ:diverrorbounf} 
which proves the claim.
\end{proof}

\end{document}